\documentclass{article}

    \usepackage[final,nonatbib]{neurips_2022}

\usepackage[utf8]{inputenc} % allow utf-8 input
\usepackage[T1]{fontenc}    % use 8-bit T1 fonts
\usepackage{hyperref}       % hyperlinks
\usepackage{url}            % simple URL typesetting
\usepackage{booktabs}       % professional-quality tables
\usepackage{amsfonts}       % blackboard math symbols
\usepackage{nicefrac}       % compact symbols for 1/2, etc.
\usepackage{microtype}      % microtypography
\usepackage{xcolor}         % colors

\usepackage{microtype}
\usepackage{graphicx}
\usepackage{subcaption}
\usepackage{booktabs}
\usepackage{tablefootnote}
\usepackage{footmisc}
\usepackage{multirow}
\usepackage{bbm}
\usepackage{hyperref}
\usepackage[ruled,vlined]{algorithm2e}
\usepackage{enumerate,paralist}%,enumitem}
\usepackage{amsmath}
\usepackage{amssymb}
\usepackage{mathtools}
\usepackage{amsthm}
\usepackage{ulem}
\usepackage{float}

\theoremstyle{plain}
\newtheorem{theorem}{Theorem}
\newtheorem{proposition}{Proposition}
\newtheorem{lemma}{Lemma}
\newtheorem{corollary}{Corollary}
\theoremstyle{definition}

\theoremstyle{remark}

\newcommand{\Ac}{\mathcal{A}}

\newcommand{\Dc}{\mathcal{D}}

\newcommand{\Fc}{\mathcal{F}}

\newcommand{\Lc}{\mathcal{L}}

\newcommand{\Sc}{\mathcal{S}}

\newcommand{\Vc}{\mathcal{V}}

\newcommand{\Xc}{\mathcal{X}}

\newcommand{\Pb}{\mathbb{P}}

\newcommand{\Eb}{\mathbb{E}}
\newcommand{\Rb}{\mathbb{R}}

\newcommand{\argmax}{\arg\max}
\newcommand{\argmin}{\arg\min}

\newcommand{\diag}{\mathrm{diag}}

\usepackage[textsize=tiny]{todonotes}

\title{Anchor-Changing Regularized Natural Policy Gradient for Multi-Objective Reinforcement Learning}

\author{%
  Ruida Zhou\thanks{The first two authors contributed equally.} \\
  Texas A\&M University\\
  \texttt{ruida@tamu.edu} \\
  \And
  Tao Liu$^*$\\
  Texas A\&M University\\
  \texttt{tliu@tamu.edu} \\
  \And
  Dileep Kalathil \\
  Texas A\&M University\\
  \texttt{dileep.kalathil@tamu.edu}\\
  \And
  P. R. Kumar \\
  Texas A\&M University\\
  \texttt{prk@tamu.edu} \\
  \And
  Chao Tian \\
  Texas A\&M University\\
  \texttt{chao.tian@tamu.edu}
}

\begin{document}

\maketitle

\begin{abstract}
We study policy optimization for Markov decision processes (MDPs) with multiple reward value functions, which are to be jointly optimized according to given criteria such as proportional fairness (smooth concave scalarization), hard constraints (constrained MDP), and max-min trade-off. We propose an Anchor-changing Regularized Natural Policy Gradient (ARNPG) framework, which can systematically incorporate ideas from well-performing first-order methods into the design of policy optimization algorithms for multi-objective MDP problems. Theoretically, the designed algorithms based on the ARNPG framework achieve $\tilde{O}(1/T)$ global convergence with exact gradients. Empirically, the ARNPG-guided algorithms also demonstrate superior performance compared to some existing policy gradient-based approaches in both exact gradients and sample-based scenarios. 
\end{abstract}

\section{Introduction}\label{intro}
In many sequential decision-making scenarios, agents usually face multiple objectives simultaneously. This motivates the study of reinforcement learning (RL) with multiple reward values $V^{\pi}_{1:m}(\rho)$. \footnote{The notations are formally defined in Section \ref{sec:prelim}.} Given the achievable region $\Vc = \{V^{\pi}_{1:m}(\rho)\}_{\pi \in \Pi}$ consisting of value vectors achieved by policies in policy class $\Pi$, the agent employs certain criteria to reflect the system requirement. For example, 
\begin{compactitem}[\hspace{0.5cm}]
    \item[1.] Proportional fairness \cite{kelly1998rate}: Given $a_{1:m} > 0$, find $v \in \Vc$ that $\sum_{i=1}^m a_i \frac{v_i' - v_i}{v_i} \leq 0,~\forall v' \in \Vc$.
    \item[2.] Hard constraints \cite{altman1999constrained}: Given $b_{2:m}$, $\text{maximize}_{v \in \Vc} \ v_1$, subject to $v_i \geq b_i, \forall i = 2,\ldots, m$.
    \item[3.] Max-min trade-off \cite{bertsekas2021data}: Given $c_{1:m} > 0$, $\text{maximize}_{v \in \Vc} \min_{i \in [m]} \ (v_i / c_i)$.
\end{compactitem}

We study policy gradient-based approaches that optimize over parameterized policies $\Pi = \{\pi_\theta: \theta \in \Theta\}$ through policy gradient. %, where $\Theta \subset \Rb^{|\Sc||\Ac|}$ is a convex set. 
In general, the optimization problems above may not be convex in terms of $\theta$, not even for single-objective MDPs with direct parameterization by $\theta_{s,a} = \pi_\theta(a|s)$ \cite{agarwal2021theory}. Due to the non-convexity,  $O(1/T)$ global convergence of policy gradient-based methods was only established very recently for single-objective MDPs with exact gradients \cite{agarwal2021theory, mei2020global}. These breakthrough results have motivated the study of policy optimization for multi-objective MDPs, e.g., smooth concave scalarization \cite{bai2021joint}, constrained MDPs (CMDPs) \cite{ding2020natural, xu2021crpo}.

However, under the exact gradients scenario, the previous approaches for multi-objective MDPs, either suffer from slow provable $O(1/\sqrt{T})$ global convergence \cite{ding2020natural}, or require extra assumptions \cite{zhang2020variational,ying2021dual,li2021faster}. The compactness of $\Theta$ is assumed in \cite{zhang2020variational}, but this assumption forbids a very common softmax parameterization, where $\Theta = \Rb^{|\Sc||\Ac|}$. The NPG-based methods have been analyzed in \cite{ying2021dual,li2021faster} under an ergodicity assumption, but such an assumption is not required for NPG in single-objective MDPs \cite{agarwal2021theory}, and therefore appears artificial.

The above criteria for multi-objective MDPs could be viewed as convex optimization problems w.r.t. a value vector $v \in \Vc$, for which there are a wide array of  well-performing first-order methods for convex optimization problems in general. It is desirable to take full advantage of such efficient first-order methods in a unified and flexible manner when designing policy gradient-based algorithms for multi-objective MDPs.

\textbf{Main contributions} 
\begin{compactitem}[\hspace{0.5cm}]
    \item[1.] We propose an anchor-changing regularized natural policy gradient (ARNPG) framework in Section \ref{sec:ARNPG} that can exploit and integrate first-order methods for the design of policy gradient-based algorithms for multi-objective MDPs.
    \item[2.] We demonstrate the strength of the ARNPG framework by designing algorithms for three general criteria: smooth concave scalarization (Section \ref{sec:smooth-scaler}), constrained MDPs (Section \ref{sec:cmdp}), and max-min trade-off (Section \ref{sec:max-min}). 
    \item[3.] Under softmax parameterization with exact gradients, the proposed algorithms inherit the advantages of the integrated first-order methods, and are guaranteed to have $\tilde{O}(1/T)$ global convergence without further assumptions on the underlying MDP. 
    % \textcolor{blue}{Some first-order methods, such as the new primal-dual update, appear in the CMDP-related literature for the first time.}
    \item[4.] In addition to the theoretical advantages, we provide the results of extensive experimentation in Section \ref{sec:empirical} and Appendices \ref{sec:app_exp_cmdp} and \ref{sec:app_exp_momdp} which demonstrate that the ARNPG-guided algorithms provide superior performance in exact gradient and sample-based tabular scenarios, as well as actor-critic deep RL scenarios, compared to several existing policy gradient-based approaches.
\end{compactitem}

\subsection{Related works}
Policy gradient (PG)-based methods have drawn much attention recently \cite{agarwal2020optimality,mei2020global,cen2021fast,khodadadian2021linear} due to their simplicity as well as the potential to generalize to large scale problems. Despite their non-convex nature, PG-based methods have been shown to converge globally for single-objective MDPs \cite{agarwal2020optimality, mei2020global}. Their convergence may be further accelerated with appropriate regularization \cite{cen2021fast,lan2021policy}, e.g., entropy regularization, 
but the algorithms only converge to the optimum of the regularized problem instead of the desired (unregularized) problem.

This paper considers \textit{single-policy} multi-objective MDPs, including CMDPs where constraints are specified on some objectives.
% , since the latter also involves multiple value functions. %Bellman equation does not hold for the multi-objective MDPs, thus value-based methods cannot apply. 
Global convergence of PG-based approaches in the multi-objective MDPs has been previously studied. For smooth concave scalarization, Bai et al. \cite{bai2021joint} showed an $O(1/\epsilon^4)$ sample complexity (to achieve $\epsilon$-optimal in expectation) of the policy-gradient method under sample-based scenarios. However, with exact gradients, we are unaware of works with fast $\tilde{O}(1/T)$ convergence. For CMDPs, Ding et al. \cite{ding2020natural} have studied a primal-dual NPG algorithm achieving $O(1/\sqrt{T})$ global convergence for both the optimality gap and the constraint violation. Xu et al. \cite{xu2021crpo} have proposed a primal approach that reduces constraint violations with a higher priority than optimizing objective, and enjoys the same $O(1/\sqrt{T})$ global convergence. In work conducted concurrently with ours, \cite{ying2021dual} and \cite{li2021faster} have proposed algorithms that achieve $\tilde{O}(1/T)$ convergence but with extra ergodicity assumptions. 

% We study the PG-based approach in the setting where there is a single agent who desires an optimal policy in the presence of multi-objectives. Bai et al. \cite{bai2021joint} have established an $O(1/\epsilon^4)$ \prk{What is $\epsilon$?}sample complexity for a policy-gradient method \textcolor{blue}{that converges to the global optimal with error $\epsilon$}, under smoothness and concavity assumptions on $F$. 
% % Due to the multi-obejctive feature of contrained MDP (CMDP), CMDP is also within the consideration of this paper. 
% The conventional primal-dual subgradient method used to solve convex optimization with functional constraints has convergence lower bounded by $\Omega(1/\sqrt{T})$ \cite{bubeck2015convex}. Assuming access to a proximal mapping, Yu et al. \cite{yu2017simple} have proposed a Lagrangian dual algorithm with an $O(1/T)$ convergence rate by augmenting the Lagrange multipliers. Under a smoothness assumption, the same $O(1/T)$ convergence rate can be attained without needing access to a proximal mapping \cite{yu2017simple}. In our work, we adapt some of the techniques of \cite{yu2017simple} into the CMDP setting. Recently, Xu \cite{xu2020first} has pointed out that $O(\exp(-T))$ convergence can be attained if the objective function possesses an additional strong-convexity property.
% when there are only a bounded number of constraints.

A general setting of optimizing a concave function of the state-action visitation distribution has been considered in \cite{zhang2020variational}. Though the problem is more general, its gradient estimation is more complicated than the canonical policy gradient estimate. Zhang et al. \cite{zhang2020variational} showed that the gradient ascent achieves $O(1/T)$ global convergence for smooth scalarization with exact gradients, under several assumptions such as convexity and compactness of the parameter set $\Theta$. Directly viewing the state-action visitation as the decision variables and imposing equality constraints for their feasibility, a smooth concave scalarization has been studied in \cite{zhang2021beyond} and later generalized to the constrained setting in \cite{bai2021achieving}. These two works focus on sample-based scenarios, but due to their primal-dual approach with equality constraints, the convergence rate is only $O(1/\sqrt{T})$ even with exact gradients. Moreover, the state-action visitation parameterization is difficult to generalize to larger scale deep RL scenarios.

%Zhang et. al. \cite{zhang2020variational} consider a more general setting where the agents optimize a concave function of the state-action visitation distribution. Viewing state-action visitation probabilities as decision variables, and under the assumption of compact convex set $\Theta$, \prk{What is $Theta$ here and why do you even need it here?} \textcolor{blue}{$\Theta$ is a convex set in general. We want to argue that compactness of $\Theta$ is a strong assumption.} the PG-based method achieves $O(1/T)$ global convergence for smooth scalarization function with access to exact gradient. Directly viewing state-action visitation probabilities as decision variables \prk{as above?}  and imposing constraints on their feasibility, a smooth concave scalarization has been studied in \cite{zhang2021beyond} and later generalized to the constrained MDP setting in \cite{bai2021achieving}. These works focus on sample-based analysis, and due to their primal-dual approach for the equality constraint for feasibility, the convergence rate is $O(1/\sqrt{T})$ even with exact access to the gradient. 

A more thorough discussion on related works is given in Appendix \ref{sec:app_related_work}.

\section{Preliminaries} \label{sec:prelim}
%\textcolor{red}{[Define $\Delta$]}

\textbf{System model}~~A Markov decision process (MDP) is represented by a tuple $(\Sc, \Ac, P, \rho, \gamma, r)$, where $\Sc$ is the state space, $\Ac$  the action space, $P: \Sc \times \Ac \to \Delta(\Sc)$ the transition kernel, $\rho \in \Delta(\Sc)$ the initial state distribution, $\gamma \in (0, 1)$  the discount factor, and $r: \Sc \times \Ac \to [0, 1]$ the reward function. Given any policy $\pi: \Sc \rightarrow \Delta(\Ac)$ and any reward function $r: \Sc \times \Ac \rightarrow [0, 1]$, we define the state value function $V_{r}^{\pi}: \Sc \to [0, \frac{1}{1-\gamma}]$, and the state-action value function $Q_{r}^{\pi}: \Sc \times \Ac \to [0, \frac{1}{1-\gamma}]$, as
\vspace{-0.2cm}
\begin{align*}
    V_{r}^{\pi}(s) := \Eb [\sum_{t=0}^{\infty} \gamma^t r(s_t, a_t) ~{|}~ s_0=s, \pi], \quad Q_{r}^{\pi}(s, a) := \Eb [\sum_{t=0}^{\infty} \gamma^t r(s_t, a_t) ~{|}~ s_0=s, a_0=a, \pi],
\end{align*}
%\vspace{-0.1cm}
where expectation $\Eb$ is taken over the random trajectory of the Markov chain induced by the policy $\pi$ and the transition kernel $P$. With a slight abuse of notation, we denote $V^{\pi}_r(\rho) := \Eb_{s \sim \rho}[ V^{\pi}_r(s)]$. Define the discounted state-action visitation distribution (state-action visitation for short) of policy $\pi$ with initial state distribution $\rho$ by $d_{\rho}^{\pi}(s, a) := (1 - \gamma)\Eb_{s_0 \sim \rho}[\sum_{t=0}^{\infty} \gamma^t \Pb(s_t=s, a_t=a|s_0, \pi) ]$. 
It then follows that $V^{\pi}_r(\rho) = \frac{1}{1 - \gamma} \langle d_{\rho}^{\pi}, r \rangle$ by viewing $d^{\pi}_\rho$ and $r$ as $|\Sc||\Ac|$-dimensional vectors indexed by $(s, a) \in \Sc \times \Ac$. When it is clear from the context, we denote the state visitation distribution by $d_{\rho}^{\pi}(s) := \Eb_{s_0 \sim \rho}\left[(1 - \gamma) \sum_{t=0}^{\infty} \gamma^t \Pb(s_t=s|s_0)\right]$, which is the marginal distribution of the state-action visitation $d^{\pi}_{\rho}(s, a)$, i.e., $d^{\pi}_\rho(s) = \sum_{a \in \Ac} d^{\pi}_\rho(s, a)$. 

We study an MDP with $m$ objectives represented by $(\Sc, \Ac, P, \rho, \gamma, r_{1:m})$, where $r_i: \Sc \times \Ac \to [0, 1]$ is the $i$-th reward function for each $i \in [m]$. For simplicity, denote $V^{\pi}_{i}(\cdot) := V^{\pi}_{r_i}(\cdot)$ and $V^{\pi}_{1:m}(\cdot) := (V^{\pi}_{1}(\cdot), \ldots, V^{\pi}_m(\cdot))$. We consider parameterized policies in $\Pi = \{\pi_\theta: \theta \in \Theta\}$, where $\Theta \subset \Rb^n$ is the parameter space. For example, the softmax policy is $\pi_\theta(a|s) = \frac{\exp(\theta_{s,a})}{\sum_{a'} \exp(\theta_{s,a'})}$ with $\Theta = \Rb^{|\Sc||\Ac|}$; and neural softmax policy is $\pi_{\theta}(a|s) = \frac{\exp(\text{NN}_\theta(s,a))}{\sum_{a'} \exp(\text{NN}_\theta(s,a'))}$, where NN${_\theta}$ is some neural network parameterized $\theta$. %In the rest of the paper, we may occasionally omit $\theta$ in $\pi_\theta$ for simplicity, but it should be noted that all updates of policies are performed on the parameters. 
%\textcolor{red}{[Check]} 
Define $\Vc := \{V^{\pi_{\theta}}_{1:m}(\rho): \theta \in \Theta\}$ as the achievable region of value vectors. The agent wishes to optimize the policy in $\Pi$ for a given specific multi-objective criterion on value vectors in $\Vc$.

%Given a particular multi-objective criterion as in \prk{(Equation numbers)}, the agent wishes to determine an optimal policy.

% involve a trade-off among individual objectives within the achievable region $\Vc_\rho := \{V^{\pi}_{1:m}(\rho): \pi \in \Pi\}$. The agent selects the policy that optimizes the desired appropriate criteria and optimizes the parameterized policy to reflect and fulfill the user-specific demand. %Detailed studies of the proposed algorithms for different criteria are in Section \ref{sec:theory-app}. %The algorithms bring ideas from first-order methods in convex optimization to the ARNPG framework proposed in Section \ref{sec:ARNPG}. 

%\paragraph{Visitation distribution} For any stationary policy $\pi$ and any $(s, a) \in \Sc \times \Ac$, 

\textbf{Mirror ascent}~~As one of the most well-known iterative optimization methods, mirror descent (actually ascent in the context of our formulation as a maximization problem) \cite{nemirovski2004prox, beck2003mirror} is a general class that encompasses many first-order methods in convex optimization. Given a variable $x$ in a compact convex set $\Xc \subset \Rb^n$ and an ascent direction $g \in \Rb^n$, the variational representation of the mirror ascent update is
\begin{align}
    x' \in \argmax_{y \in \Xc} \{ \langle g, y \rangle - \alpha B_h(y || x) \}, \label{eqn:mirror_ascent}
\end{align}
where $B_h(x || y) := h(x) - h(y) - \langle \nabla h(y), x - y\rangle$ is some Bregman divergence generated by a differentiable convex function $h: \Xc \rightarrow \Rb$. %that encodes prior knowledge on the geometry of the problem.
When analyzing the convergence of first-order methods, certain fundamental inequalities are usually established to facilitate the proof. One such inequality is
\begin{align}\label{eqn:fundamental-MD}
    \langle g, x' \rangle - \alpha B_h(x' || x) \geq \langle g, y \rangle - \alpha B_h(y || x) + \alpha B_h(y || x'), \quad \forall y \in \Xc,
\end{align}
which is a critical step in many previous works, e.g., \cite{rakhlin2013optimization,wei2020linear,lan2020first}.

It is desirable to construct a similar fundamental inequality for multi-objective MDPs that can facilitate the analysis of convergence. 
As we will show in the next section, such an inequality can indeed be established in a new framework, which we refer to as the Anchor-Changing Regularized Natural Policy Gradient (ARNPG). 

\textbf{Notations}~~Denote KL-divergence between two $n$-dimensional probability vectors $x,y$ by $D(x||y) := \sum_{i=1}^n x_i \log(x_i/y_i)$, which is a widely-used Bregman divergence. % -- KL-divergence between \prk{Say that $\Delta$ is a simplex in $n$ dimensions.}$x, y \in \Delta([n])$ as $D(x||y) := \sum_{i=1}^n \log(x_i)\frac{\log(x_i)}{\log(y_i)}$. \prk{Correct the statement.} %, which is a Bregman divergence $B_h$ generated by function $h(x) = \sum_{i=1} x_i \log(x_i)$. 
For any policies $\pi, \pi'$ and state visitation distribution $d$, define $D_{d}(\pi || \pi') := \sum_{s \in \Sc} d(s) D(\pi(\cdot| s)|| \pi'(\cdot | s))$. A \textit{uniform policy} is one which chooses actions uniformly at random.

%the expected KL divergence between $\pi$ and $\pi'$ under $d$ is defined as $D_{d}(\pi || \pi') := \sum_{s \in \Sc} d(s) D(\pi(\cdot| s)|| \pi'(\cdot | s))$.

% We say a differentiable function $F: \Rb^m \rightarrow \Rb$ is $L$-Lipschitz w.r.t. to norm $\|\cdot\|$ if $\|\nabla F(v)\|_* \leq L$ and we say $F: \Rb^m \rightarrow \Rb$ is $\beta$-smooth w.r.t. to norm $\|\cdot\|$ if $\|\nabla F(v) - \nabla F(v')\|_* \leq \beta \|v - v'\|$, $\forall v, v' \in \Vc$.

\section{Anchor-changing regularized natural policy gradient} \label{sec:ARNPG}

%\prk{Through mixed policies does it not follow that $\Vc_\rho$ is indeed convex?} \sout{In the multi-objective RL setting we studied in the paper, the agent seeks a value vector $v$ in $\Vc_\rho$ according to certain criteria. When viewing $v$ as the decision variables, the problem could be a convex program with respect to $v$ and mirror ascent could be applied to improve the value vector \textcolor{purple}{[]Tian: why "could be"? Why not "is" and "can be"?]} \textcolor{blue}{If $\Vc_\rho$ is not convex, the prolem is not a convex program}.  For example\textcolor{purple}{[Tian: remove "for example"? this reads very undefinitive and shows a lack of confidence]},} 
Let us consider a hypothetical mirror ascent update on decision value vector $v_k \in \Vc$ according to \eqref{eqn:mirror_ascent}. Given an ascent direction $\tilde{G}_k$ along which to improve $v_k$, the updated value vector is %represented by the mirror ascent update \eqref{eqn:mirror_ascent}
\begin{align}
    v' \in \argmax_{v \in \Vc}\{ \langle \tilde{G}_k, v \rangle - \alpha B_h(v || v_k) \}. \label{eqn:v-mirror-ascent}
\end{align}
Suppose the value vector $v_k$ is achieved by a policy $\pi_{\theta_k}$, i.e., $v_k = V^{\pi_{\theta_k}}_{1:m}(\rho)$. %\textcolor{purple}{[Tian: This sentence should be moved before the equation?]} 
Denote the reward function in the ascent direction as $\tilde{r}_k(s,a) = \langle \tilde{G}_k, r_{1:m}(s,a) \rangle$. It follows that $\langle \tilde{G}_k, v_k \rangle = V^{\pi_{\theta_k}}_{\tilde{r}_k}(\rho)$. Note that $B_h(v || v_k)$ in \eqref{eqn:v-mirror-ascent} serves the role of a soft constraint on $v$ by keeping $v$ within a vicinity of $v_k$. %It will fulfill the similar purpose by 
Replacing $B(v || v_k)$ by $\frac{D_{d^{\pi_\theta}_\rho}(\pi_\theta || \pi_{\theta_k})}{1 - \gamma}$ will induce a similar soft constraint that prefers the  vicinity of the ``anchor" policy $\pi_{\theta_k}$.  %by providing force on $\pi_{\theta}$ to move towards $\pi_{\theta_k}$.
Therefore we consider replacing the variational update in \eqref{eqn:v-mirror-ascent} by

\vspace{-0.4cm}
\begin{align}
    \theta' \in \argmax_{\theta \in \Theta} \left\{ \tilde{V}_{k, \alpha}^{\pi_\theta}(\rho)  \right\}, \quad \text{where} \quad \tilde{V}_{k, \alpha}^{\pi_\theta}(\rho) := V^{\pi_\theta}_{\tilde{r}_k}(\rho) - \alpha \frac{D_{d^{\pi_\theta}_\rho}(\pi_\theta || \pi_{\theta_k})}{1 - \gamma}. \label{eqn:inner-loop-variational}
\end{align}

%The argmax for updating $\theta'$  in \eqref{eqn:inner-loop-variational} may not be solved efficiently, however, we can approximate $\theta'$ via iteratively methods. 
%At the macro step $k$, ARNPG framework first leverages first-order methods (for convex optimization) to determine the ascent reward function $\tilde{r}_k$ and the anchor policy $\pi_{\theta_k}$. Since the optimal solution $\theta'$ in \eqref{eqn:inner-loop-variational} may not be solved efficiently, ARNPG approaches the optimal solution via a subroutine that performs natural policy gradient (NPG) w.r.t. the KL-regularized value function $\tilde{V}^{\pi_\theta}_{k,\alpha}(\rho)$. We name this subroutine InnerLoop($\tilde{r}_k, \pi_{\theta_k}, \alpha, \eta, t_k$) as shown in Algorithm \ref{alg:InnerLoop}. It iteratively updates the parameter $\theta_k^{(t)}$ for $t_k$ (micro) steps according to the NPG update rule in \eqref{eqn:inner-NPG}, where $\Fc_{\rho}(\theta)^{\dagger}$ is the Moore-Penrose inverse of the Fisher information matrix $\Fc_{\rho}(\theta) := \mathbb{E}_{(s,a) \sim d_{\rho}^{\pi_{\theta}}} \left[\nabla_{\theta} \log \pi_{\theta}(a|s)\left(\nabla_{\theta} \log \pi_{\theta}(a|s)\right)^{\top}\right]$. %\textcolor{purple}{[Tian: This sentence is way too long. Break into two?}]
\textbf{ARNPG}~~ Motivated by the intuition above, we propose the Anchor-Changing Regularized Natural Policy Gradient (ARNPG) framework. At (macro) step $k$, the ARNPG framework determines the reward function in the ascent direction $\tilde{r}_k$ and the anchor policy $\pi_{\theta_k}$, which can exploit well-performed first-order methods in convex optimization literature utilizing the features of the specific criteria in use. 
%At step $k$, the ARNPG framework leverages first-order methods (for convex optimization) to determine the reward in the ascent direction $\tilde{r}_k$, and the anchor policy $\pi_{\theta_k}$. 
With $\tilde{r}_k$ and $\pi_{\theta_k}$, we wish to solve for \eqref{eqn:inner-loop-variational} to improve the value vector. %\textcolor{red}{[Check]}
However the optimal solution $\theta'$ of \eqref{eqn:inner-loop-variational} is generally not determinable explicitly. ARNPG therefore approaches the optimal solution via a subroutine that executes a natural policy gradient (NPG) algorithm w.r.t. the KL-regularized value function $\tilde{V}^{\pi_\theta}_{k,\alpha}(\rho)$. We refer to this subroutine,
given in Algorithm \ref{alg:InnerLoop}, as InnerLoop($\tilde{r}_k, \pi_{\theta_k}, \alpha, \eta, t_k$). It %initializes $\theta_k^{(0)} = \theta_k$ and 
iteratively updates the parameter $\theta_k^{(t)}$ for $t_k$ (micro) steps according to the NPG update rule as in \eqref{eqn:inner-NPG}, 
% \begin{align}
%     \theta_k^{(t+1)} \leftarrow \theta_k^{(t)} + \eta \Fc_{\rho}(\theta_k^{(t)})^{\dagger} \nabla_\theta \tilde{V}_{k, \alpha}^{{\theta_k^{(t)}}}(\rho),
% \end{align}
where $\Fc_{\rho}(\theta)^{\dagger}$ is the Moore-Penrose inverse of the Fisher information matrix $\Fc_{\rho}(\theta) := \mathbb{E}_{(s,a) \sim d_{\rho}^{\pi_{\theta}}} \left[\nabla_{\theta} \log \pi_{\theta}(a|s)\left(\nabla_{\theta} \log \pi_{\theta}(a|s)\right)^{\top}\right]$. %\textcolor{purple}{[Tian: This sentence is way too long. Break into two?}]

\begin{algorithm}[ht]
%\addcontentsline{loa}{algorithm}{My Algorithm}
\caption{InnerLoop($\tilde{r}_k, \pi_{\theta_k}, \alpha, \eta, t_k$)}
\label{alg:InnerLoop}
%\noindent \textbf{Input:} $\rho, K, \alpha, \eta, t_{0:K-1}, F, r_{1:m}$;\\
% \prk{Also $V$ or $\tilde{Q}$ is an input.}\\
\noindent \textbf{Initialize} $\theta^{(0)}_k = \theta_k$\\
\For{$t = 0, 1, \ldots t_k - 1$}{
\vspace{-0.4cm}
\begin{align}
    \theta_k^{(t+1)} \leftarrow \theta_k^{(t)} + \eta \Fc_{\rho}(\theta_k^{(t)})^{\dagger} \nabla_\theta \tilde{V}_{k, \alpha}^{\pi^{(t)}_{k}}(\rho) \label{eqn:inner-NPG}
\end{align}
\vspace{-0.7cm}
}
\textbf{Return} $\theta_k^{(t_k)}$
\end{algorithm}

The choice of the number of iterations in InnerLoop (i.e., $t_k$) involves a trade-off between the variational update precision and the overall efficiency. On the one hand, a larger $t_k$ leads to a more accurate approximation of the optimal solution $\theta'$ to \eqref{eqn:inner-loop-variational}, but it may cause the algorithm to spend unnecessary computational resources on the regularized objective $\tilde{V}_{k,\alpha}^{\pi_\theta}(\rho)$, instead of on the true optimization problem. On the other hand, a smaller $t_k$ saves inner loop iterations but the update follows less closely to the underlying mirror-ascent update in improving the value vector. In our experiments, % (c.f. Section \ref{sec:empirical} and Appendix \ref{sec:app_exp_momdp}), 
we choose $t_k$ within 10 to strike a balance and empirically 
observe $t_k > 1$ has better performance. %for both the exact gradient and the sample-based scenarios.

We note that when $t_k = 1$, the gradient $\nabla_\theta \tilde{V}_{k,\alpha}^{\pi_{\theta_k}}(\rho) = \nabla_\theta V_{\tilde{r}_k}^{\pi_{\theta_k}}(\rho)$, since $D_{d^{\pi_\theta}_\rho}(\pi_\theta || \pi_{\theta_k})$ has zero gradient at $\theta = \theta_k$. The update in \eqref{eqn:inner-NPG} reduces to an NPG update on the unregularized value function $\tilde{V}_{\tilde{r}_k}^{\pi_\theta}(\rho)$. For single-objective MDPs, it reduces to the canonical NPG method. %\textcolor{blue}{The update will follow less closely to the guidance of the mirror ascent, and we will empirically show $t_k > 1$ may have better performance.} 
%\prk{Can you add a comment on how the theoretical results depend on the choice of $t_k$?}

% Note that $\tilde{V}_{k,\alpha}^{\pi_\theta}(\rho)$ is a KL-regularized MDP, and the NPG applied on the this regularized MDP has fast linear convergence when  

% The output of the InnerLoop is then selected as the anchor point for the next macro step $k+1$, i.e., $\theta_{k+1} = \theta_{k}^{(t_k)}$. As $\theta_k$ is updated towards the optimal policy, the anchor points are also updated towards the optimal policy, and the bias from the regularization is thus mitigated.

\subsection{Theoretical guarantee of ARNPG} \label{sec:interpretation}
We now present the main theoretical tool for the analysis of the ARNPG framework. Recall the discussion of the fundamental inequality after \eqref{eqn:fundamental-MD}. Proposition \ref{pro:fundamental-inequality} establishes such a fundamental inequality with controllable approximation error under the softmax policy parameterization, i.e., $\pi_\theta(a|s) = \frac{\exp(\theta_{s,a})}{\sum_{a'} \exp(\theta_{s,a'})}$. 
% in Proposition \ref{pro:fundamental-inequality} that the ARNPG framework, under the softmax policy parameterization where $\pi_\theta(a|s) = \frac{\exp(\theta_{s,a})}{\sum_{a'} \exp(\theta_{s,a'})}$, supports a fundamental inequality of similar form to \eqref{eqn:fundamental-MD} with controllable approximation error. 
In the rest of the paper, we omit $\theta$ in $\pi_\theta$ when it is clear from the context, but it should be noted that all updates of policies are performed on the parameters. 
%Note that the class of softmax policies is almost complete, in the sense that its closure contains all stationary policies, and we will thus omit the parameter $\theta$ in $\pi_\theta$.
%\textcolor{purple}{[Tian: Need some thing more explicit to point out the importance of this proposition. For example "Next we present the fundamental inequality associated with the proposed ARNPG framework, which will serve an instrumental role in the next section." The sentence after the proposition is too subdued. ]}

\begin{proposition}\label{pro:fundamental-inequality}
Under the softmax parameterization, given $\epsilon_k > 0$, for any $\tilde{r}_k$, $t_k \geq \frac{1}{1 - \gamma} \log(\frac{5 \|\tilde{r}_k\|_\infty}{(1-\gamma)^2 \epsilon_k}) + 1$, $\alpha > 0$ and $\eta = \frac{1 - \gamma}{\alpha}$, the update $\pi_{k+1} \leftarrow \text{InnerLoop}(\pi_k,\tilde{r}_k, \alpha, \eta, t_k)$ satisfies
\begin{align}
    V^{\pi_{k+1}}_{\tilde{r}_k}(\rho) - \alpha \frac{D_{d^{\pi_{k+1}}_\rho}(\pi_{k+1} || \pi_k)}{1 - \gamma} & \geq V^{\pi}_{\tilde{r}_k}(\rho) - \alpha \frac{D_{d^{\pi}_\rho}(\pi || \pi_k) - D_{d^{\pi}_\rho}(\pi || \pi_{k+1})}{1 - \gamma} - \epsilon_k,\quad \forall \pi. \label{eqn:fundamental-inequality}
\end{align}
\end{proposition}
%Proposition \ref{pro:fundamental-inequality} plays a key role in allowing the incorporation of ideas from mirror-ascent first-order methods into the ARNPG framework, thereby making possible the theoretical guarantee on convergence rate.
The inequality \eqref{eqn:fundamental-inequality} is critical to the convergence proof.  Its right hand side allows telescoping, which by summing over $k$ can iteratively cancel the terms $D_{d^{\pi}_\rho}(\pi||\pi_k)$. Since $t_k = \Theta(\log(1/\epsilon_k))$ it suffices to use very few iterations in InnerLoop for maintaining precision.

\textit{Remark.} It has been shown that for the entropy-regularized MDP, i.e., KL-regularized with the uniform policy as the anchor policy, NPG converges linearly (i.e., geometrically fast) to the regularized optimal policy \cite{cen2021fast}. It is natural to anticipate that for the KL-regularized MDP $\tilde{V}^{\pi}_{k,\alpha}(\rho)$ with anchor $\pi_k$, NPG would similarly converge linearly (i.e., $\tilde{V}^{\pi_{k}}_{k, \alpha} \geq \tilde{V}^{\pi^*_k}_{k, \alpha} - \epsilon$
for $t_k = \Theta(\log(1/\epsilon))$) to a corresponding optimal policy, denoted as $\pi_k^*$. %Such a linearly convergent algorithm therefore takes very few steps to reach a solution sufficiently close to the optimal one.
%Quantitatively,  $\tilde{V}^{\pi_{k}}_{k, \alpha} \geq \tilde{V}^{\pi^*_k}_{k, \alpha} - \epsilon$ for $t_k = \Theta(\log(1/\epsilon))$. 
In contrast, the right hand side of inequality \eqref{eqn:fundamental-inequality} has a \textit{positive drift} $\alpha \frac{D_{d^\pi_\rho}(\pi||\pi_{k+1})}{1- \gamma}$ \textit{for any policy $\pi$}, which is considerably stronger.

\textit{Proof sketch of Proposition \ref{pro:fundamental-inequality}.} 
%Although $V^{\pi_\theta}_{\tilde{r}_k}(\rho)$ is not concave w.r.t. $\theta$ and $D_{d^{\pi_\theta}_\rho}(\pi_\theta||\pi_{\theta_k})$ may not be a Bregman divergence, 
We can show that InnerLoop approximately solves the variational update in (\ref{eqn:inner-loop-variational}) with linear convergence as anticipated. However to establish \eqref{eqn:fundamental-inequality}, the difficulty lies in the introduction of positive drift, since $V^{\pi_\theta}_{\tilde{r}_k}(\rho)$ is not concave w.r.t. $\theta$ and $D_{d^{\pi_\theta}_\rho}(\pi_\theta||\pi_{\theta_k})$ may not be a Bregman divergence. 
% Besides showing the linear convergence of InnerLoop, %the difficulty in proving Proposition \ref{pro:fundamental-inequality} is to 
% we need to show that InnerLoop approximately solves the variational update in (\ref{eqn:inner-loop-variational}), even when  
We tackle this difficulty by showing that optimizing $\pi_\theta$ in InnerLoop implicitly performs a mirror ascent update for state action visitation $d^{\pi_\theta}_\rho$. %, with a Bregman divergence $\tilde{D}(d^{\pi_\theta'}_\rho || d^{\pi_\theta}_\rho)$ defined in the appendix.
%Noting that $d^{\pi_\theta}_\rho$ belongs to a convex and compact set, $V^{\pi_\theta}_{\tilde{r}_k}(\rho)$ is a linear function of $d^{\pi_\theta}_\rho$, and rewriting $D_{d^{\pi_\theta}_\rho}(\pi_\theta||\pi_{\theta_k})$ in terms of $d^{\pi_\theta}_\rho, d^{\pi_{\theta_k}}_\rho$ to show that it induces a Bregman divergence.
\qed

As demonstrated in the next section, Proposition \ref{pro:fundamental-inequality} ensures that the convergence rate of the algorithms derived from the ARNPG framework is of the same rate as the underlying first-order methods with only extra logarithmic factors. 

\section{Theoretical applications} \label{sec:theory-app}
In this section, we apply the ARNPG framework to several important multi-objective MDP scenarios and obtain new policy optimization algorithms by integrating first-order methods in convex optimization. 
%Leveraging the ideas from first-order methods in convex optimization, the algorithms specify ascent reward functions $\tilde{r}_k$ and present pseudo codes under the ARNPG framework.
%Integrating first-order methods in convex optimization, we obtain new policy optimization algorithms guided by the ARNPG framework
%We propose policy optimization algorithms based on the ARNPG framework and ideas from first-order methods in convex optimization. 
All the theoretical results presented in this section are under the softmax parameterization with exact gradients. However, the obtained algorithms can be implemented in more general settings such as neural softmax and sample-based scenarios, as in the next section. We theoretically establish $\tilde{O}(1/T)$ convergence of these algorithms by leveraging the fundamental inequality in Proposition \ref{pro:fundamental-inequality}.% under the softmax parameterization.

\subsection{Smooth concave scalarization function} \label{sec:smooth-scaler}
We start by considering the following optimization problem
\begin{align}
    \max_{\theta} F(V_{1:m}^{\pi_\theta}(\rho)), \label{def:smooth-scalar}
\end{align}
where $F$ is a concave function, and $\beta$-smooth w.r.t. $\|\cdot\|_\infty$ norm, i.e., $\|\nabla F(v) - \nabla F(v')\|_1 \leq \beta\|v - v'\|_\infty$. Since the set of achievable values $\Vc \subseteq \left[0, \frac{1}{1-\gamma}\right]^{m}$, it can be verified that $\|\nabla F(v)\|_1 \leq L$ for some factor $L > 0$. 

The proportional fair criterion discussed in Section \ref{intro} can be approximated by $F(v) := \sum_{i=1}^m a_i\log(\delta + v_i)$, where $\delta > 0$ is some constant introduced to circumvent the pathological case $v_i = 0$ for some $i \in [m]$. Under this criterion, $\beta = \sum_{i=1}^m a_i /\delta^2$ and $L = \sum_{i=1}^m a_i /\delta$.

%To apply the ARNPG, we only need to specify the ascent reward functions $\tilde{r}_k$ at macro step $k$, and appropriately select other hyper-parameters: regularization coefficient $\alpha$ and step size (learning rate) $\eta$, number of iterations in inner loop $t_k$.

When $v$ is viewed as the decision variable, at macro step $k$ with value vector $V^{\pi_k}_{1:m}(\rho)$, the ascent direction in a typical gradient ascent step is the gradient $\tilde{G}_k = \nabla_v F(V^{\pi_k}_{1:m}(\rho))$. This naturally determines the reward in the ascent direction as %$\tilde{r}_k$ with 
$\tilde{r}_k(s,a) = \langle \tilde{G}_k, r_{1:m}(s,a) \rangle$. Adapting the ARNPG framework to this specific context, we present the algorithm for solving the program (\ref{def:smooth-scalar}) in Algorithm \ref{alg:ARNPG-IMD}. We refer to it as ``implicit mirror descent" because the algorithm implicitly employs mirror descent. %the state-action visitation space with Bregman divergence $\tilde{D}(\cdot|| \cdot)$.

%simply iteratively calling InnerLoop($\pi_k, \tilde{r}_k, \alpha, \eta, t_k$) for $k = 0,1, K-1$, starting from uniform policy $\pi_0$.
% \begin{algorithm}[ht]
% \caption{\textbf{ARNPG Implicit Mirror Descent (ARNPG-IMD)}}
% \label{alg:ARNPG-IMD}
% %\noindent \textbf{Input:} $\rho, K, \alpha, \eta, t_{0:K-1}, F, r_{1:m}$;\\
% % \prk{Also $V$ or $\tilde{Q}$ is an input.}\\
% \noindent \textbf{Initialize} $\pi_0$ as uniform policy\\
% \For{$k = 0, 1, \dots, K-1$}{
% \noindent Update $\pi_{k+1} \leftarrow $InnerLoop($\pi_k, \tilde{r}_k, \alpha, \eta, t_k$)
% }
% \textbf{Return} policy in $\{\pi_k\}_{k=1}^K$ with the largest $F(V^{\pi_k}_{1:m}(\rho))$
% \end{algorithm}

\begin{algorithm}[ht]
\caption{\textbf{ARNPG Implicit Mirror Descent (ARNPG-IMD)}}
\label{alg:ARNPG-IMD}
\textbf{Input} $\pi_0, \alpha, \eta, t_{0:K-1}, K$ \\
%\noindent \textbf{Initialize} $\pi_0$ as uniform policy\\
\For{$k = 0, 1, \dots, K-1$}{
% \noindent Calculate $\tilde{r}_k(s,a) = \langle \nabla_v F(V^{\pi_k}_{1:m}(\rho)), r_{1:m}(s,a)\rangle$\\
\noindent Update $\pi_{k+1} \leftarrow $InnerLoop($\pi_k, \tilde{r}_k, \alpha, \eta, t_k$)
}
\textbf{Return} the policy in $\{\pi_k\}_{k=1}^K$ with the largest $F(V^{\pi_k}_{1:m}(\rho))$
\end{algorithm}

Let $\pi^*$ be the optimal policy for (\ref{def:smooth-scalar}). Based on Proposition \ref{pro:fundamental-inequality}, we present the following theorem which guarantees the convergence of ARNPG-IMD with appropriately selected parameters $\pi_0, \alpha, \eta, t_{k}$.
\begin{theorem}
\label{thm:ARNPG-IMD}
For any $K \geq 1$, take uniform policy $\pi_0$, $\alpha \geq \frac{\beta}{(1 - \gamma)^3}$, $\eta = \frac{1-\gamma}{\alpha}$, and $t_k = \lceil\frac{1}{1 - \gamma} \log(\frac{5LK}{\beta \log(|\Ac|)}) + 1\rceil$. The optimality gap of ARNPG-IMD (Algorithm \ref{alg:ARNPG-IMD}) satisfies 
\begin{align}
    F(V_{1:m}^{\pi^*}(\rho)) - \max_{k \in [1:K]} F(V_{1:m}^{\pi_k}(\rho)) \leq & F(V_{1:m}^{\pi^*}(\rho)) - \frac{1}{K} \sum_{k=1}^{K} F(V_{1:m}^{\pi_k}(\rho)) \leq \frac{2\alpha \log(|\Ac|)}{(1-\gamma)K}. \label{eqn:ARNPG-IMD-gap}
\end{align}
\end{theorem}

There are a total of $K$ macro steps, and the total number of iterations is $T = \sum_{k = 0}^{K-1} t_k = \Theta(\frac{K}{1-\gamma} \log(K))$. The following corollary provides the convergence rate in terms of $T$.%, where the big-$O$ notation is w.r.t. $T$, but we keep some $m$ and $\frac{1}{1-\gamma}$ related factors for preserving the details.
\begin{corollary}
\label{cor:smooth-scalar}
%Denote by $T := \sum_{k=0}^{K-1} t_k$ the total number of iterations of the ARNPG-IMD (Algorithm \ref{alg:ARNPG-IMD}).
Under the same conditions as in Theorem \ref{thm:ARNPG-IMD}, the ARNPG-IMD algorithm satisfies 
$F(V_{1:m}^{\pi^*}(\rho)) - \frac{1}{K} \sum_{k=1}^{K} F(V_{1:m}^{\pi_k}(\rho)) = O\left(\frac{\beta \log(T)}{(1 -\gamma)^5 T}\right)$.
%where $b_1$ is a universal constant.
\end{corollary}

%\textcolor{red}{[When $F$ is linear, $\beta = 0$. The convergence rate of NPG for single-objective MDP is $\frac{1}{(1-\gamma)^2 T}$. In other words, we should set sufficiently small $\alpha = \Theta(1/T)$ and large $t_0 = T$. The algorithm with one macro step is reduced to the entropy regularized NPG.]}

\textit{Remark.} In the absence of knowledge of $K$, we can select time-varying numbers of InnerLoop iterations, such as $t_k = \Theta(\log(k))$, and ARNPG-IMD will still have the same $\tilde{O}(1/T)$ convergence.

\subsection{Constrained Markov decision process} \label{sec:cmdp}
Another way of trading off the objectives is to optimize one while setting hard constraints on the others. This can be formulated as the following constrained MDP (CMDP) problem:
\begin{align}\label{def:cmdp}
    \max_{\theta} V_{1}^{\pi_\theta}(\rho), \quad \text{s.t.} \ \ V_{i}^{\pi_{\theta}}(\rho) \geq b_i,~\forall i \in [2:m],
\end{align}
where $b_{2:m} \in [0, \frac{1}{1-\gamma}]^{m-1}$. Let $\pi^* = \pi_{\theta^*}$ be the optimal policy of the CMDP problem in \eqref{def:cmdp}.

% Let $\pi^*$ be the optimal policy of the CMDP problem in (\ref{def:cmdp}). It is well-known that in general the optimal policy $\pi^*$ is randomized and the Bellman equation may not hold \cite{altman1999constrained}. We assume strict feasibility of \eqref{def:cmdp}, which naturally implies the existence of the optimal policy.

% \begin{assumption}[Slater's condition] \label{asm:slater}
% There exists $\xi > 0$ and $\overline{\pi}$ such that $V_{c_i}^{\overline{\pi}}(\rho) \le - \xi$, $\forall i \in [m]$.
% \end{assumption}
% This assumption is quite standard in the optimization literature for analyzing primal-dual algorithms \cite{bertsekas2014constrained}. 
% %\dk{cite a textbook, maybe Beck?}.
% In particular, many related works in the CMDP literature (see, e.g., \cite{ding2020natural, ding2021provably, efroni2020exploration, liu2021learning}) make the same strict feasibility assumption. Note that unlike previous primal-dual algorithms \cite{ding2020natural, ding2021provably} for CMDPs, where $\xi$ is required to be known a priori for the projection of dual variables, our proposed algorithm does not require the knowledge of $\xi$, and this assumption is made only for the analysis. Under Assumption \ref{asm:slater}, the optimal dual variables are bounded by $\|\lambda^*\| \le \frac{2}{\xi (1 - \gamma)}$ (cf. Lemma in the Appendix).

Define the Lagrangian of the CMDP problem as $\Lc(\pi_\theta, \lambda) = V_1^{\pi_\theta}(\rho) + \sum_{i=2}^m \lambda_i (V_i^{\pi_\theta}(\rho) - b_i)$, where $\lambda_i$ is the Lagrange multiplier (dual variable) corresponding to the constraint $V_i^{\pi_\theta} \geq b_i$, for each $ i \in [2:m]$. The Lagrange dual function $\max_\pi \Lc(\pi, \cdot)$ is a convex function of dual variables $\lambda \geq 0$. Denote by $\lambda^*$ the optimal dual variables that minimize the Lagrange dual function. We assume $\lambda^*$ is finite, which is guaranteed by Slater's condition, i.e., there is some $\pi_\theta$ and $\xi > 0$ with $V^{\pi_\theta}_i(\rho) - b_i \geq \xi$ for any $i \in [2:m]$. Note $(\pi^*, \lambda^*)$ is a saddle point of the Lagrangian $\Lc(\pi, \lambda)$. This motivates the primal-dual approach, which iteratively performs gradient ascent for $\pi_\theta$ and gradient descent for $\lambda$. This is suitable for the CMDP setting, since for any fixed $\lambda$, the Lagrangian $\Lc(\pi, \lambda)$ corresponds to an MDP for which policy gradient can be employed.

The canonical primal-dual gradient ascent-descent method for constrained convex optimization can only guarantee $O(1/\sqrt{T})$ convergence, and consequently the primal-dual policy gradient-based approach for CMDPs \cite{ding2020natural} has the same convergence. Recently, Yu et al. \cite{yu2017simple} have proposed a primal-dual-based method with $O(1/T)$ convergence under the Euclidean setting, i.e., $B_h(x||y) = \frac{1}{2}\|x-y\|_2^2$. 
Adopting ideas from \cite{yu2017simple}, we next propose the ARNPG with Extra Primal-Dual (ARNPG-EPD) algorithm (Algorithm \ref{alg:ARNPG-EPD}). To the best of our knowledge, this new primal-dual update appears in the CMDP-related literature for the first time.

Note that $b_i - V^{\pi}_i(\rho)$ is the amount of constraint violation. %If the constraint violation is positive (or negative), the dual variable, updated following gradient descent of $\Lc(\pi_\theta, \cdot)$, should increase (or decrease) corresponding to the second term in \eqref{eqn:dual-update-equation}. 
There are two key ideas we adopt from \cite{yu2017simple}. The first is the design of the reward in the ascent direction
\vspace{-0.2cm}
\begin{center}
    $\tilde{r}_k(s, a) := r_1(s, a) + \sum_{i=2}^m (\lambda_{k, i} + \eta' (b_i - V_{i}^{\pi_k}(\rho))) r_i(s, a)$,
\end{center}
\vspace{-0.2cm}
where an extra constraint violation term is added to the dual variables. % so as to ``predict" the dual variables in the next step. 
The second idea is that the update of dual variables should not fall below the negative constraint violation (the first term in \eqref{eqn:dual-update-equation}), and it can alleviate the overshooting of dual variables. The extra constraint violation terms in $\tilde{r}_k$ and the dual update work jointly to ensure the $\tilde{O}(1/T)$ convergence.

\begin{algorithm}[ht]
\caption{\textbf{ARNPG with Extra Primal Dual (ARNPG-EPD)}}
\label{alg:ARNPG-EPD}
%\noindent \textbf{Input:} $\rho, K, \alpha, \eta, t_{0:K-1}, F, r_{1:m}$;\\
% \prk{Also $V$ or $\tilde{Q}$ is an input.}\\
\textbf{Input} $\pi_0, \eta', \alpha, \eta, t_{0:K-1}, K$\\
\noindent \textbf{Initialize} $\lambda_{0,i} = \max\{\eta'(V_i^{\pi_{0}}(\rho) - b_i), 0\},~\forall i \in [2:m]$\\
\For{$k = 0, 1, \dots, K-1$}{
% \noindent \hspace{-0.175cm} Calculate $\tilde{r}_k(s, a) = r_1(s, a) + \sum_{i=2}^m (\lambda_{k, i} + \eta' (b_i - V_{i}^{\pi_k}(\rho))) r_i(s, a)$ \\
\noindent \hspace{-0.175cm} Update $\pi_{k+1} \leftarrow $InnerLoop($\pi_k, \tilde{r}_k, \alpha, \eta, t_k$)\\
\noindent %$\lambda_{k+1, i} = \max\left\{\eta' (V^{\pi_{k+1}}_{i}(\rho) - b_{i}), \lambda_{k,i} + \eta'(b_{i} - V^{\pi_{k+1}}_{i}(\rho)) \right\}, ~\forall i \in [2:m]$ 
\vspace{-0.5cm}
\begin{flalign}
    \hspace{-0.6cm} \text{Update}~ \lambda_{k+1, i} = \max\left\{\eta' (V^{\pi_{k+1}}_{i}(\rho) - b_{i}), \lambda_{k,i} + \eta'(b_{i} - V^{\pi_{k+1}}_{i}(\rho)) \right\}, ~\forall i \in [2:m]
    \label{eqn:dual-update-equation}
\end{flalign}
\vspace{-0.55cm}
}
\textbf{Return:} a policy randomly chosen from $\{\pi_k\}_{k = 1}^{K}$
\end{algorithm}

\begin{theorem}
\label{thm:ARNPG-EPD}
For any $K \geq 1$ and $\eta' \in (0, 1]$, take uniform policy $\pi_0$, $\alpha \ge \frac{2 \eta' m}{(1 - \gamma)^3}$, $\eta = \frac{1-\gamma}{\alpha}$, and choose $t_k = \lceil\frac{1}{1 - \gamma} \log(\frac{5 L_k K}{2 \eta' m \log(|\Ac|)}) + 1 \rceil$ with $L_k = 1 + \frac{\eta' (m-1)}{1-\gamma} + \sum_{i=2}^m \lambda_{k, i}$.
% $L_k = 2 \gamma (\frac{1+\sum_{i=1}^m \lambda_{k,i}}{1-\gamma} + \frac{m \eta'}{(1-\gamma)^2})$. 
The average optimality gap and the average constraint violation of ARNPG-EPD (Algorithm \ref{alg:ARNPG-EPD}) satisfy
\begin{align}
    V_{1}^{\pi^*}(\rho) - \frac{1}{K} \sum_{k=1}^{K} V_{1}^{\pi_k}(\rho) &\leq \frac{3\alpha \log(|\Ac|)}{(1-\gamma)K} \label{eqn:PD-opt-gap}, \\
    b_i - \frac{1}{K} \sum_{k=1}^{K} V_{i}^{\pi_k}(\rho) &\leq \frac{1}{K} \left(\frac{2 \|\lambda^*\|_2}{\eta'} + 3\sqrt{\frac{\alpha \log(|\Ac|)}{(1 - \gamma)\eta'} }\right) \quad \forall i \in [2:m]. \label{eqn:con-vio}
    % &\frac{1}{K} \sum_{k=1}^{K} V_{1}^{\pi_k}(\rho) \geq V_{1}^{\pi^*}(\rho) - \frac{3\alpha \log(|\Ac|)}{(1-\gamma)K} \label{eqn:PD-opt-gap} \\
    % &\frac{1}{K} \sum_{k=1}^{K} V_{i}^{\pi_k}(\rho) \geq b_i - \frac{1}{K} \left(\frac{2 \|\lambda^*\|_2}{\eta'} + 3\sqrt{\frac{\alpha \log(|\Ac|)}{(1 - \gamma)\eta'} }\right) \quad \forall i \in [2:m]. \label{eqn:con-vio}
\end{align}
%where $\lambda^*$ is the vector of optimal dual variables of the CMDP problem \eqref{def:cmdp}. 
\end{theorem}
%\textcolor{red}{[Discuss the drawback. Check the proof of Corollary, uniform upper bound for $\lambda_k$]}
Note that the number of micro steps $t_k$ is chosen according to the dual variables $\lambda_{k}$ in the previous theorem. Denote by $T := \sum_{k=0}^{K-1} t_k$ the total number of iterations. 
\begin{corollary}
\label{cor:ARNPG-EPD}
Under the same conditions as in Theorem \ref{thm:ARNPG-EPD}, the ARNPG-EPD algorithm satisfies $V^{\pi^*}_1(\rho) - \frac{1}{K} \sum_{k=1}^{K} V_{1}^{\pi_k}(\rho) = O(\frac{m \log(T)}{(1 - \gamma)^5 T})$, and $b_i - \frac{1}{K} \sum_{k=1}^{K} V_{i}^{\pi_k}(\rho) = O(\frac{\sqrt{m} \log(T)}{(1 - \gamma)^{2.5} T})$.% for each $i \in [2:m]$.
% \begin{align*}
%     V^{\pi^*}_1(\rho) - \frac{1}{K} \sum_{k=1}^{K} V_{1}^{\pi_k}(\rho) = O(\frac{m \log(|\Ac|) \log(T)}{(1 - \gamma)^5 T}), b_i - \frac{1}{K} \sum_{k=1}^{K} V_{i}^{\pi_k}(\rho) = O(\frac{\sqrt{m \log(|\Ac|)} \log(T)}{(1 - \gamma)^{2.5} T}).%~\forall i \in [2:m].
%     %b_1' \frac{\|\lambda^*\|\sqrt{m \log |\Ac|} \log (C^* T)}{(1-\gamma)^3T},
% \end{align*}
% \begin{align*}
%     V^{\pi^*}_1(\rho) - \frac{1}{K} \sum_{k=1}^{K} V_{1}^{\pi_k}(\rho) &= O\left(\frac{m \log(|\Ac|) \log(T)}{(1 - \gamma)^5 T}\right), \\ 
%     %\\b_1 \frac{m \log|\Ac| \log (C^* T)}{(1-\gamma)^5 T}, \\
%     b_i - \frac{1}{K} \sum_{k=1}^{K} V_{i}^{\pi_k}(\rho) &= O\left(\frac{\sqrt{m \log(|\Ac|)} \log(T)}{(1 - \gamma)^{2.5} T}\right).%~\forall i \in [2:m].
%     %b_1' \frac{\|\lambda^*\|\sqrt{m \log |\Ac|} \log (C^* T)}{(1-\gamma)^3T},
% \end{align*}
\end{corollary}
The theorem and corollary establish convergence of the average optimality gap and the average constraint violation, in the same manner as many previous works \cite{ding2020natural,xu2021crpo,jain2022towards,liu2021learning} on CMDPs. However, a guarantee on the last iterate is more preferable. This drawback is inherited from the primal-dual algorithm for convex optimization, where the primal-dual algorithm with sublinear convergence can only be guaranteed on the average solution, as of our knowledge. Last iterate convergence is still an on-going open research topic. %We believe the ARNPG framework is powerful and general enough to incorporate more advanced techniques developed in the future.

% only provides a sublinear convergence guarantee 

\subsection{Max-min trade-off criteria}
\label{sec:max-min}
Finally, we consider the max-min trade-off criterion defined as
\begin{align}\label{def:minimax}
    \max_{\theta} \min_{\lambda \in \Lambda} \Phi(V^{\pi_\theta}_{1:m}(\rho), \lambda),
\end{align}
where $\Lambda$ is a subset of the $m$-dimensional probability simplex $\Delta([m])$. We assume $\Phi(\cdot, \lambda)$ is concave and $\Phi(v, \cdot)$ is convex. We also assume $\Phi$ is $\beta$-smooth w.r.t. the norm $\Psi(v,\lambda) = \|v\|_\infty + \|\lambda\|_1$. 
%whose formal definition is left in Appendix \ref{sec:max-min}. 

% An example of interest is the robust scalarization function, where the optimization program becomes
% \begin{align*}
%     \max_{\theta} \min_{\lambda \in \Lambda} \langle V_{1:m}^{\pi}(\rho), \lambda \rangle,
% \end{align*}
% % The minimax function $\Phi(v, \lambda) = \sum_{i = 1}^m v_i \lambda_i$ 
% which satisfies the concave-convex assumption and is also $\beta = m$ smooth w.r.t. the norm $\Psi$.
%An example of interest is 
The max-min criterion mentioned in Section \ref{intro} can be represented by $\Phi(v, \lambda) = \sum_{i = 1}^m v_i \lambda_i / c_i$ and $\Lambda = \Delta([m])$. %\prk{Explain why it is a robustness criterion.}
% The minimax function $\Phi(v, \lambda) = \sum_{i = 1}^m v_i \lambda_i$ 
$\Phi$ satisfies the concave-convex assumption and is $\beta$-smooth w.r.t. the norm $\Psi$ with $\beta = O(m)$.

Denote $F(v) := \min_{\lambda \in \Lambda} \Phi(v, \lambda)$, which is concave but not necessarily smooth. Thus we cannot apply the ARNPG-IMD algorithm (Algorithm \ref{alg:ARNPG-IMD}) due to the non-smoothness of $F$, and the subgradient-based method can only guarantee $O(1/\sqrt{T})$ convergence. 

We next integrate the optimistic mirror descent ascent (OMDA) method \cite{wei2020linear} for solving minimax optimization %(variational inequality problem in general) 
in the ARNPG framework. Denote the gradients $\tilde{G}_{k}^{\lambda} = \nabla_{\lambda} \Phi(V^{\tilde{\pi}_k}_{1:m}(\rho),\tilde{\lambda}_k)$ and $\tilde{G}_{k}^{v} = \nabla_v \Phi(V^{\tilde{\pi}_k}_{1:m}(\rho), \tilde{\lambda}_k)$. It can be verified that $\|\tilde{G}_k^v\|_1 \leq L$ for some $L$ due to the smoothness of $\Phi$. OMDA performs gradient ascent along the direction $\tilde{G}_k^{v}$ w.r.t. the value vector, and therefore we construct the reward in the ascent direction as $\tilde{r}_k(s,a) = \langle \tilde{G}_{k}^{v}, r_{1:m}(s,a) \rangle$. 
OMDA performs mirror descent along direction $\tilde{G}_{k}^{\lambda}$ w.r.t. the dual vector $\lambda$. A key ingredient of OMDA is that it updates twice in each macro step. ARNPG-OMD adopts this idea and update $(\pi,\lambda)$ from the same anchor points $(\pi_k, \lambda_k)$, first with ascent direction $(\tilde{r}_k, -\tilde{G}^{\lambda}_k) \in \Rb^{2m}$ and then a further step with direction $(\tilde{r}_{k+1}, -\tilde{G}^{\lambda}_{k+1}) \in \Rb^{2m}$. %It is presented in Algorithm \ref{alg:ARNPG-OMD}.

% \begin{align}
% \label{eqn:lambda}
%     \tilde{\lambda}_{k+1} := \argmin_{\lambda} \langle \nabla_{\lambda} \Phi(V^{\pi_k}_{1:m}(\rho), \lambda_k), \lambda \rangle + \frac{1}{\eta'} D(\lambda || \lambda_k).
% \end{align}

\begin{algorithm}[ht]
\caption{\textbf{ARNPG with Optimistic Mirror Descent Ascent Update (ARNPG-OMDA)}}
\label{alg:ARNPG-OMD}
%\noindent \textbf{Input:} $\rho, K, \alpha, \eta, t_{0:K-1}, F, r_{1:m}$;\\
% \prk{Also $V$ or $\tilde{Q}$ is an input.}\\
\noindent \textbf{Input} $\pi_0, \lambda_0, \eta', \alpha, \eta, t_{0:K-1}, K$\\
\noindent \textbf{Initialize} $\tilde{\pi}_0 = \pi_0$ and $\lambda_0, \tilde{\lambda}_0$ as uniform distribution on $[m]$\\
\For{$k = 0, 1, \dots, K-1$}{
\noindent Update $\tilde{\pi}_{k+1} \leftarrow $InnerLoop($\pi_k, \tilde{r}_k, \alpha, \eta, t_k$), $\tilde{\lambda}_{k+1} \leftarrow \argmin_{\lambda \in \Lambda}\{\langle\tilde{G}_{k}^{\lambda}, \lambda \rangle +  \frac{D(\lambda || \lambda_k)}{\eta'}\}$ \\
\noindent Update $\pi_{k+1} \leftarrow $InnerLoop($\pi_k, \tilde{r}_{k+1}, \alpha, \eta, t_k$), $\lambda_{k+1} \leftarrow \argmin_{\lambda \in \Lambda}\{\langle\tilde{G}_{k+1}^{\lambda}, \lambda \rangle + \frac{D(\lambda || \lambda_k)}{\eta'}\}$
}
\textbf{Return:} a policy randomly chosen from $\{\tilde{\pi}_k\}_{k=1}^K$ %the policy in $\{\tilde{\pi}_k\}_{k=1}^K$ with the largest $F(V^{\tilde{\pi}_k}_{1:m}(\rho))$%$$\min_{\lambda \in \Lambda} \Phi(V^{\tilde{\pi}_k}_{1:m}(\rho), \lambda)$
\end{algorithm}

We present ARNPG-OMDA in Algorithm \ref{alg:ARNPG-OMD}, and establish the following
 performance guarantees:
\begin{theorem}
\label{thm:ARNPG-OMDA}
For any $K \geq 1$, take uniform policy $\pi_0$, $\eta' \leq \frac{1}{6\beta}$, $\alpha \geq \frac{6 \beta}{(1 - \gamma)^3}$, $\eta = \frac{1-\gamma}{\alpha}$, and $t_k = \lceil\frac{1}{1 - \gamma} \log(\frac{5LK}{6 \beta \log(|\Ac|)}) + 1\rceil$. The ARNPG-OMDA algorithm (Algorithm \ref{alg:ARNPG-OMD}) satisfies
\begin{align}
    &F(V_{1:m}^{\pi^*}(\rho)) - F\left( \frac{1}{K} \sum_{k=1}^{K} V_{1:m}^{\tilde{\pi}_k}(\rho) \right) \leq \frac{3\alpha \log(|\Ac|)}{(1 - \gamma)K} + \frac{\log(m)}{\eta' K}. \label{eqn:max-min-scalar-gap}
\end{align}
\end{theorem}
%\textcolor{red}{[Discuss the drawback of average]} 
% There are some unsatisfactory parts for the theorem. 
Similar to the discussion after Corollary \ref{cor:ARNPG-EPD}, Theorem \ref{thm:ARNPG-OMDA} provides a performance guarantee on the average value vector $F(\frac{1}{K} \sum_{k=1}^{K} V_{1:m}^{\tilde{\pi}_k}(\rho))$, which is inherited from the OMDA methods. Denote the total number of iterations by $T := \sum_{k=0}^{K-1} 2 t_k$.
\begin{corollary}
\label{cor:ARNPG-OMDA}
Under the same conditions as in Theorem \ref{thm:ARNPG-OMDA},  ARNPG-OMDA satisfies $F\left(V_{1:m}^{\pi^*}(\rho)\right) - F\left(\frac{1}{K} \sum_{k=1}^{K}V_{1:m}^{\pi_k}(\rho)\right) = O\left(\frac{\beta \log(T)}{(1 -\gamma)^5 T}\right).$
% \begin{align*}
%     F\left(V_{1:m}^{\pi^*}(\rho)\right) - F\left(\frac{1}{K} \sum_{k=1}^{K}V_{1:m}^{\pi_k}(\rho)\right) = O\left(\left(\frac{\beta \log(|\Ac|)}{(1 -\gamma)^5} + \frac{\log(m)}{\eta' (1 - \gamma)}\right) \frac{\log(T)}{T}\right).
% \end{align*}
%where $b_1$ is a universal constant.
\end{corollary}
% Similar to Theorem \ref{thm:ARNPG-EPD}, the theorem and corollary above only give convergence for the average policy, while the last iterate guarantee is more preferable. This drawback is inherited from the primal-dual algorithm for max-min optimization. 
%that as far as we know, the primal-dual algorithm with sublinear convergence have only guarantee on the average, and the last iterate convergence is still an on-going open research topic. We believe the ARNPG framework is powerful and general enough to incorporate more advanced techniques developed in the future, yet such development is not the focus of this work.

%\prk{This decomposition into the general in Sections 2 and 3, followed by the particular specializations in Section 4 is excellent.}

\section{Empirical evaluation and application} \label{sec:empirical}
In this section, we present the experimental results on CMDP. We compare the performance of the proposed ARNPG-EPD algorithm (Algorithm \ref{alg:ARNPG-EPD}) with two benchmarks: NPG-PD \cite{ding2020natural} and CRPO \cite{xu2021crpo}. 
% \textcolor{blue}{Here, ``average'' represents the forms of LHS of (\ref{eqn:PD-opt-gap}) and (\ref{eqn:con-vio}), while ``last-iterate'' means the optimality gap between the optimal value and the last output of the algorithm.} 
Experimental details on CMDP are postponed to Appendix \ref{sec:app_exp_cmdp} and further experiments on smooth concave scalarization and max-min trade-off are presented in Appendix \ref{sec:app_exp_momdp}. We provide code at \href{https://github.com/tliu1997/ARNPG-MORL}{https://github.com/tliu1997/ARNPG-MORL}.

% Below, the definitions of the average optimality gap and the average constraint violation come from the left-hand side of (\ref{eqn:PD-opt-gap}) and (\ref{eqn:con-vio}), while the last-iterate optimality gap and last-iterate constraint violation are defined as $d$

\subsection{Tabular CMDP with exact gradients}
\label{sec:cmdp_exact}
Recall that under softmax policy with exact gradients, Corollary \ref{cor:ARNPG-EPD} (Theorem \ref{thm:ARNPG-EPD}) guarantees $\tilde{O}(1/T)$ convergence of both performance measures: average optimality gap and average constraint violation. We compare the proposed ARNPG-EPD with the benchmarks NPG-PD and CRPO under both performance measures on a randomly generated CMDP with a single constraint, which are illustrated in Figure \ref{fig:simple_cmdp}. The horizontal axis is the total number of iterations, i.e., including the micro steps in InnerLoop of ARNPG-EPD.

\begin{figure}[ht]
\centering
\begin{subfigure}{0.24\textwidth}
\centering
\includegraphics[width=\textwidth]{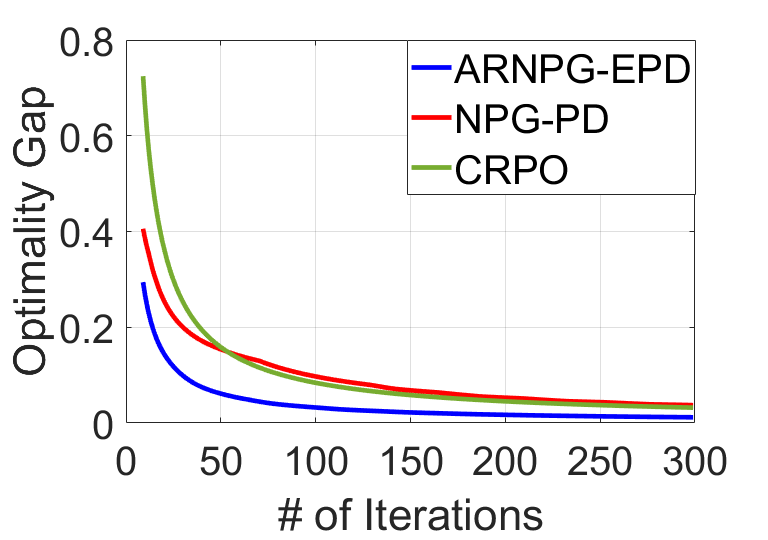}
\caption{}
\end{subfigure}
\begin{subfigure}{0.24\textwidth}
\centering
\includegraphics[width=\textwidth]{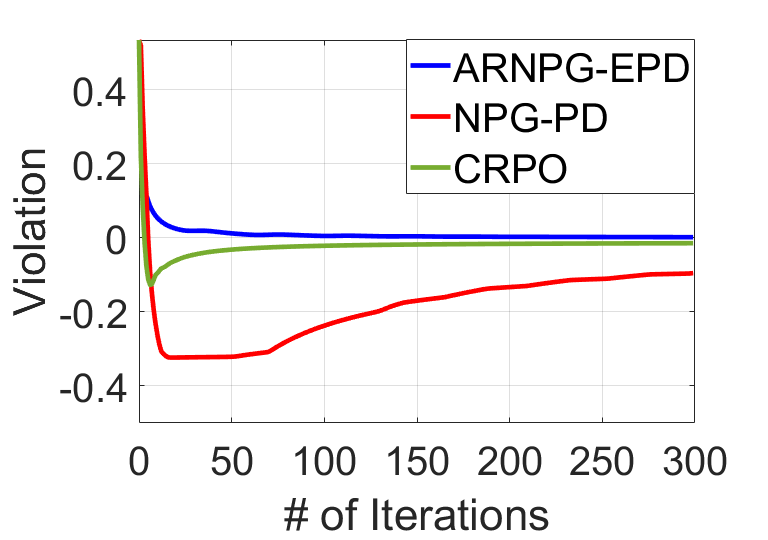}
\caption{}
\end{subfigure}
\begin{subfigure}{0.24\textwidth}
\centering
\includegraphics[width=\textwidth]{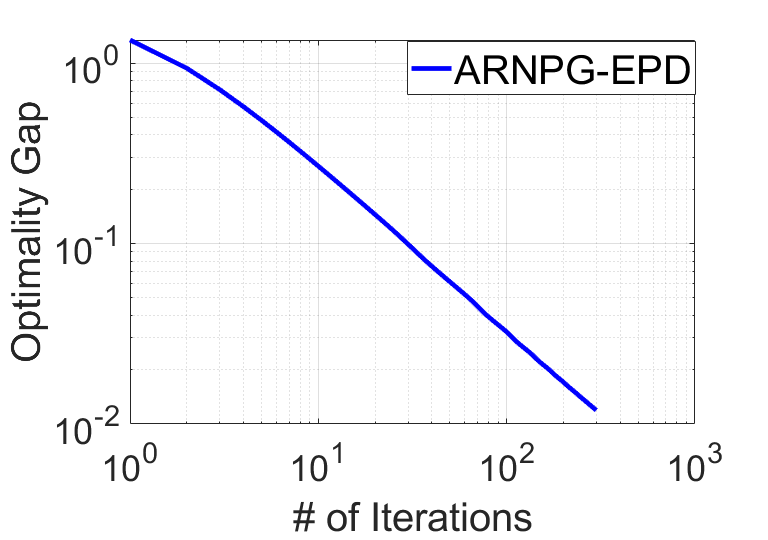}
\caption{}
\end{subfigure}
\begin{subfigure}{0.24\textwidth}
\centering
\includegraphics[width=\textwidth]{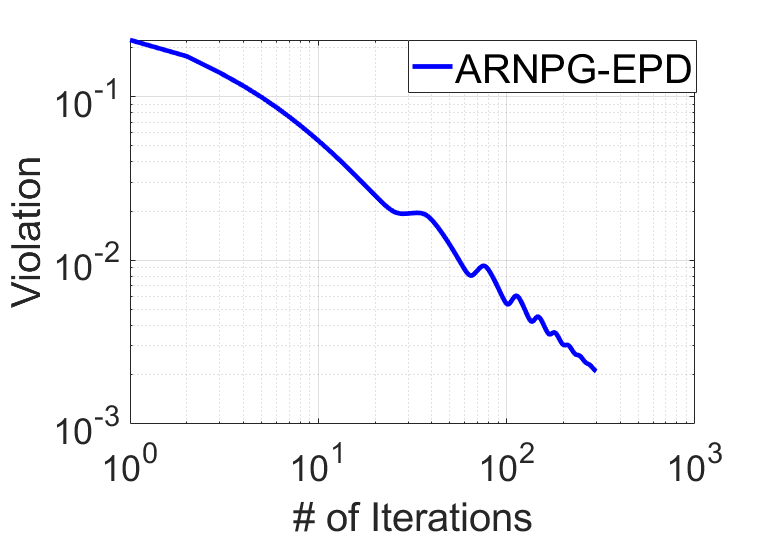}
\caption{}
\end{subfigure}
\caption{The average optimality gap and the average constraint violation versus the total number of iterations, for ARNPG-EPD, NPG-PD, and CRPO on a randomly generated CMDP.}
\label{fig:simple_cmdp}
\end{figure}
%The CMPD is randomly generated with $|\Sc|=20, |\Ac|=10, \gamma=0.8, m=2, \text{ and } b_{2}=3$ \cite{ding2020natural}. 
%We first consider a randomly generated CMDP with $|\Sc|=20, |\Ac|=10, \gamma=0.8, m=2 \text{ and } b_{2}=3$ \cite{ding2020natural}. 

Figures \ref{fig:simple_cmdp}(a) and \ref{fig:simple_cmdp}(b) show that both the average optimality gap and the average constraint violation of the ARNPG-EPD algorithm converge faster than those of NPG-PD. Since the CRPO focuses on the violated constraint, the policy becomes feasible quickly, though at the cost of an initially slower convergence for the optimality gap. As illustrated in Figures \ref{fig:simple_cmdp}(c) and \ref{fig:simple_cmdp}(d), the slopes of both the optimality gap and the constraint violation of the ARNPG-EPD algorithm in the log-log plots are approximately between -0.9 and -1, % in both Figures \ref{fig:simple_cmdp}(a) and \ref{fig:simple_cmdp}(b), 
indicating a converge rate of $\tilde{O}(1/T)$.

\subsection{Sample-based tabular CMDP} \label{sec:cmdp_sample}
\begin{figure}[ht]
\centering
\begin{subfigure}{0.24\textwidth}
\centering
\includegraphics[width=\textwidth]{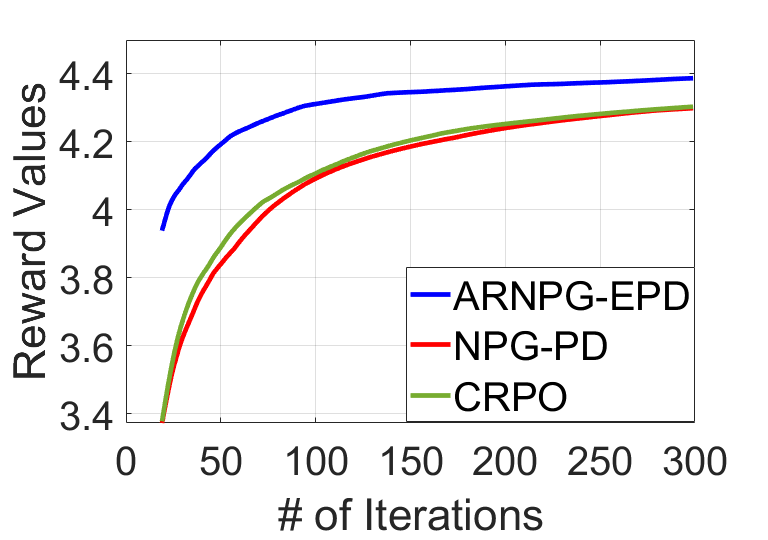}
\caption{}
\end{subfigure}
\begin{subfigure}{0.24\textwidth}
\centering
\includegraphics[width=\textwidth]{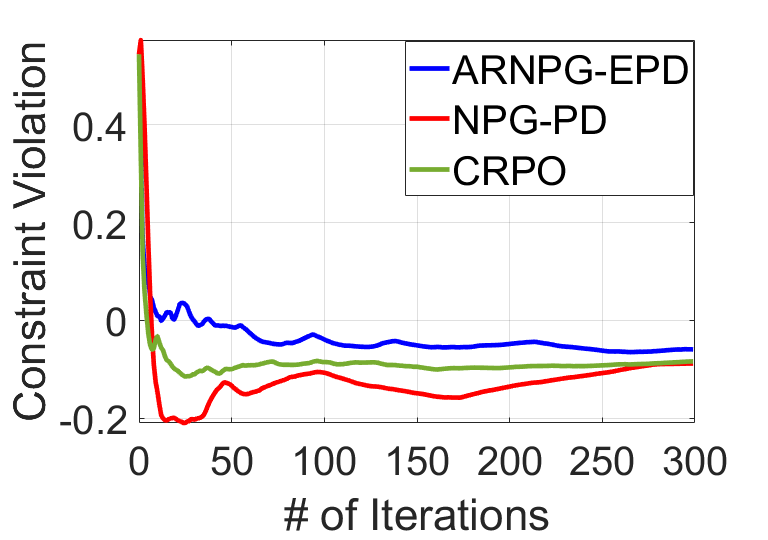}
\caption{}
\end{subfigure}
\begin{subfigure}{0.24\textwidth}
\centering
\includegraphics[width=\textwidth]{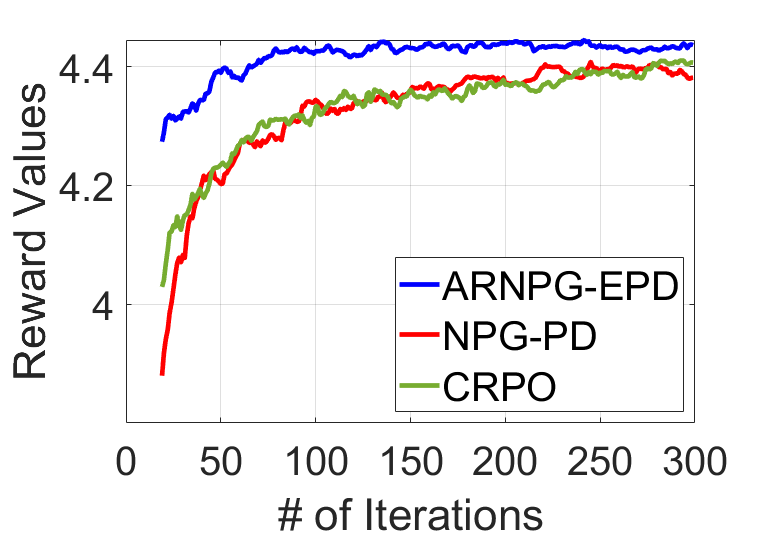}
\caption{}
\end{subfigure}
\begin{subfigure}{0.24\textwidth}
\centering
\includegraphics[width=\textwidth]{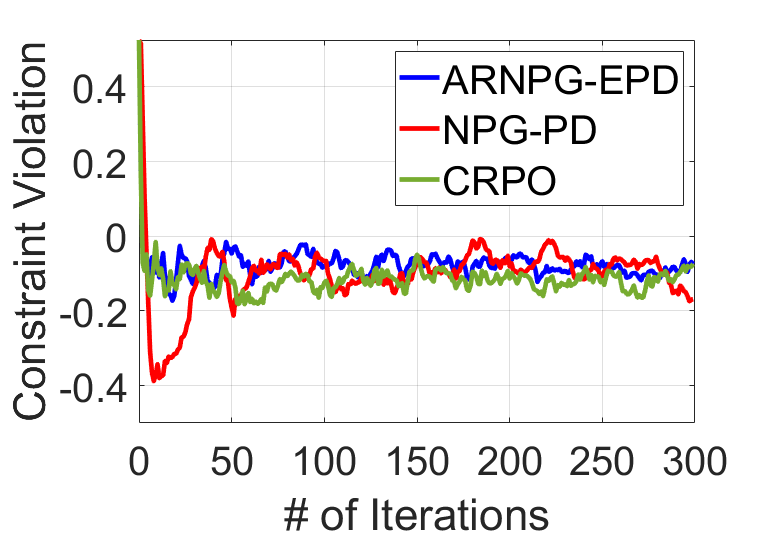}
\caption{}
\end{subfigure}
\caption{The reward values and the constraint violation 
% ((a)\&(b) average values/violation, (c)\&(d) last-iterate values/violation after averaging over 10 random seeds) 
with respect to the total number of iterations, for sample-based ARNPG-EPD, NPG-PD, and CRPO on a randomly generated CMDP.}
\label{fig:simple_cmdp_A}
\end{figure}

We next consider the same tabular CMDP described in Section \ref{sec:cmdp_exact} without exact policy gradients. Instead, policy gradients are estimated by samples from a generative model that can generate independent trajectories starting from any state and action pair. The assumption of such a generative model is common \cite{lan2021policy, ding2020natural, xu2021crpo}. 

The performances of CRPO, NPG-PD, and ARNPG-EPD in the sample-based scenario are shown in Figure \ref{fig:simple_cmdp_A}. %It displays both the optimality gap and the constraint violation versus the number of iterations. 
%The $x$-axis is the total number of iterations. 
Figures \ref{fig:simple_cmdp_A}(a) and \ref{fig:simple_cmdp_A}(b) display the averaged performance, while Figures \ref{fig:simple_cmdp_A}(c) and \ref{fig:simple_cmdp_A}(d) display the performance of the current iterate (a.k.a. last-iterate in optimization literature).
% while the $y$-axis is the reward value or the constraint violation at current iteration. 
It shows that in this sample-based scenario, ARNPG-EPD achieves higher reward values with faster convergence, while all three algorithms satisfy the constraint after a few iterations.

\subsection{Acrobot-v1}
To demonstrate the efficacy of ARNPG-EPD on complex tasks, we have conducted experiments on the Acrobot-v1 environment from OpenAI Gym \cite{1606.01540}.
% To demonstrate the performance of the ARNPG-EPD algorithm on more complex tasks with a large state space and multiple constraints, we conduct experiments on the  Acrobot-v1 from OpenAI Gym \cite{1606.01540}. 
% The acrobot is a planar two-link robotic arm including two joints and two links, where the joint between the links is actuated. 
%where two links are connected and the agent can apply torque on the joint. 
We follow the same experiment setup in \cite{xu2021crpo}, where there is a reward value to maximize, and two cost values to be constrained below some thresholds. The superior performance of ARNPG-EPD is shown in Figure \ref{fig:acrobot}. 
%The objective is to swing the end of the lower link to a given height. The agent suffers costs if the upper link swings in a prohibited direction or when the lower link swings in a prohibited direction w.r.t. the upper link

% The two constraints are to apply torque on the joint (i) when the upper link swings in a prohibited direction, and (ii) when the lower link swings in a prohibited direction with respect to the upper link.

\begin{figure}[ht]
\centering
\begin{subfigure}{0.3\textwidth}
\centering
\includegraphics[width=\textwidth]{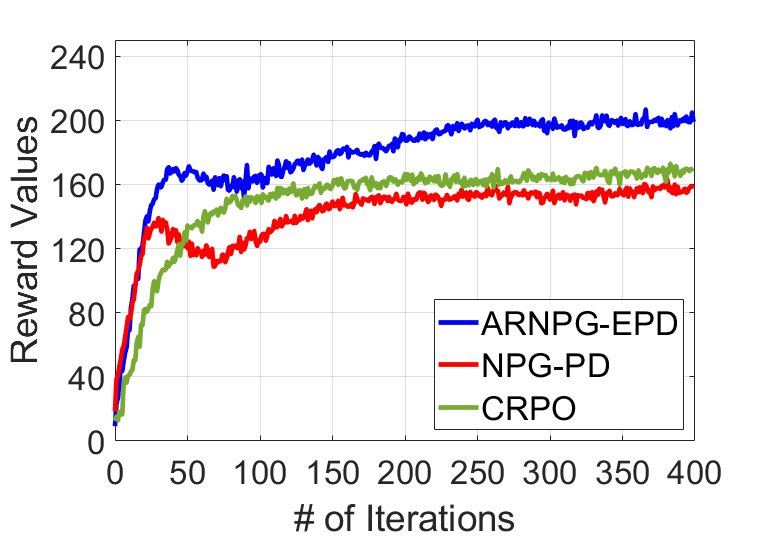}
\caption{}
\end{subfigure}
\begin{subfigure}{0.3\textwidth}
\centering
\includegraphics[width=\textwidth]{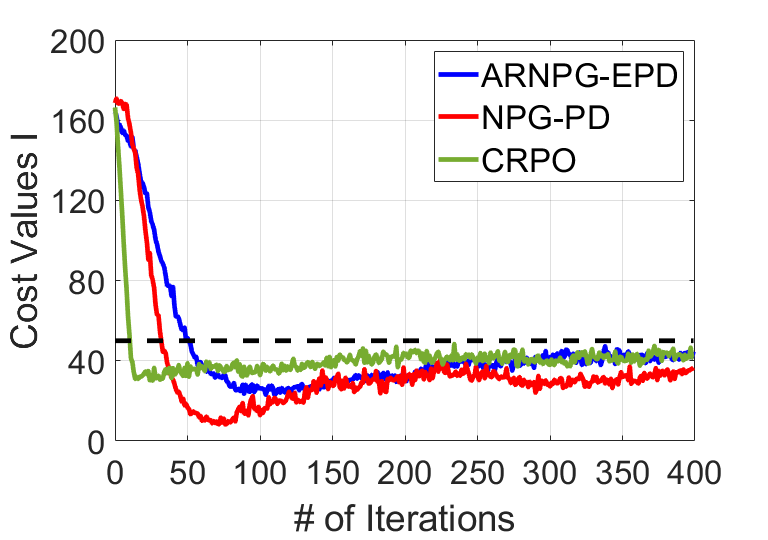}
\caption{}
\end{subfigure}
\begin{subfigure}{0.3\textwidth}
\centering
\includegraphics[width=\textwidth]{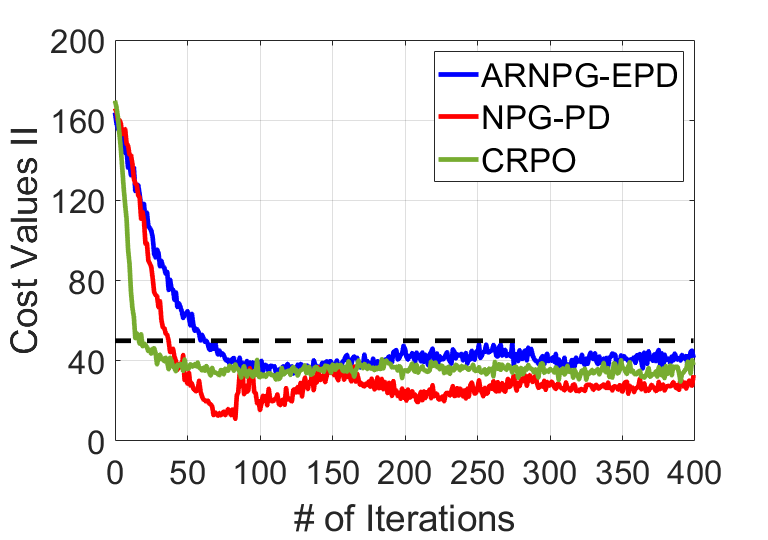}
\caption{}
\end{subfigure}
\caption{Last-iterate performance for sample-based ARNPG-EPD, NPG-PD, CRPO averaged over 10 random seeds. The black dashed lines in (b) and (c) represent given thresholds.}
\label{fig:acrobot}
\end{figure}

Figure \ref{fig:acrobot}(a) shows that ARNPG-EPD achieves a higher reward value compared to NPG-PD and CRPO, while Figures \ref{fig:acrobot}(b) and \ref{fig:acrobot}(c) demonstrate that the cost values of all three algorithms are below the thresholds after a few initial iterations. We believe the superiority is due to the new primal-dual design inspired by \cite{yu2017primal} (discussed in Section \ref{sec:cmdp}) and the flexibility of choosing $t_k$ in the InnerLoop  in the framework. More experiments with different $t_k$ are presented in Appendix \ref{sec:app_exp_cmdp}.

\section{Conclusion and future works} \label{sec:conclusion}
We propose an ARNPG framework to systematically integrate well-performing first-order methods into the design of policy gradient-based algorithms for multi-objective MDPs. The designed algorithms achieve a global $\tilde{O}(1/T)$ convergence rate  under the softmax parameterization with exact gradients, and empirically have satisfactory performance beyond tabular and exact gradient settings. We believe that ARNPG has potential applications in other scenarios, since the general and flexible framework allows integration with more advanced first-order methods, currently and in the future. 

A natural future direction is to extend the theoretical results to more general settings such as function approximation and sample-based scenarios. Viewing ARNPG as a heuristic, the anchor-changing ideas can also be applied to policy optimization for multi-agent RL and meta RL.
%The convergence in the theorems should be improved in terms of smaller $\frac{1}{1-\gamma}$ factors. 

%More general function approximation setting and sample-based setting should be considered. We leave these as future directions.

%future advances of first-order methods can be integrated to 

% We show its advantages both theoretically and empirically. 
%, and thereby attain an $\tilde{O}(1/T)$ global convergence rate with access to exact gradients. 
% It does so without requiring strong assumptions on the multi-objective MDPs. 
% which is an important first step to analyze more complicated scenarios, e.g., sample-based setting. 
%We verify the fast convergence rate by numerical simulations. We also empirically demonstrate the advantage of the proposed algorithms in sample-based settings, including deep RL settings.    

\section{Acknowledgement}
P. R. Kumar's work is based upon work partially supported by the US Army Contracting Command under W911NF-22-1-0151, US Office of Naval Research under N00014-21-1-2385, 4/21-22 DARES: Army Research Office W911NF-21-20064, US National Science Foundation under CMMI-2038625. The views expressed herein and conclusions contained in this document are those of the authors and should not be interpreted as representing the views or official policies, either expressed or implied, of the U.S. Army Contracting Command, ONR, ARO, NSF, or the United States Government. The U.S. Government is authorized to reproduce and distribute reprints for Government purposes notwithstanding any copyright notation herein.

Dileep Kalathil gratefully acknowledges funding from the U.S. National Science Foundation (NSF) grants NSF-CRII-CPS-1850206 and NSF-CAREER-EPCN-2045783.

We thank Dongsheng Ding and Tengyu Xu for generously sharing their code in \cite{ding2020natural, xu2021crpo} as baselines.

%\bibliography{ref}
%\bibliographystyle{plain}

%%%%%%%%%%%%%%%%%%%%%%%%%%%%%%%%%%%%%%%%%%%%%%%%%%%%%%%%%%%%
%%%%%%%%%%%%%%%%%%%%%%%%%%%%%%%%%%%%%%%%%%%%%%%%%%%%%%%%%%%%
\section*{Checklist}

%%% BEGIN INSTRUCTIONS %%%
The checklist follows the references.  Please
read the checklist guidelines carefully for information on how to answer these
questions.  For each question, change the default \answerTODO{} to \answerYes{},
\answerNo{}, or \answerNA{}.  You are strongly encouraged to include a {\bf
justification to your answer}, either by referencing the appropriate section of
your paper or providing a brief inline description.  For example:
\begin{itemize}
  \item Did you include the license to the code and datasets? \answerYes{See Section~\ref{sec:empirical}.}
  \item Did you include the license to the code and datasets? \answerNo{The code and the data are proprietary.}
  \item Did you include the license to the code and datasets? \answerNA{}
\end{itemize}
Please do not modify the questions and only use the provided macros for your
answers.  Note that the Checklist section does not count towards the page
limit.  In your paper, please delete this instructions block and only keep the
Checklist section heading above along with the questions/answers below.
%%% END INSTRUCTIONS %%%

\begin{enumerate}

\item For all authors...
\begin{enumerate}
  \item Do the main claims made in the abstract and introduction accurately reflect the paper's contributions and scope?
    \answerYes{}
  \item Did you describe the limitations of your work?
    \answerYes{See the last paragraphs of Section \ref{sec:cmdp} and \ref{sec:max-min}.}
  \item Did you discuss any potential negative societal impacts of your work?
    \answerNA{}
  \item Have you read the ethics review guidelines and ensured that your paper conforms to them?
    \answerYes{}
\end{enumerate}

\item If you are including theoretical results...
\begin{enumerate}
  \item Did you state the full set of assumptions of all theoretical results?
    \answerYes{After the definition of the problems \eqref{def:smooth-scalar},\eqref{def:cmdp},\eqref{def:minimax} } %See Lines 185-186, 213-215, 248-249.
        \item Did you include complete proofs of all theoretical results?
    \answerYes{See Appendix \ref{appendix:support}, \ref{sec:app_proof_arnpg}, and \ref{sec:proof-theory-app}.}
\end{enumerate}

\item If you ran experiments...
\begin{enumerate}
  \item Did you include the code, data, and instructions needed to reproduce the main experimental results (either in the supplemental material or as a URL)?
    \answerYes{See \href{https://github.com/tliu1997/ARNPG-MORL}{https://github.com/tliu1997/ARNPG-MORL}.}
  \item Did you specify all the training details (e.g., data splits, hyperparameters, how they were chosen)?
    \answerYes{See Appendix \ref{sec:app_exp_cmdp} and \ref{sec:app_exp_momdp}.}
        \item Did you report error bars (e.g., with respect to the random seed after running experiments multiple times)?
    \answerYes{We only report the mean in Figure \ref{fig:acrobot} for clarity. In Appendix \ref{sec:app_exp_cmdp}, we will report error bars.}
        \item Did you include the total amount of compute and the type of resources used (e.g., type of GPUs, internal cluster, or cloud provider)?
    \answerYes{}
\end{enumerate}

\item If you are using existing assets (e.g., code, data, models) or curating/releasing new assets...
\begin{enumerate}
  \item If your work uses existing assets, did you cite the creators?
    \answerYes{}
  \item Did you mention the license of the assets?
    \answerYes{}
  \item Did you include any new assets either in the supplemental material or as a URL?
    \answerYes{See \href{https://github.com/tliu1997/ARNPG-MORL}{https://github.com/tliu1997/ARNPG-MORL}.}
  \item Did you discuss whether and how consent was obtained from people whose data you're using/curating?
    \answerNA{}
  \item Did you discuss whether the data you are using/curating contains personally identifiable information or offensive content?
    \answerNA{}
\end{enumerate}

\item If you used crowdsourcing or conducted research with human subjects...
\begin{enumerate}
  \item Did you include the full text of instructions given to participants and screenshots, if applicable?
    \answerNA{}
  \item Did you describe any potential participant risks, with links to Institutional Review Board (IRB) approvals, if applicable?
    \answerNA{}
  \item Did you include the estimated hourly wage paid to participants and the total amount spent on participant compensation?
    \answerNA{}
\end{enumerate}

\end{enumerate}

\newpage
\appendix

% \tableofcontents

\section{Experimental settings and additional results for CMDP}
\label{sec:app_exp_cmdp}
%We first consider a randomly generated multi-obejctive MDP $(\Sc, \Ac, \gamma)$ $|\Sc|=20, |\Ac|=10, \gamma=0.8, m=2 \text{ and } b_{2}=3$ \cite{ding2020natural}. 

%Since the ARNPG-EPD algorithm has a similar primal-dual structure as the NPG-primal-dual (NPG-PD) algorithm in \cite{ding2020natural}, we follow the same experimental setting and the hyperparameter 

\subsection{Tabular CMDP} \label{app:tabular_cmdp}
The tabular CMDP, for both exact-gradient scenario (Section \ref{sec:cmdp_exact}) and sample-based scenario (Section \ref{sec:cmdp_sample}), follows the same experimental setting as in \cite{ding2020natural}. 

% The MDP with $m=2$ objectives, where $|\Sc|=20$, $|\Ac|=10$, $\rho$ is uniform distribution, $\gamma=0.8$, and $P(\cdot|s,a)$, $r_1(s,a), r_2(s,a)$ are randomly generated. 
The MDP with $m = 2$ objectives represented by $(\Sc, \Ac, P, \rho, \gamma, r_{1:2})$ (as the system model in Section \ref{sec:prelim}) is randomly generated, where $|\Sc| = 20$, $|\Ac| = 10$, $\rho$ is uniform distribution, and $\gamma=0.8$. For each $(s,a) \in \Sc \times \Ac$, $P(\cdot|s,a) \in \Delta(\Sc)$ is generated by normalizing a random vector $\sim Unif([0, 1]^{\Sc})$, and independent rewards $r_1(s,a), r_2(s,a) \sim Unif([0,1])$. Choosing the constraint coefficient $b_2 = 3$, the experiments are performed on the CMDP
\begin{align}
    \max_{\theta \in \Theta} \ V_{r_1}^{\pi_\theta}(\rho) \quad \text{s.t.}\ ~V_{r_2}^{\pi_\theta}(\rho) \geq b_2,
\end{align}
with the softmax policy class.

For both the exact-gradient scenario (Section \ref{sec:cmdp_exact}) and the sample-based scenario (Section \ref{sec:cmdp_sample}), we choose $\eta=1$ and $\eta'=1$ for ARNPG-EPD and NPG-PD (following the same hyperparameter selection as in \cite{ding2020natural}), since both rely on a primal-dual framework. Additionally, we fix $t_k=1, \forall k = 0, 1, \dots, K-1$ and select $\alpha = \frac{1-\gamma}{\eta} = 0.2$ for ARNPG-EPD. As for CRPO with exact gradients, we first fix the tolerance parameter as 0.01 and then choose the best learning rate 0.4 from the set $\{0.1, 0.2, \dots, 0.9, 1.0\}$, which enjoys the smallest average optimality gap after 300 iterations. For sample-based CRPO, we select the best learning rate 1.0 from the set $\{0.1, 0.5, 1, 2, 5\}$, which leads to the largest reward value after 300 iterations.

% \textcolor{blue}{$\rho$ not coverable. Larger $t_k$.}

% For the tabular CMDP, we choose $\eta=1$ and $\eta'=1$ for ARNPG-EPD and NPG-PD, which follow the hyperparameter choices in \cite{ding2020natural}. Furthermore, we fix $t_k=1, \forall k = 0, 1, \dots, K-1$ and choose $\alpha = \frac{1-\gamma}{\eta} = 0.2$ for ARNPG-EPD. For CRPO, we first fix the tolerance parameter as 0.01 and then select the best learning rate 1.0 from the set $\{0.5, 1, 2, 5\}$, which enjoys the largest reward values after 300 iterations.

\subsection{Acrobot-v1}
To demonstrate the performance of the ARNPG-EPD algorithm on more complex tasks with a large state space and multiple constraints, we conduct experiments on the Acrobot-v1 from OpenAI Gym \cite{1606.01540}. The acrobot is a planar two-link robotic arm with two joints and two links, where the joint between the links is actuated. 

We follow the same experimental setup as in \cite{xu2021crpo}, where the task is to swing the end of the lower link to a given height, and there are two costs (objectives) corresponding to (i) the upper link swinging in a prohibited direction; (ii) the lower link swinging in a prohibited direction w.r.t. the upper link. The upper bound constraints on the cost values are $(50, 50)$. An actor-critic framework is adopted, where the actor is defined by a neural softmax policy with two hidden layers of widths $(128, 128)$, and there are 3 critics (one for each of the value functions) also parameterized by neural networks with two hidden layers of widths $(128, 128)$. 
% The two constraints are to apply torque on the joint (i) when the upper link swings in a prohibited direction, and (ii) when the lower link swings in a prohibited direction with respect to the upper link.
% where two links are connected and the agent can apply torque on the joint. 
%The NPG updates (required by all three algorithms ARNPG-EPD, NPG-PD, CRPO) are implemented by the trust region policy optimization (TRPO) \cite{schulman2015trust}. %\footnote{Note that TRPO is implemented via penalty and linear-quadratic approximation for the KL-divergence, it is equivalent to NPG.}
% For fairness of comparisons, all algorithms are based on trust region policy optimization (TRPO) \cite{schulman2015trust} and the same neural softmax policy parameterization with two hidden layers of size $(128, 128)$. 
% \prk{employing?} 

% Given that the exact gradients is no longer accessible for Acrobot-v1, we will adopt the sample-based versions for all algorithms, i.e., using empirical estimates of gradients. 
Figure \ref{fig:acrobot} provides the performances of the last-iterates generated by the algorithms, averaged over 10 random seeds, where the step size of the dual update (i.e., 0.0005) is tuned from the set $\{0.00001, 0.0005, 0.001, 0.005, 0.01, 0.05\}$, and the tolerance parameter 0.5 of CRPO follows from \cite{xu2021crpo}. For ARNPG-EPD in Figure \ref{fig:acrobot}, we choose $t_k=2, \forall k = 0, 1, \dots, K-1$ and $\alpha = 1$. To provide more information about variance, Figure \ref{fig:acrobot_shade} includes shaded error bars ($\pm$ standard deviation). 
% \prk{Is shaded region 1-sigma or 3-sigma? Specify} 
%based on Figure \ref{fig:acrobot}. 

% \vspace{-0.1in}
\begin{figure}[ht]
\centering
\begin{subfigure}{0.3\textwidth}
\centering
\includegraphics[width=\textwidth]{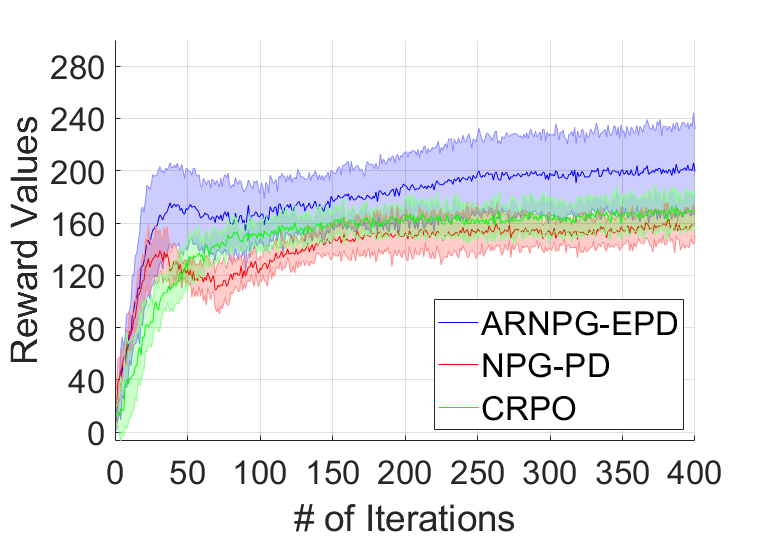}
% \caption{}
\end{subfigure}
\begin{subfigure}{0.3\textwidth}
\centering
\includegraphics[width=\textwidth]{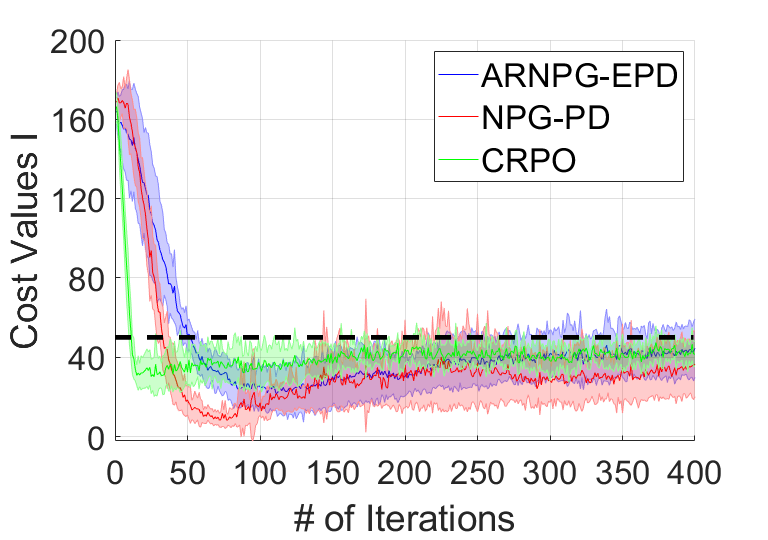}
% \caption{}
\end{subfigure}
\begin{subfigure}{0.3\textwidth}
\centering
\includegraphics[width=\textwidth]{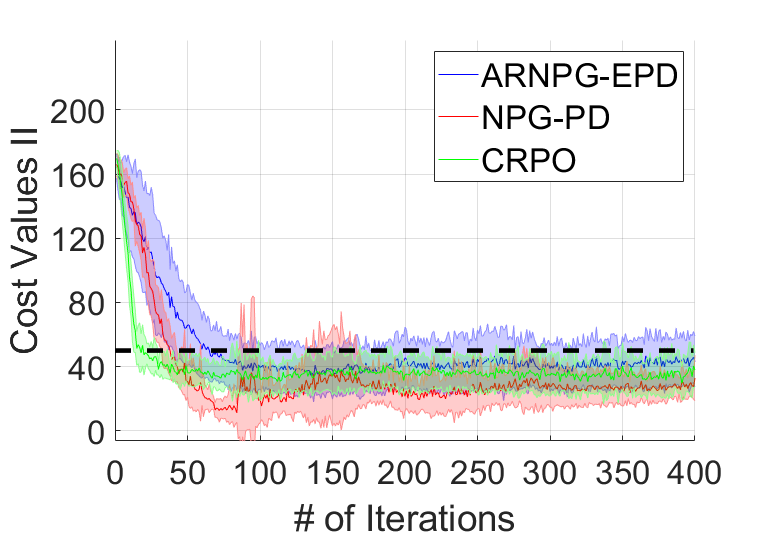}
% \caption{}
\end{subfigure}
\caption{Last-iterate performance for sample-based ARNPG-EPD, NPG-PD, CRPO averaged over 10 random seeds with shaded error bars. The black dashed lines represent given thresholds.}
\label{fig:acrobot_shade}
\end{figure}
% \vspace{-0.1in}

To illustrate the impact of the number of InnerLoop iterations, we fix $\alpha = 1$ and take $t_k = 1, 2, 5, 10$ respectively, as shown in Figure \ref{fig:acrobot_inner}. We omit the shaded error bars as in Figure \ref{fig:acrobot_shade} for better visualization and comparison of the mean performance of ARNPG-EPD algorithms under different $t_k$s.
%To show the impact of our framework on the newly designed Lagrangian, we fix $\alpha = 1$ and choose InnerLoop iterations as $t_k = 1, 2, 5, 10$ respectively. 
%To make the figure look clearer, \prk{Why does this increase ``clarity"?} we will only show the mean over 10 random seeds.
Figure \ref{fig:acrobot_inner} demonstrates the trade-off induced by the selection of the hyperparameter $t_k$. %When $t_k$ is large, ARNPG-EPD follows more closely to its underlying EPD algorithm (for convex optimization). Following the underlying EPG algorithm may lead to good reward objective (e.g., the final reward is large for $t_k=10$), but it suffers slower convergence (e.g., it takes more time to become feasible for $t_k=10$) since large $t_k$ is required to follow EPD algorithm closely.
When $t_k$ is small, e.g., $t_k = 1$, ARNPG-EPD cannot follow the underlying EPD algorithm (for convex optimization) closely. The curve is not as stable as the others, but is more adaptive (becomes feasible first) because it does not spend too much time on the KL-regularized MDP with reward $\tilde{r}_k$ in InnerLoop.
%proposed framework will degrade to solving an unregularized CMDP, which has the lowest reward values. 
% \textcolor{purple}{[Tian: where do we look to make this comparison? In Fig. 3? If so, shouldn't we merge Fig 3 and Fig. 4 so that the readers can see?]}) 
When $t_k$ is large, e.g., $t_k=5$ or $t_k = 10$, the curves are more stable.  %since larger $t_k$ follows the underlying EPD algorithm more closely. 
However, the convergence of the algorithms is relatively slow (become feasible slowly), since larger $t_k$ may waste computation on the regularized objective $\tilde{V}_{k,\alpha}^{\pi_\theta}(\rho)$ instead of focusing on the true optimization program. 

In Acrobot-v1, we find $t_k = 2$ to be the best choice empirically. Note that at iteration 300, returns of ARNPG-EPD algorithms are all feasible, and even the lowest mean reward value of 181.8  (ARNPG-EPD with $t_k = 1$), is higher than the mean reward values of NPG-PD and TRPO in Figure \ref{fig:acrobot}. 

% \vspace{-0.1in}
\begin{figure}[ht]
\centering
\begin{subfigure}{0.3\textwidth}
\centering
\includegraphics[width=\textwidth]{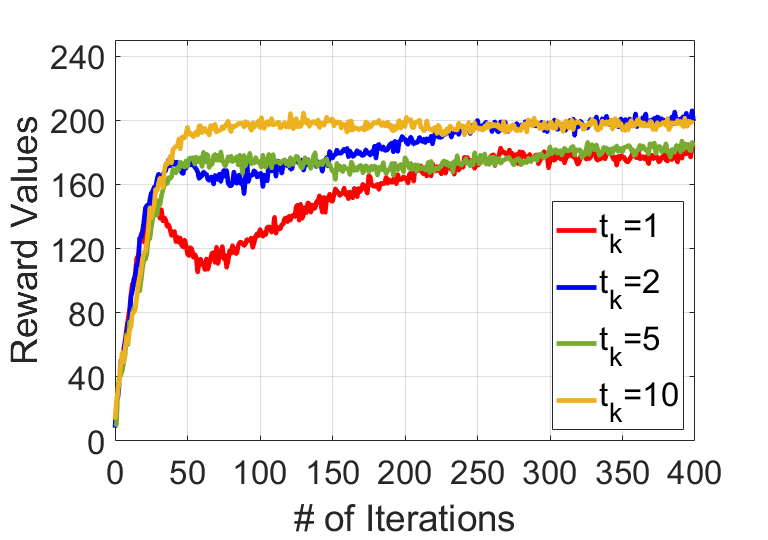}
% \caption{}
\end{subfigure}
\begin{subfigure}{0.3\textwidth}
\centering
\includegraphics[width=\textwidth]{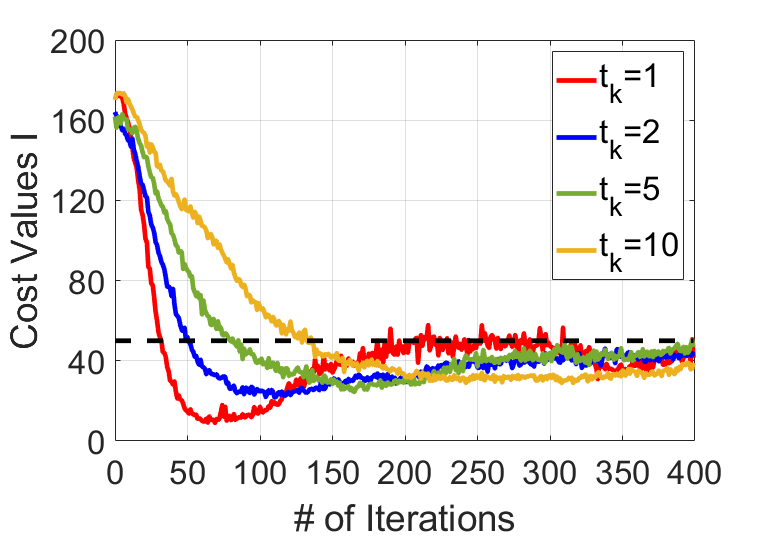}
% \caption{}
\end{subfigure}
\begin{subfigure}{0.3\textwidth}
\centering
\includegraphics[width=\textwidth]{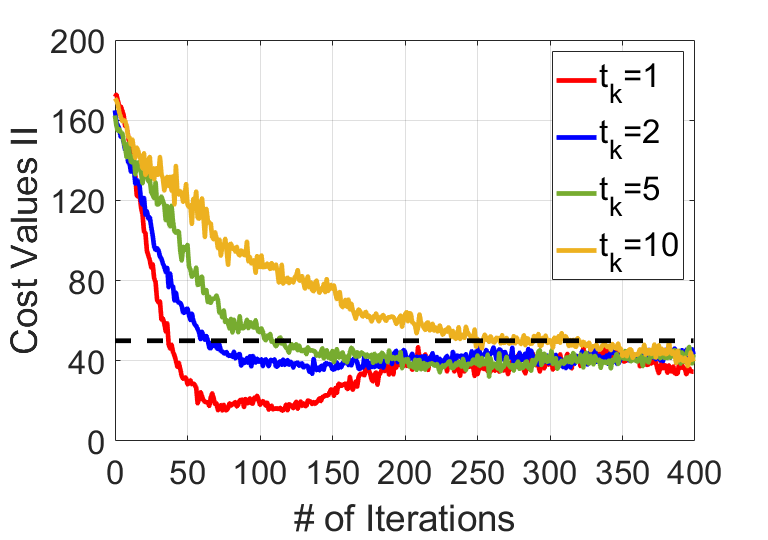}
% \caption{}
\end{subfigure}
\caption{Last-iterate performance for sample-based ARNPG-EPD with different inner loops ($t_k = 1, 2, 5, 10$) averaged over 10 random seeds versus the total number of iterations.}
\label{fig:acrobot_inner}
\end{figure}

\subsection{Hopper-v3}
We now demonstrate the performance of our approach on Hopper-v3, a more complex robotics control task, with a constraint of moving speed 82.748 \cite{zhang2020first}. Hopper-v3 is implemented via the OpenAI Gym based on the MuJoCo physical simulators \cite{todorov2012mujoco}, where both the state space and the action space are continuous. We choose the current SOTA of this task, namely the FOCOPS (First Order Constrained Optimization in Policy Space) algorithm \cite{zhang2020first}, for a comparison with our approach. Since the policy update of the FOCOPS algorithm is based on the PPO (Proximal Policy Optimization) algorithm, for a fair comparison, we also revise our algorithm to a corresponding version called ``ARPPO-EPD", where the NPG update is replaced by PPO. \emph{It is also an illustration that the anchor-changing idea is readily to be combined with other policy gradient methods.} We compare performance with other baselines NPG-PD \cite{ding2020natural} and CRPO \cite{xu2021crpo}.

Figure \ref{fig:hopper} reports performance averaged over 5 random seeds and illustrates that our ARPPO-EPD algorithms achieve higher reward values compared with other methods, while achieving similar constraint values. (Note that 1, 2, 5 in parentheses of ARPPO-EPD represent the number of inner loops $t_k=1, 2, 5$.)
\begin{figure}[ht]
\centering
\begin{subfigure}{0.45\textwidth}
\centering
\includegraphics[width=\textwidth]{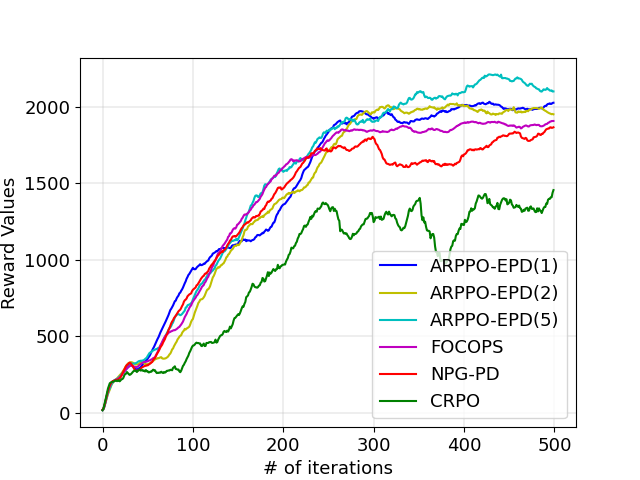}
% \caption{}
\end{subfigure}
\begin{subfigure}{0.45\textwidth}
\centering
\includegraphics[width=\textwidth]{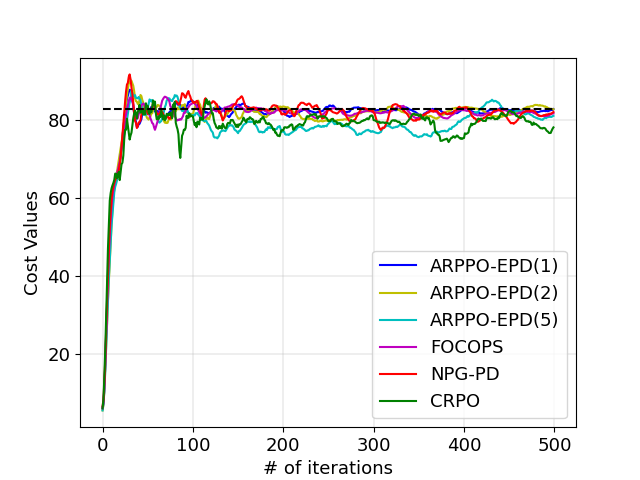}
% \caption{}
\end{subfigure}
\caption{Last-iterate performance for sample-based ARPPO-EPD with different inner loops ($t_k=1, 2, 5$), FOCOPS, NPG-PD, and CRPO on Hopper-v3 with a constraint on the moving speed 82.748.}
\label{fig:hopper}
\end{figure}

\section{Experimental results for multi-objective MDP with scalarization}
\label{sec:app_exp_momdp}
In this section, we numerically study the performance of the ARNPG-guided algorithms under the smooth concave scalarization and max-min trade-off, respectively. 
% The problem of MOMDP under smooth scalarization has been studied by Bai et. al. \cite{bai2021joint}. There algorithm can be viewed as a special case of ARNPG-IMD with $t_k = 1$ for any $k = 0, 1, \dots, K-1$. Thus in the following experiments, we are particularly interested in the impact of the number of micro steps $t_k$ in InnerLoop on the performance. 
%The smooth concave scalarization of interest is the approximate proportional fairness sum-of-logarithms function.
We first consider the tabular setting for exact-gradient and sample-based scenarios with softmax policy parameterization, and then study the Acrobot-v1 scenario with neural softmax policy parameterization.

\subsection{Tabular multi-objective MDP with exact gradients} \label{sec:app_momdp_oracle}
For the tabular setting, we follow the same MDP construction with $m=2$ objectives as in Section \ref{app:tabular_cmdp}. 
The goal of considering exact gradients is to empirically study the theoretically established $O(1/K)$ convergences of the proposed ARNPG-IMD and ARNPG-OMDA algorithms suggested by Theorems \ref{thm:ARNPG-IMD} and \ref{thm:ARNPG-OMDA}. We will show that the proposed algorithms maintain $O(1/K)$ convergence in terms of the macro steps $K$, for different numbers of micro steps $t_k$.

\subsubsection{Smooth concave scalarization} \label{sec:exp_smooth}

An example of interest is the sum-logarithmic function, which for two objectives is
% . Consider the simplest case of sum-logarithmic, i.e., the scalarization function of two objectives
\begin{align}
\label{eqn:sum_log}
    F(V_{r_1}^{\pi}(\rho), V_{r_2}^{\pi}(\rho)) = \log(V_{r_1}^{\pi}(\rho) + \delta) + \log(V_{r_2}^{\pi}(\rho) + \delta),
\end{align}
where $\delta > 0$ is a small constant. This can be viewed as an approximate formulation of the proportional fairness criteria. If $\Vc$ is convex and with $v^*$ denoting the optimal value vector achieving the largest $F(v)$, it can be verified by the first-order optimality condition that $\frac{v_1 - v^*_1}{v^*_1 + \delta} + \frac{v_2 - v^*_2}{v^*_2 + \delta} \leq 0$, $\forall v_{1:2} \in \Vc$.

%We generate a random MOMDP with $|\Sc|=20, |\Ac|=10, \gamma=0.8, \text{ and } m = 2$ and demonstrate the performance of the proposed ARNPG-IMD algorithm with different inner loops ($t_k = 1, 2, 5, 10$). 

% Note that if $t_k=1$, then the ARNPG-IMD algorithm is equivalent to the multi-objective NPG (MO-NPG) algorithm, which is a multi-objective version of the original NPG algorithm \cite{kakade2001natural} and has been studied in \cite{bai2021joint}. 

% Let $t_k$ be the number of inner loops at the macro step $k$, $\forall k = 0, 1, \dots, K-1$. 
% Therefore, we choose $t_k = 10, \forall k = 0, 1, \dots, K-1$. 

As mentioned in Section \ref{sec:ARNPG}, when $t_k = 1$, 
%the gradient $\nabla_\theta \tilde{V}_{k,\alpha}^{\pi_{\theta_k}}(\rho) = \nabla_\theta \tilde{V}_{\tilde{r}_k}^{\pi_{\theta_k}}(\rho)$, since $D_{d^{\pi_\theta}_\rho}(\pi_\theta || \pi_{\theta_k})$ has zero gradient at $\theta = \theta_k$. 
the update in \eqref{eqn:inner-NPG} reduces to an NPG update on the unregularized value function $V_{\tilde{r}_k}^{\pi_\theta}(\rho)$. In other words, when $t_k = 1$, $\alpha$ has no impact on the ARNPG-IMD algorithm. 

Theorem \ref{thm:ARNPG-IMD} suggests a lower bound on the number of micro steps $t_k$. As with many existing algorithms in the optimization literature, the theoretically chosen hyperparameters are usually too conservative, and the convergence rate is maintained over a wider range of hyperparameters. We set up the experiments as follows. 
%To study the impact of the number of micro steps $t_k$ in InnerLoop, w
We first fix $t_k = 1$ and conduct experiments with learning rates $\eta \in \{0.5, 1.0, \dots, 9.5, 10\}$. The hyperparameter achieving the smallest average optimality gap is $\eta = 4.5$. We then fix the learning rate $\eta = 4.5$, choose the regularization parameter $\alpha = 0.01$, and show the convergence of ARNPG-IMD with $t_k \in \{1, 2, 5, 10\}$ in Figure \ref{fig:simple_smooth}. Though the learning rate $\eta=4.5$ and the regularization parameter $\alpha$ are not chosen to favor $t_k > 1$, it can still be observed that larger $t_k$ leads to faster convergence in terms of the number of macro steps ($K$). 
% for MO-NPG and choose the best hyperparameter 4.5 since its corresponding average optimality gap is the smallest after 300 iterations. To discover the impact of the number of micro steps $t_k$, we fix learning rate $\eta = 4.5$ and regularization parameter $\alpha = 0.01$. 
% Similarly, we do a grid-search over $\alpha \in \{0.001, 0.005, 0.01, 0.02, 0.05\}$ and $\eta \in \{1, 2, 5, 10\}$ for ARNPG-IMD and select the best hyperparameter $\alpha = 0.005$ and $\eta = 5$.

\begin{figure}[ht]
\centering
\begin{subfigure}{0.45\textwidth}
\centering
\includegraphics[width=\textwidth]{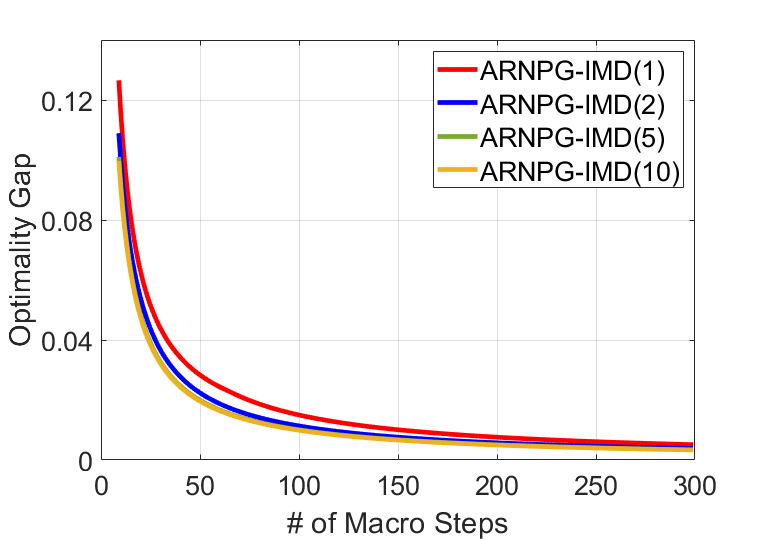}
\caption{}
\end{subfigure}
\begin{subfigure}{0.45\textwidth}
\centering
\includegraphics[width=\textwidth]{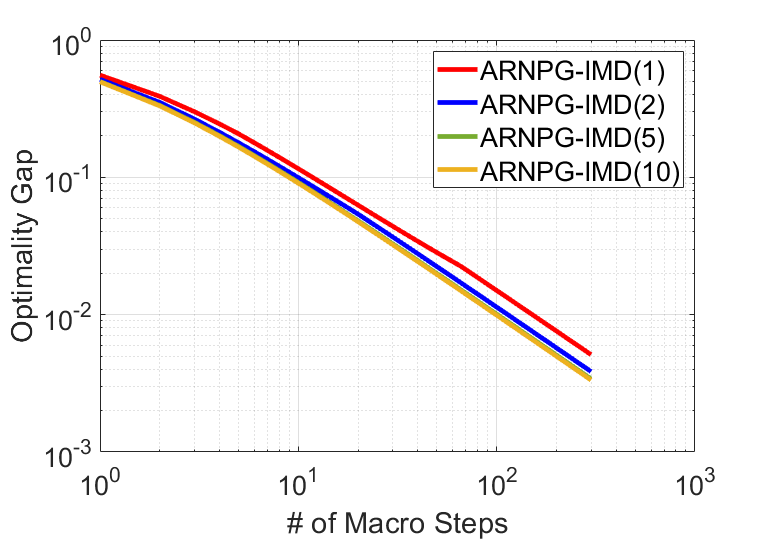}
\caption{}
\end{subfigure}
\caption{The average optimality gap versus the number of macro steps, for ARNPG-IMD with different number of inner loop iterations ($t_k = 1, 2, 5, 10$) on a randomly generated sum-logarithmic two-objective MDP.}
\label{fig:simple_smooth}
\end{figure}
% \vspace{-0.1in}
%Recall that under softmax policy with exact gradients, Theorem \ref{thm:ARNPG-IMD} guarantees $O(1/K)$ convergence of average optimality gap. 

% The horizontal axis is the number of iterations (macro steps). 
Figure \ref{fig:simple_smooth}(a) illustrates that all ARNPG-IMD algorithms with different $t_k$s enjoy fast convergence rates and small optimality gaps. Additionally, the log-log plot in Figure \ref{fig:simple_smooth}(b) indicates a slope of approximately -1, which confirms the $O(1/K)$ convergence of ARNPG-IMD. %since $T = \sum_{k=0}^{K-1} t_k = \Theta(K)$ in experiments.

% \begin{figure}
% \centering
% \begin{subfigure}{0.49\textwidth}
% \centering
% \includegraphics[width=\textwidth]{figure/gap_300_inner_best.png}
% \caption{}
% \end{subfigure}
% \begin{subfigure}{0.49\textwidth}
% \centering
% \includegraphics[width=\textwidth]{figure/gap_300_last.png}
% \caption{}
% \end{subfigure}
% \caption{The optimality gap (last-iterate gap) versus the number of iterations ($t$), for MO-PMD/MO-PMP and MO-NPG on a randomly generated two-objective MOMDP with (a) sum-logarithmic (\ref{eqn:sum_log}) and (b) max-min trade-off (\ref{eqn:minimax}).}
% \label{fig:simple_smooth_last}
% \end{figure}

\subsubsection{Max-min trade-off}
\label{sec:exp_non_smooth}
Another multi-objective MDP of interest is max-min fairness, corresponding to a robust scalarization function, which for the case of two objectives is
\begin{align}
\label{eqn:minimax}
F(V_{1}^{\pi}(\rho), V_{2}^{\pi}(\rho)) = \min_{i} V_i^{\pi}(\rho) = \min_{\lambda \in \Delta([2])} \langle V_{1:2}^{\pi}(\rho), \lambda\rangle.
\end{align}
% One may note that the MO-PMP algorithm gains more advantage over the MO-NPG algorithm as the number of objectives increases since the proposed mirror-prox can avoid oscillating among multiple corner points.
The function $F$ is concave but not always differentiable. We can employ a subgradient ascent-based NPG algorithm, which calculates the subgradient of $F$ and performs one-step of NPG for the induced reward function. We call such an algorithm multi-objective NPG (MO-NPG). 

% We generate the same MDP with $m=2$ objectives and $|\Sc|=20, |\Ac|=10, \gamma=0.8$ similar to the 
% and compare the performances of MO-NPG with those of the proposed ARNPG-OMDA algorithms under different inner loops ($t_k = 1, 2, 5, 10$). 
\begin{figure}[ht]
\centering
\begin{subfigure}{0.45\textwidth}
\centering
\includegraphics[width=\textwidth]{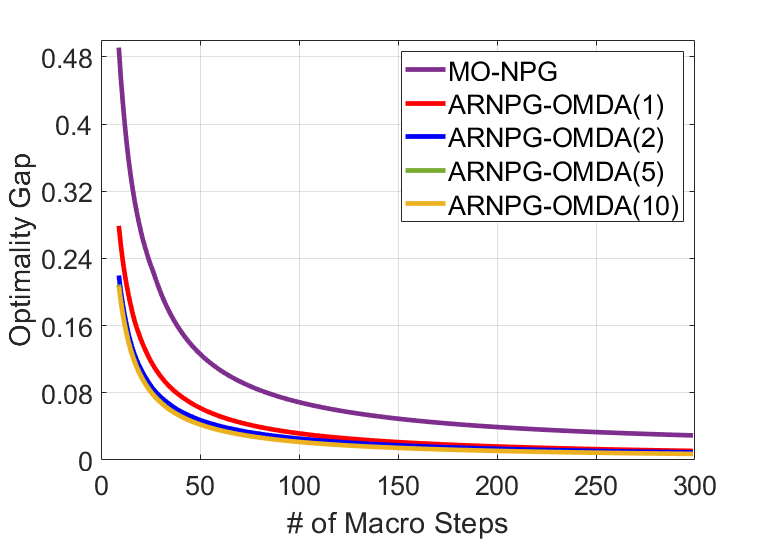}
\caption{}
\end{subfigure}
\begin{subfigure}{0.45\textwidth}
\centering
\includegraphics[width=\textwidth]{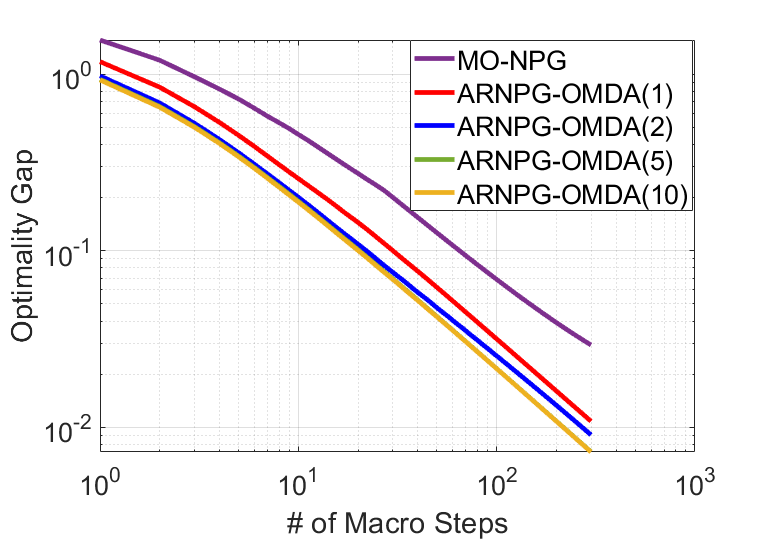}
\caption{}
\end{subfigure}
\caption{The average optimality gap with respect to the number of macro steps, for MO-NPG and ARNPG-OMDA with different number of inner loop iterations ($t_k = 1, 2, 5, 10$) on a randomly generated max-min two-objective MDP (\ref{eqn:minimax}).}
\label{fig:minimax}
\end{figure}
% \vspace{-0.1in}

%Unlike smooth concave scalarization, the ARNPG-OMDA algorithm is different from the MO-NPG algorithm even if $t_k = 1$ due to the existence of an extragradient step. 
We conduct the experiments under learning rates $\{0.8, 0.81, \dots, 1.18, 1.19\}$ for MO-NPG and choose the best hyperparameter, 0.93. Similarly, we do a grid-search for ARNPG-OMDA ($t_k = 1$) over $\alpha \in \{1, 2, 5\}, \eta \in \{0.06, 0.08, 1.0\}, \eta' \in \{0.5, 1.0, 1.5, 2.0\}$ and select the best hyperparameters, $\alpha = 1, \eta = 0.08$, and $\eta' = 2$. Then, we fix $\alpha, \eta, \eta'$ to explore the impact of $t_k > 1$.
% for fairness, we choose $\eta=0.5$ for MO-PMP, and $\eta' = 0.1$ and $t_k = 5, \forall k = 0, 1, \dots, K-1$ for MO-PMP. 
Figure \ref{fig:minimax} shows that the ARNPG-OMDA algorithms converge faster than MO-NPG due to the better underlying optimization algorithm OMDA compared to the subgradient ascent. Moreover, larger $t_k$ gives faster convergence in terms of macro steps $K$, even though the parameters are 
% \sout{chosen not specifically favoring} 
not specifically chosen to favor $t_k > 1$.

Recall that under softmax policy with exact gradients, Theorem \ref{thm:ARNPG-OMDA} guarantees $O(1/K)$ convergence of the average optimality gap. Due to the underlining subgradient ascent of MO-NPG, the convergence rate can only be guaranteed by $O(1/\sqrt{K})$. 
% The horizontal axis is the number of iterations (macro steps). 
Figure \ref{fig:minimax}(a) shows that the optimality gap of the ARNPG-OMDA algorithms converges faster than that of the MO-NPG algorithm. The
corresponding log-log plots in Figure \ref{fig:minimax}(b) shows that the
slopes are around $-1$ for the ARNPG-OMDA algorithms, which demonstrates that 
% the corresponding log-log plot, which can analyze the convergence rate more clearly by the slope. Figure \ref{fig:minimax}(b) shows that the slope of the MO-NPG algorithm is around -0.5, while the slope of the MO-PMP algorithm is around -0.8, which means that 
the optimality gap of ARNPG-OMDA converges at an $O(1/K)$ rate. 

% since $T = \sum_{k=0}^{K-1} 2 t_k = \Theta(K)$ in experiments.
% In Appendix \ref{sec:exp_appendix}, we show the comparative performance of the sample-based, rather than oracle-based,
% % \prk{Please check in the literature if it Oracle is written with a capitalized O, or if it is written as oracle.} 
% MO-PMD and MO-PMP algorithms on the same network system.

\subsection{Sample-based tabular multi-objective MDP}
We next consider the same tabular MDP with $m=2$ objectives, but with
the gradients estimated by samples from a generative model that can generate independent trajectories starting from any state and action pair. 
% For fairness, we sample 20 independent trajectories for each iteration of all algorithms.
% We study the impact of $t_k$ 
% \textcolor{purple}{[Tian: "high" reward value, instead of "larger" reward value]}

\subsubsection{Smooth concave scalarization}
% To distinguish between MO-PMD and MO-NPG, we choose $t_k = 2, \forall k = 0, 1, \dots, K-1$. 
%Recall that ARNPG-IMD will be equivalent to MO-NPG when $t_k=1$. 
We conduct the experiments with learning rates $\{0.5, 1.0, \dots, 4.0, 4.5\}$ for sample-based ARNPG-IMD with $t_k = 1$ and choose its best hyperparameter 1.5. To discover the impact of the number of micro steps $t_k$, we fix the learning rate $\eta = 1.5$ and the regularization parameter $\alpha = 0.005$. 
% Similarly, we do a grid-search over $\alpha \in \{0.005, 0.01, 0.05\}$ and $\eta \in \{0.1, 0.5, 1\}$ for sample-based ARNPG-IMD and select the best hyperparameter $\alpha = 0.005$ and $\eta = 1$. 
Figure \ref{fig:simple_momdp_A}(a) demonstrates that $t_k > 1$ is helpful to achieve a faster convergence (i.e., during the first 50 iterations, larger $t_k$ leads to a higher scalarized objective), though after 250 iterations they all converge to the same optimal value. %It is because the environment is relatively simple and when the values are approximately optimal. 
% converge within several 
% \prk{What is the conclusion the reader is supposed to draw? That ``several" is good? If so, say it more positively: ``within ... iterations"} 
% iterations and $t_k > 1$ is helpful to achieve higher reward values. 

\begin{figure}[ht]
\centering
\begin{subfigure}{0.45\textwidth}
\centering
\includegraphics[width=\textwidth]{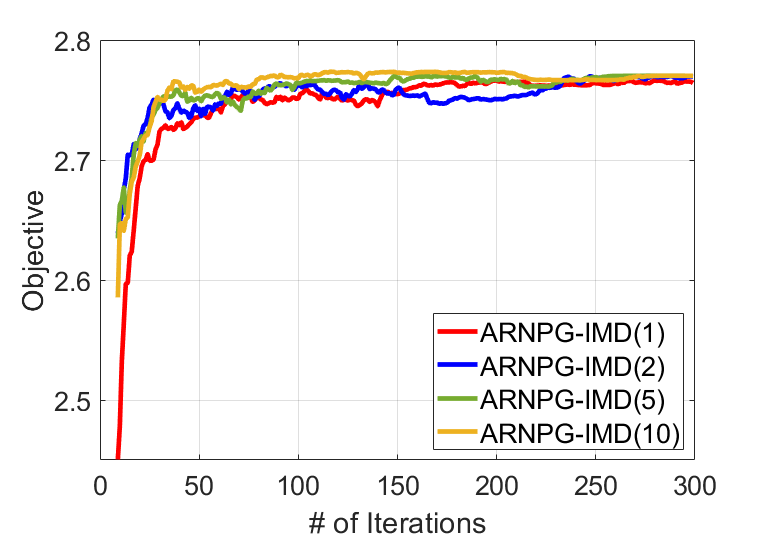}
\caption{}
\end{subfigure}
\begin{subfigure}{0.45\textwidth}
\centering
\includegraphics[width=\textwidth]{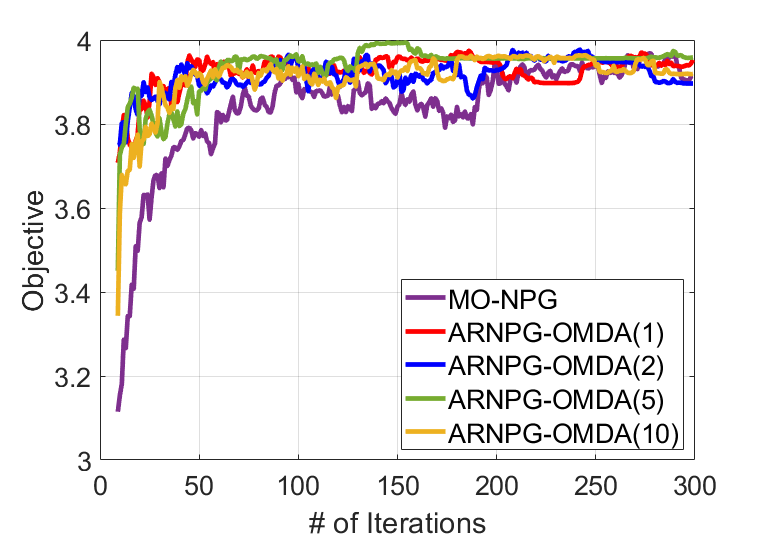}
\caption{}
\end{subfigure}
\caption{The last-iterate objective value versus the total number of iterations, for sample-based ARNPG-IMD, ARNPG-OMDA, MO-NPG on a randomly generated two-objective MOMDP with (a) sum-logarithmic (\ref{eqn:sum_log}) and (b) max-min trade-off (\ref{eqn:minimax}).}
\label{fig:simple_momdp_A}
\end{figure}

\subsubsection{Max-min trade-off}
% Recall that the ARNPG-OMDA algorithm differs from the MO-NPG algorithm even if $t_k = 1$. 
% Since MO-PMP and MO-NPG are fundamentally different \textcolor{purple}{[This is confusing. Are we talking about for this setting only? Because it is said in the last paragraph they are the same when $t_k=1$. ]}, we fix $t_k = 1, \forall k = 0, 1, \dots, K-1$. 
We conduct the experiments with learning rates $\{0.1, 0.2, \dots, 1.0\}$
% $\{0.5, 0.75, \dots, 6.75, 7\}$ 
for sample-based MO-NPG and choose the best hyperparameter 0.5. 
We also conduct a grid-search over $\alpha \in \{0.05, 0.1, 0.2\}$, $\eta \in \{0.1, 0.2, 0.5\}$, and $\eta' \in \{0.1, 0.2, 0.5\}$ for sample-based ARNPG-OMDA with $t_k=1$ and select the best hyperparameter $\alpha = 0.1$, $\eta = 0.5$, and $\eta' = 0.2$. 
%since it enjoys the highest objective value after 300 iterations. 
% Similarly, we do a grid-search over $\alpha \in \{0.05, 0.1, 0.2\}$, $\eta \in \{0.1, 0.2, 0.5\}$, and $\eta' \in \{0.1, 0.2, 0.5\}$ for sample-based ARNPG-OMDA ($t_k=1$) and select the best hyperparameter $\alpha = 0.1$, $\eta = 0.5$, and $\eta' = 0.2$. 
Fixing such $\alpha, \eta, \eta'$, we explore the impact of $t_k > 1$. 

Figure \ref{fig:simple_momdp_A}(b) shows that compared to the MO-NPG algorithm, ARNPG-OMDA algorithms converge faster (i.e., achieve larger scalarized objectives during the first 200 iterations), and that all algorithms approximately find the optimal value after 250 iterations. 
%It is because ARNGP-OMDA leverage the underlying OMDA algorithm and 
%larger and more stable reward values of  since the proposed methods can avoid
%alternating attention between two reward value functions. 
% \prk{You need to point out some more conclusions that the reader can draw from the performance graphs. }
% oscillating among corner points. 
% \textcolor{purple}{[I'm not sure what "corner points" this is referring to. Need clarification.]}

\subsection{Acrobot-v1}
We examine two reward value functions for MOMDP scenarios where the agent is rewarded 1 when it swings the end of the lower link to the given height ranges, i.e., greater than 0.5 and 0-0.1, respectively. The goal is to maximize the objective according to the given criteria. (In this section, we focus on sum-logarithmic scalarization.) To show the impact of $t_k$ in the ARNPG framework, we fix $\alpha = 1$ and choose InnerLoop iterations $t_k = 1, 2, 5, 10$ respectively.

\begin{figure}[ht]
\centering
\begin{subfigure}{0.3\textwidth}
\centering
\includegraphics[width=\textwidth]{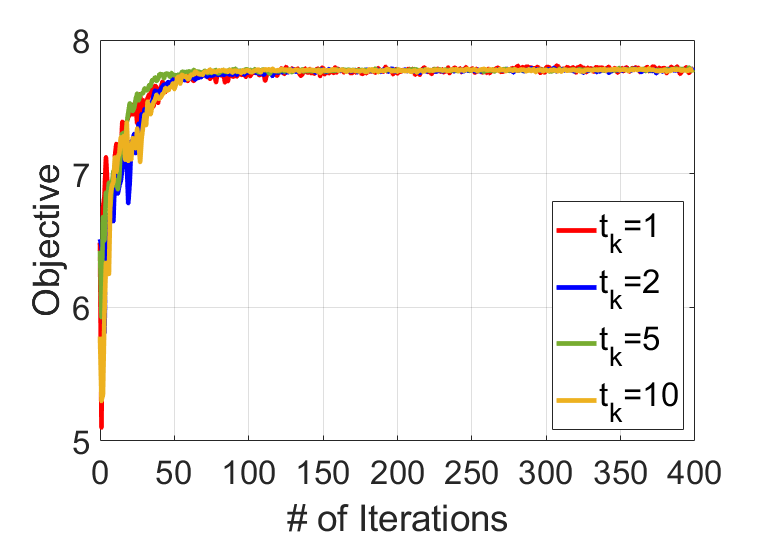}
% \caption{}
\end{subfigure}
\begin{subfigure}{0.3\textwidth}
\centering
\includegraphics[width=\textwidth]{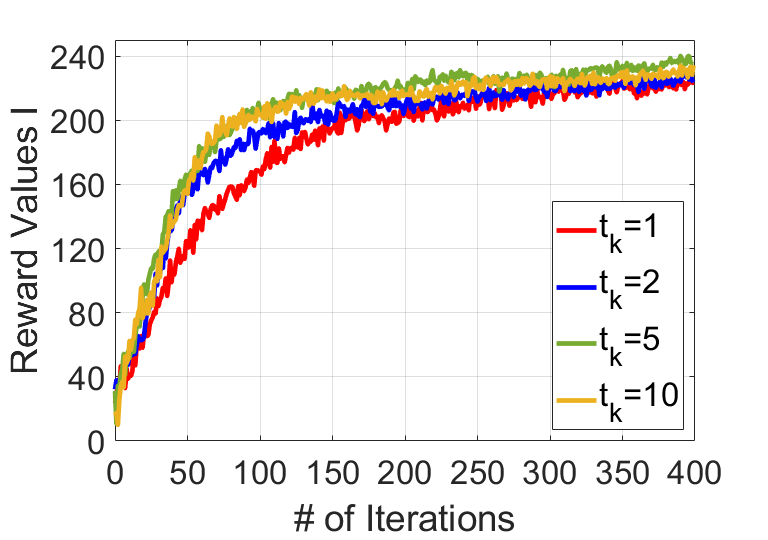}
% \caption{}
\end{subfigure}
\begin{subfigure}{0.3\textwidth}
\centering
\includegraphics[width=\textwidth]{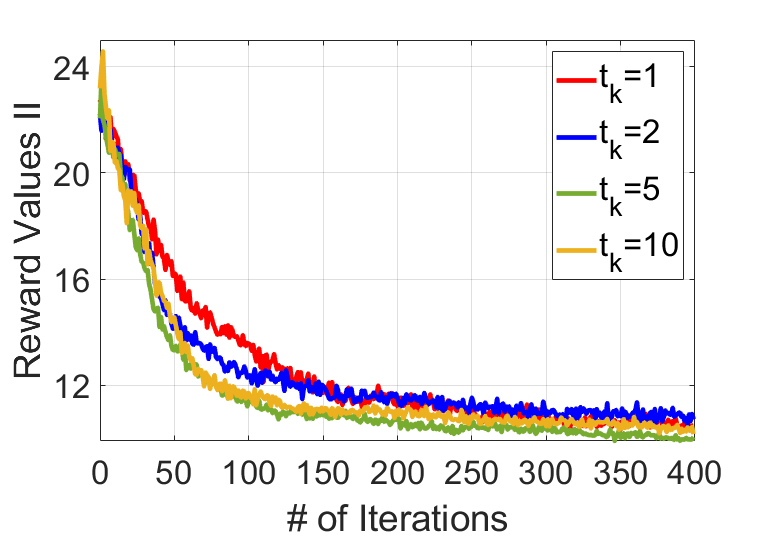}
% \caption{}
\end{subfigure}
\caption{Last-iterate performance for sample-based ARNPG-IMD with different inner loops ($t_k = 1, 2, 5, 10$) averaged over 10 random seeds.}
\label{fig:acrobot_smooth}
\end{figure}

It can be seen from Figure \ref{fig:acrobot_smooth} that the scalarized objectives are indistinguishable because the gradient of logarithmic functions is small when the input values become large. With respect to the two reward objectives (reward values I and reward values II), we can observe some trade-offs in that $t_k = 5$ converges faster than other $t_k$s.

\section{Supporting lemmas} \label{appendix:support}
Before delving into detailed proofs for the proposition and theorems, we introduce some supporting lemmas.

Recall that the Bregman divergence generated by a convex differentiable function $h(\cdot)$ is
\begin{align*}
    B_{h}(x, y) := h(x) - h(y) - \langle \nabla h(y), x-y\rangle.
\end{align*}
The fundamental inequality \eqref{eqn:fundamental-MD} associated with mirror ascent is formally presented in the following lemma.
\begin{lemma}
\label{lem:pushback}
Let $B_h: \Xc \times \Xc \to \mathbb{R}$ be a Bregman divergence function, $\Xc \subset \Rb^n$ be a compact convex set, and $g \in \Rb^n$. Suppose $x'=\argmax_{y \in \Xc}\{ \langle g, y \rangle - \alpha B_h(y|| x)\}$ for a fixed $x \in \Xc$ and $\alpha>0$. Then for any $y \in \Xc$,
\begin{align*}
    \langle g, x' \rangle-\alpha B_h(x'|| x) \geq \langle g, y \rangle - \alpha B_h(y || x) + \alpha B_h(y || x').
\end{align*}
\end{lemma}
Inequalities of the same form have appeared in many previous works, e.g., Lemma 3.4 in \cite{lan2020first} and a case of $\Xc$ being a probability simplex (Lemma 2.1 in \cite{wei2020online}). For completeness, we provide a proof of Lemma \ref{lem:pushback}.
%[Pushback property of Bregman divergences, Lemma 2.1 in {\cite{wei2020online}}]
\begin{proof}[Proof of Lemma \ref{lem:pushback}]
Since $h$ is proper and convex, $x':=\argmax_{y \in \Xc}\{ \langle g, y \rangle - \alpha B_h(y|| x)\}$ exists and satisfies the first order condition
\begin{align*}
    \langle g - \alpha \nabla h(x') + \alpha \nabla h(x), x' - y \rangle = \langle g - \alpha \nabla_{x'} B_h(x' || x), x' - y\rangle \geq 0, \quad \forall y \in \Xc,
\end{align*}
which implies $\langle g, x' - y \rangle \geq \alpha \langle \nabla h(x) - \nabla h(x'), x' - y \rangle$. It can be verified that
\begin{align*}
    \langle \nabla h(x) - \nabla h(x'), x' - y \rangle = B_h(x' || x) - B_h(y || x) + B_h(y || x').
\end{align*}
We can conclude the proof by substituting the equation into the previous inequality.
\end{proof}

The following lemma draws a connection between the $\ell_1$ difference of state-action visitation distributions and averaged KL-divergence.
\begin{lemma} 
\label{lem:distance}
    Let $d_\rho^{\pi'}, d^{\pi}_\rho$ be two discounted state-action visitation distributions corresponding to policies $\pi'$ and $\pi$. Then
    \begin{align*}
    \|d^{\pi'}_\rho - d^{\pi}_\rho\|_1 \leq \frac{\gamma \sqrt{2}}{1 - \gamma} \sqrt{\min \left(D_{d_\rho^{\pi'}}(\pi' || \pi), D_{d_\rho^{\pi'}}(\pi || \pi'), D_{d_\rho^{\pi}}(\pi' || \pi), D_{d_\rho^{\pi}}(\pi || \pi')\right) }.
\end{align*}
\end{lemma}
\begin{proof}
% \dk{Compare this with the proof of Lemma 3 in \cite{achiam2017constrained} } {\color{blue} Results are the same, but proof via a different way.}
Let $d^{\pi}_{\rho, h}(\cdot, \cdot)$ be the state-action visitation distribution at step $h$, which implies $\frac{1}{1 - \gamma} d^{\pi}_{\rho}(\cdot,\cdot) = \sum_{h \geq 0} \gamma^h d^{\pi}_{\rho, h}(\cdot,\cdot)$. Denote $\tilde{\pi}_h$ as the policy that implements policy $\pi$ for the first $h$ steps and then commits to policy $\pi'$ thereafter. Denote its corresponding discounted state-action visitation distribution by $d^{\tilde{\pi}_h}_\rho$. It follows that
\begin{align*}
    \frac{1}{1 - \gamma}\|d^{\pi'}_\rho - d^{\pi}_\rho\|_1 & \stackrel{(a)}{=} \frac{1}{1 - \gamma} \left\| \sum_{h=0}^\infty (d^{\tilde{\pi}_h}_\rho - d^{\tilde{\pi}_{h+1}}_\rho) \right\|_1 \stackrel{(b)}{\leq} \frac{1}{1 - \gamma} \sum_{h=0}^\infty \|  d^{\tilde{\pi}_h}_\rho - d^{\tilde{\pi}_{h+1}}_\rho \|_1 \\
    & = \sum_{h=0}^\infty \left\| \sum_{t = 0}^{\infty} \gamma^t (d^{\tilde{\pi}_h}_{\rho,t} - d^{\tilde{\pi}_{h+1}}_{\rho,t}) \right\|_1 = \sum_{h=0}^\infty \left\| \sum_{t = h+1}^{\infty} \gamma^t (d^{\tilde{\pi}_h}_{\rho,t} - d^{\tilde{\pi}_{h+1}}_{\rho,t}) \right\|_1 \\
    & \stackrel{(c)}{\leq} \sum_{h=0}^\infty \sum_{t \geq h+1}^{\infty} \gamma^t \| d^{\tilde{\pi}_h}_{\rho,t} - d^{\tilde{\pi}_{h+1}}_{\rho, t} \|_1 \stackrel{(d)}{\leq} \sum_{h=0}^\infty \sum_{t \geq h+1}^{\infty} \gamma^t \| d^{\tilde{\pi}_h}_{\rho,h} - d^{\tilde{\pi}_{h+1}}_{\rho, h} \|_1 \\
    & = \frac{\gamma}{1 - \gamma} \sum_{h=0}^\infty \gamma^h \Eb_{s \sim d^{\pi}_{\rho, h}}\| \pi(\cdot|s) - \pi'(\cdot|s) \|_1 \\
    & \stackrel{(e)}{\leq} \frac{\gamma}{1 - \gamma} \sqrt{ \left(  \sum_{h \geq 0} \gamma^h \right) \left(\sum_{h = 0}^\infty \gamma^h \Eb_{s \sim d^{\pi}_{\rho, h}}\| \pi(\cdot|s) - \pi'(\cdot|s) \|_1^2 \right)} \\
    & = \frac{\gamma}{(1 - \gamma)^{2}} \sqrt{\Eb_{s \sim d^{\pi}_\rho}\| \pi(\cdot|s) - \pi'(\cdot|s) \|_1^2} \ .
\end{align*}
Above, $(a)$ holds by telescoping, $(b)$ and $(c)$ hold due to the triangle inequality of $\ell_1$-norm and the definition of $\tilde{\pi}_h$, $(d)$ hold owing to the data processing inequality for $f$-divergence $\|\cdot\|_1$, 
% \ct{I'm not familiar with this form of data processing inequality. Can we add a reference to it?} 
and $(e)$ holds due to the Cauchy-Schwarz inequality. Due to the symmetry between $\pi$ and $\pi'$, it can be similarly derived
\begin{align*}
    \|d^{\pi'}_\rho - d^{\pi}_\rho\|_1 \leq \frac{\gamma}{1 - \gamma} \sqrt{\Eb_{s \sim d^{\pi'}_\rho}\| \pi(\cdot|s) - \pi'(\cdot|s) \|_1^2} \ .
\end{align*}
We can conclude the proof by further applying Pinsker's inequality.% (Lemma \ref{lem:pinsker}). \dk{You can just state the Pinsker's inequality here}
\end{proof}

An application of Lemma \ref{lem:distance} gives an upper bound on the difference between value function vectors as follows.
\begin{lemma}
\label{lem:kl_property}
For any $k = 0, 1, \dots, K-1$,
\begin{align*}
    \frac{1}{2} \|V^{\pi_k}_{1:m}(\rho) - V^{\pi_{k+1}}_{1:m}(\rho) \|_{\infty}^2 \leq \frac{\gamma^2}{(1 - \gamma)^4} D_{d^{\pi_{k+1}}_{\rho}}(\pi_{k+1} || \pi_k).
\end{align*}
\begin{proof}
For any $i = 1,2,\ldots, m$, we have
\begin{align*}
   \left|V^{\pi_k}_{i}(\rho) - V^{\pi_{k+1}}_{i}(\rho)\right|  & = \frac{1}{1-\gamma} \left|\sum_{(s, a) \in \Sc \times \Ac} r_i(s, a) (d_{\rho}^{\pi_k}(s, a) - d_{\rho}^{\pi_{k+1}}(s, a)) \right| \\
    &\leq \frac{1}{1 - \gamma} \|d^{\pi_k}_\rho - d^{\pi_{k+1}}_\rho\|_1 \leq \frac{\gamma \sqrt{2}}{(1 - \gamma)^2} \sqrt{D_{d^{\pi_{k+1}}_{\rho}}(\pi_{k+1} || \pi_k)},
\end{align*}
where the last inequality is due to Lemma \ref{lem:distance}.
\end{proof}
\end{lemma}

% InnerLoop of the ARNPG framework is solving a KL-regularized MDP with value as in \eqref{eqn:inner-loop-variational}, 
% \begin{align*}
%     \tilde{V}_{k, \alpha}^{\pi_\theta}(\rho) := V^{\pi_\theta}_{\tilde{r}_k}(\rho) - \alpha \frac{D_{d^{\pi_\theta}_\rho}(\pi_\theta || \pi_{\theta_k})}{1 - \gamma}.
% \end{align*}

\section{Proof in Section \ref{sec:ARNPG}}
\label{sec:app_proof_arnpg}
% For any policies $\pi, \pi'$, define a pseudo KL-divergence between $d_\rho^\pi, d^{\pi'}_\rho \in \Dc_\rho$ by
% \begin{align}
%     \tilde{D}(d_{\rho}^{\pi} || d_{\rho}^{\pi'}) : = \sum_{(s, a) \in \Sc \times \Ac} d_{\rho}^{\pi}(s, a) \log \frac{d_{\rho}^{\pi}(s, a) / d_{\rho}^{\pi}(s)}{d_{\rho}^{\pi'}(s, a) / d_{\rho}^{\pi'}(s)}. \label{eqn:pseudo-KL}
% \end{align}

% InnerLoop of the ARNPG framework is solving a KL-regularized MDP with value as in \eqref{eqn:inner-loop-variational}, 
% \begin{align*}
%     \tilde{V}_{k, \alpha}^{\pi_\theta}(\rho) := V^{\pi_\theta}_{\tilde{r}_k}(\rho) - \alpha \frac{D_{d^{\pi_\theta}_\rho}(\pi_\theta || \pi_{\theta_k})}{1 - \gamma}.
% \end{align*}
This section presents the formal proof of Proposition \ref{pro:fundamental-inequality}. We begin by presenting some properties of InnerLoop. We shall omit $\theta$ in $\pi_{\theta}$, since the policies are under softmax parameterization. 

\subsection{Linear convergence of InnerLoop}
InnerLoop($\tilde{r}_k, \pi_k, \alpha, \eta, t_k$) approximately solves the following KL-regularized MDP via natural policy gradient. Note that 
\begin{align}
    \tilde{V}_{k, \alpha}^{\pi}(s) =  \Eb\left[\sum_{t \geq 0} \gamma^t \left( \tilde{r}_k(s_t, a_t)  - \alpha \log \pi_k(a_t|s_t) + \alpha \log \pi(a_t|s_t) \right) | s_0 = s, \pi \right],
\end{align}
which can be viewed as an entropy regularized value with reward function $\tilde{r}_k(s,a) - \alpha \log \pi_k(a|s) $. The entropy-regularized state-action value function is then defined as \cite{cen2021fast}
\begin{align}
\label{eqn:regu-Q}
    \tilde{Q}_{k, \alpha}^{\pi}(s, a) = \tilde{r}_k(s, a) - \alpha \log\pi_k(a|s) + \gamma \Eb_{s' \sim P(\cdot | s, a)} [\tilde{V}_{k, \alpha}^{\pi}(s')]. 
\end{align}
% The NPG update \eqref{eqn:inner-NPG} in InnerLoop yields a closed-form expression (cf. Equation (16) in \cite{cen2021fast}):
% \begin{align}
% &\pi_k^{(t+1)}(a|s) = \frac{(\pi_k^{(t)}(a|s))^{1 - \frac{\eta \alpha}{1-\gamma}} \exp(\frac{-\eta \tilde{Q}_{k, \alpha}^{\pi_k^{(t)}}(s, a)}{1-\gamma})}{Z^{(t)}(s)},
% \label{eqn:npg-entropy}
% \end{align}
% where $Z^{(t)}(s) := \sum_{a'} (\pi_k^{(t)}(a'|s))^{1 - \frac{\eta \alpha}{1-\gamma}} \exp(\frac{-\eta \tilde{Q}_{k, \alpha}^{\pi_k^{(t)}}(s, a')}{1-\gamma})$. 
The convergence of NPG in entropy-regularized MDP has been well-studied in \cite{cen2021fast}, with the key results summarized in the following lemma. 
% The following results from \cite{cen2021fast} can be applied. 
\begin{lemma}[Linear convergence of entropy-regularized NPG, Theorem 1 in \cite{cen2021fast}]
\label{lem:cen} 
For any learning rate $0 < \eta \le (1 - \gamma)/\alpha$ and any $k = 0, 1, \dots, K-1$, the entropy-regularized NPG updates satisfy
\begin{align*}
    \left\|\tilde{Q}_{k, \alpha}^{\pi_k^*} - \tilde{Q}_{k, \alpha}^{\pi_k^{(t+1)}}\right\|_{\infty} & \leq C_{k} \gamma(1-\eta \alpha)^{t}, \\
    \left\|\log \pi_{k}^{*}-\log \pi_k^{(t+1)}\right\|_{\infty} & \leq 2 C_{k} \alpha^{-1}(1-\eta \alpha)^{t},\\
    \left\|\tilde{V}_{k, \alpha}^{\pi_k^*} - \tilde{V}_{k, \alpha}^{\pi_k^{(t+1)}}\right\|_{\infty} & \leq 3 C_{k} (1-\eta \alpha)^{t},
\end{align*}
for all $t \geq 0$, where $C_k$ satisfies $C_{k} \geq \left\|\tilde{Q}_{k, \alpha}^{\pi_k^*} - \tilde{Q}_{k, \alpha}^{\pi_k^{(0)}}\right\|_{\infty}+2 \alpha\left(1-\frac{\eta \alpha}{1-\gamma}\right)\left\|\log \pi_{k}^{*}-\log \pi_k^{(0)}\right\|_{\infty}$.
\end{lemma}
\textit{Remark.} There is a typographical mistake in the inequality `` $\|\tilde{V}_{k, \alpha}^{\pi_k^*} - \tilde{V}_{k, \alpha}^{\pi_k^{(t+1)}}\|_{\infty} \leq 3 \textcolor{red}{\gamma} C_{k} (1-\eta \alpha)^{t}$ '' in \cite{cen2021fast}, and it has been corrected here. It is not hard to verify that the proofs of the inequalities in Lemma \ref{lem:cen} \cite{cen2021fast} hold without the assumption that $0 \leq r(s,a) \leq 1$.

Denote $\tilde{V}^{\pi}_{k}(s) := V^{\pi}_{\tilde{r}_k}(s)$. For the regularized MDP, its optimal policy is uniformly optimal, i.e., for any state $s \in \Sc$, 
\begin{align}
\label{eqn:uniform-optimality}
    \frac{1}{1 - \gamma}\|\tilde{r}_k\|_\infty \geq \tilde{V}_{k}^{\pi^*_k}(s) \geq \tilde{V}_{k}^{\pi^*_k}(s) - \frac{\alpha}{1 - \gamma}D_{d^{\pi^*}_s}(\pi^* || \pi_k) =\tilde{V}_{k, \alpha}^{\pi^*_k}(s) \geq \tilde{V}^{\pi_k}_{k, \alpha}(s) = \tilde{V}^{\pi_k}_{k}(s).
\end{align} 
It follows that $\forall (s, a) \in \Sc \times \Ac$,
\begin{align*}
\left|\tilde{Q}_{k, \alpha}^{\pi_k^*}(s, a) - \tilde{Q}_{k, \alpha}^{\pi_k}(s, a)\right| &= \gamma \sum_{s' \in \Sc} P(s'|s,a) \left|\tilde{V}_{k, \alpha}^{\pi_k^*}(s') - \tilde{V}_{k, \alpha}^{\pi_k}(s')\right| \stackrel{(a)}{\le} \frac{\gamma \|\tilde{r}_k\|_\infty}{1 - \gamma}, 
\end{align*}
where $(a)$ holds due to the relation in (\ref{eqn:uniform-optimality}). It implies $\|\tilde{Q}_{k, \alpha}^{\pi_k^*} - \tilde{Q}_{k, \alpha}^{\pi_k}\|_{\infty} \leq \frac{\gamma \|\tilde{r}_k\|_\infty}{1 - \gamma}$.
Since $1-\frac{\eta \alpha}{1-\gamma} = 0$ when $\eta = \frac{1 - \gamma}{\alpha}$, we can 
% therefore conclude the proof by 
apply results in Lemma \ref{lem:cen} with $C_k = \frac{\gamma \|\tilde{r}_k\|_{\infty}}{1 - \gamma}$, which gives% \dk{Typo here? $s$?}
% $C_k = 2 \gamma \left(\frac{1 + \sum_{i=1}^m \lambda_{k, i}}{1-\gamma} + \frac{m \eta'}{(1-\gamma)^2}\right)$. 
\begin{align}
\label{eqn:apply_cen}
\tilde{V}^{\pi_{k+1}}_k(\rho) + \frac{\alpha}{1 - \gamma} D_{d^{\pi_{k+1}}_\rho}(\pi_{k+1} || \pi_k) \leq - \tilde{V}^{\pi_{k}^*}_k(\rho) + \frac{\alpha}{1 - \gamma} D_{d^{\pi_{k}^*}_\rho}(\pi_{k}^* || \pi_k) + 3 C_k (1 - \eta \alpha)^{t_k}.
\end{align}

\subsection{Hidden convexity in state-action visitation distribution}
Noting that the class of softmax policies is almost complete in the sense that its closure contains all stationary policies, we will omit the parameter $\theta$ in $\pi_\theta$. The set of achievable state-action visitations is $\Dc = \{d \in \Delta(\Sc\times\Ac): \gamma \sum_{s', a'} P(s | s', a') d(s', a') +  (1 - \gamma) \rho(s) =  \sum_{a} d(s, a), ~ \forall s \in \Sc\}$, which is a convex compact set.

For any policies $\pi, \pi'$, define a pseudo KL-divergence between $d_\rho^\pi, d^{\pi'}_\rho \in \Dc_\rho$ by
\begin{align}
    \tilde{D}(d_{\rho}^{\pi} || d_{\rho}^{\pi'}) : = \sum_{(s, a) \in \Sc \times \Ac} d_{\rho}^{\pi}(s, a) \log \frac{d_{\rho}^{\pi}(s, a) / d_{\rho}^{\pi}(s)}{d_{\rho}^{\pi'}(s, a) / d_{\rho}^{\pi'}(s)}. \label{eqn:pseudo-KL}
\end{align}
It is not hard to verify that
\begin{align}
\label{eqn:bregman_equ}
    D_{d_{\rho}^{\pi}}(\pi || \pi') & = \sum_{s \in \Sc} d_{\rho}^{\pi}(s) \sum_{a \in \Ac} \pi(a|s) \log \frac{\pi(a|s)}{\pi'(a|s)} = \sum_{s \in \Sc} d_{\rho}^{\pi}(s) \sum_{a \in \Ac} \frac{d_{\rho}^{\pi}(s, a)}{d_{\rho}^{\pi}(s)} \log \frac{d_{\rho}^{\pi}(s, a)/d_{\rho}^{\pi}(s)}{d_{\rho}^{\pi'}(s, a)/d_{\rho}^{\pi'}(s)} \notag \\
    &= \sum_{(s, a) \in \Sc \times \Ac} d_{\rho}^{\pi}(s, a) \log \frac{d_{\rho}^{\pi}(s, a)/d_{\rho}^{\pi}(s)}{d_{\rho}^{\pi'}(s, a)/d_{\rho}^{\pi'}(s)} = \tilde{D}(d_{\rho}^{\pi} || d_{\rho}^{\pi'}).
\end{align}
%We notice that the right-hand side of (\ref{eqn:bregman_equ}) is a Bregman divergence via Lemma \ref{lem:Bregman}, and t
This equation bridges the state-action visitation space and the policy space. The following lemma shows that the pseudo KL-divergence defined in (\ref{eqn:pseudo-KL}) is actually a Bregman divergence between state-action visitation distributions.
\begin{lemma}\label{lem:Bregman}
    The pseudo KL-divergence $\tilde{D}(d_{\rho}^{\pi} || d_{\rho}^{\pi'})$ defined in (\ref{eqn:pseudo-KL}) is a Bregman divergence $B_h(d_{\rho}^{\pi} || d_{\rho}^{\pi'})$ generated by the convex function
\begin{align*}
    h(d_{\rho}^{\pi}) = \sum_{(s, a) \in \Sc \times \Ac} d_{\rho}^{\pi}(s, a) \log d_{\rho}^{\pi}(s, a) - \sum_{s \in \Sc} d_{\rho}^{\pi}(s) \log d_{\rho}^{\pi}(s).
\end{align*}
\end{lemma}
\begin{proof}[Proof of Lemma \ref{lem:Bregman}]
It can be verified by elementary algebera that
\begin{align*}
    \tilde{D}(d_{\rho}^{\pi} || d_{\rho}^{\pi'}) = h(d_{\rho}^{\pi}) - h(d_{\rho}^{\pi'}) - \langle \nabla h(d_{\rho}^{\pi'}), d_{\rho}^{\pi} - d_{\rho}^{\pi'} \rangle,
\end{align*}
where $\nabla_{(s,a)} h(d_{\rho}^{\pi'}) = \log d_{\rho}^{\pi'}(s, a) - \log d_{\rho}^{\pi'}(s)$. Hence we only need to show that $h(d_{\rho}^{\pi})$ is convex. The Hessian matrix of function $h(d_{\rho}^{\pi})$ can be calculated as $ \diag\left(H_1, H_2, \ldots, H_{|\Sc|}\right)$, where $H_{s} = \frac{1}{d_{\rho}^{\pi}(s)}\left( \diag(d_{\rho}^{\pi}(s)/d_{\rho}^{\pi}(s, \cdot)) - \mathbf{1} \mathbf{1}^T \right)$ is an $|\Ac| \times |\Ac|$ matrix corresponding to state $s$. For each $H_s$, we know for any $x_{1:|\Ac|} \in \Rb^{|\Ac|}$,
\begin{align*}
    x^T H_s x & = \frac{1}{d_{\rho}^{\pi}(s)}\left( \sum_{a \in \Ac} \frac{d_{\rho}^{\pi}(s)}{d_{\rho}^{\pi}(s, a)} x_a^2 - \left(\sum_{a \in \Ac} x_a\right)^2 \right) \\
    &= \frac{1}{d_{\rho}^{\pi}(s)}\left( \left(\sum_{a \in \Ac} \frac{d_{\rho}^{\pi}(s, a)}{d_{\rho}^{\pi}(s)} \right)\left(\sum_{a \in \Ac} \frac{d_{\rho}^{\pi}(s)}{d_{\rho}^{\pi}(s, a)} x_a^2\right) - \left(\sum_{a \in \Ac} x_a\right)^2 \right) \\
    & \stackrel{(a)}{\geq} \frac{1}{d_{\rho}^{\pi}(s)}\left( \left(\sum_{a \in \Ac} |x_a|\right)^2 - \left(\sum_{a \in \Ac} x_a\right)^2 \right) \geq 0,
\end{align*}
where $(a)$ is due to the Cauchy-Schwarz inequality. Thus the Hessian matrix of $h(d_{\rho}^{\pi})$ is positive semi-definite, which implies that $h(d_{\rho}^{\pi})$ is convex.
\end{proof}

% Recall the definition of pseudo KL-divergence that
% \begin{align}
%     \tilde{D}(d_{\rho}^{\pi} || d_{\rho}^{\pi'}) : = \sum_{(s, a) \in \Sc \times \Ac} d_{\rho}^{\pi}(s, a) \log \frac{d_{\rho}^{\pi}(s, a) / d_{\rho}^{\pi}(s)}{d_{\rho}^{\pi'}(s, a) / d_{\rho}^{\pi'}(s)}.
% \end{align}

InnerLoop of the ARNPG framework is solving a KL-regularized MDP with value as in \eqref{eqn:inner-loop-variational}, 
\begin{align*}
    \tilde{V}_{k, \alpha}^{\pi_\theta}(\rho) = V^{\pi_\theta}_{\tilde{r}_k}(\rho) - \alpha \frac{D_{d^{\pi_\theta}_\rho}(\pi_\theta || \pi_{\theta_k})}{1 - \gamma}.
\end{align*}
This optimization can be equivalently represented by viewing state-action visitation as the decision variables:
\begin{align}
\label{eqn:inner_equ}
    \max_{\pi} V^{\pi}_{\tilde{r}_k}(\rho) - \alpha\frac{D_{d^{\pi}_\rho}(\pi || \pi_k)}{1 - \gamma} \quad \Leftrightarrow \quad \max_{d \in \Dc} \langle \tilde{r}_k, d \rangle - \alpha \tilde{D}(d || d^{\pi_k}_\rho).
\end{align}
Here $\Leftrightarrow$ means that they are equivalent in the sense that the optimal policy solution $\pi^*_k$ for the former optimization and the optimal visitation solution $d^*_k$ for the latter satisfy $d^*_k = d^{\pi^*_k}_\rho$. Note that $\tilde{V}_{\tilde{r}_{k}}^{\pi}(\rho) = \frac{1}{1 - \gamma} \langle \tilde{r}_{k}, d^{\pi}_{\rho} \rangle$ is a linear function of $d^{\pi}_{\rho}$, $\tilde{D}(\cdot || \cdot)$ is a Bregman divergence, and $\Dc$ is compact. We can apply Lemma \ref{lem:pushback} on the latter optimization and have
%The latter optimization is one-step of mirror ascent, and by the fundamental inequality for mirror ascent \eqref{eqn:fundamental-MD} we know 
\begin{align}
    \langle \tilde{r}_k, d^*_{k} \rangle - \alpha \tilde{D}(d^*_k || d^{\pi_k}_\rho) & \geq \langle \tilde{r}_k, d \rangle - \alpha \tilde{D}(d || d^{\pi_k}_\rho) + \alpha \tilde{D}(d || d_k^*) ,\quad \forall d \in \Dc.
\end{align}
Since the policy class and the state-action visitation class are both complete, the inequality above implies that
\begin{align}
    V^{\pi_{k+1}}_{\tilde{r}_k}(\rho) - \alpha \frac{D_{d^{\pi_{k}^*}_\rho}(\pi_{k}^* || \pi_k)}{1 - \gamma} & \geq V^{\pi}_{\tilde{r}_k}(\rho) - \alpha \frac{D_{d^{\pi}_\rho}(\pi || \pi_k) - D_{d^{\pi}_\rho}(\pi || \pi_{k}^* )}{1 - \gamma} ,\quad \forall \pi. \label{eqn:pushback_fund}
\end{align}
InnerLoop does not seek to find the precise solution $\pi^*_k$ but approximates it with $\pi_{k+1} = \pi_k^{(t_k)}$ via $t_k$ micro-step iterations. Proposition \ref{pro:fundamental-inequality} provides a quantitative bound regarding the approximation error of $\pi_{k+1}$.

\subsection{Proof of Proposition \ref{pro:fundamental-inequality}}
\begin{proof}[Proof of Proposition \ref{pro:fundamental-inequality}]
Combining (\ref{eqn:apply_cen}) and (\ref{eqn:pushback_fund}) gives
\begin{align*}
    - \tilde{V}^{\pi_{k+1}}_k(\rho) + \alpha \frac{D_{d^{\pi_{k+1}}_\rho}(\pi_{k+1} || \pi_k)}{1 - \gamma} & \leq  -\tilde{V}^{\pi}_k(\rho) + \alpha \frac{D_{d^{\pi}_\rho}(\pi || \pi_k) -D_{d^{\pi}_\rho}(\pi || \pi_{k+1}) }{1 - \gamma} \notag \\
    & \quad\quad + 3 C_k (1 - \eta \alpha)^{t_k} + \alpha \frac{D_{d^{\pi}_\rho}(\pi || \pi_{k+1}) - D_{d^{\pi}_\rho}(\pi || \pi_{k}^*)}{1 - \gamma}.
\end{align*}
Note that
\begin{align}
    & D_{d^{\pi}_\rho}(\pi || \pi_{k+1}) - D_{d^{\pi}_\rho}(\pi || \pi_{k}^*) = \left\langle d^\pi_\rho(\cdot, \cdot), \log\frac{\pi_k^*(\cdot, \cdot)}{\pi_{k+1}(\cdot, \cdot)} \right \rangle \notag \\
    \leq& \|d^{\pi}_\rho\|_1 \|\log \pi^*_k - \log \pi_{k+1} \|_{\infty} = \|\log \pi^*_k - \log \pi_{k+1} \|_{\infty} \leq 2 C_k \alpha^{-1} (1 - \eta \alpha)^{t_k}, \notag
\end{align}
where the first inequality follows from Cauchy-Schwartz, and the last inequality is due to Lemma \ref{lem:cen}. We thus have
\begin{align*}
    - \tilde{V}^{\pi_{k+1}}_k(\rho) + \alpha \frac{D_{d^{\pi_{k+1}}_\rho}(\pi_{k+1} || \pi_k)}{1 - \gamma}  & \leq - \tilde{V}^{\pi}_k(\rho) + \alpha\frac{ D_{d^{\pi}_\rho}(\pi || \pi_k) - D_{d^{\pi}_\rho}(\pi || \pi_{k+1})}{1 - \gamma} + \frac{5C_k (1 - \eta\alpha)^{t_k}}{1 - \gamma}.
\end{align*}
We then conclude the proposition, since $\frac{5C_k (1 - \eta\alpha)^t}{1 - \gamma} \leq \epsilon_k$ can be guaranteed by $t_k \geq \frac{1}{1 - \gamma} \log(\frac{5 \gamma \|\tilde{r}_k\|_\infty}{(1-\gamma)^2 \epsilon_k})$. \end{proof}

% $\|\tilde{V}_{k, \alpha}^{\pi_k^*} - \tilde{V}_{k, \alpha}^{\pi_{k+1}}\|_{\infty} \le 3 C_k (1 - \eta\alpha)^{t_k}$, $\alpha \|\log \pi_{k}^{*}-\log \pi_k^{(t+1)}\|_{\infty} \leq 2 C_{k} (1-\eta \alpha)^{t}$, 
% \begin{align*}
%     C_{k} \geq \left\|\tilde{Q}_{k, \alpha}^{\pi_k^*} - \tilde{Q}_{k, \alpha}^{\pi_k^{(0)}}\right\|_{\infty}+2 \alpha\left(1-\frac{\eta \alpha}{1-\gamma}\right)\left\|\log \pi_{k}^{*}-\log \pi_k^{(0)}\right\|_{\infty}.
% \end{align*}

\section{Proof in Section \ref{sec:theory-app}} \label{sec:proof-theory-app}

\subsection{ARNPG-IMD for smooth scalarization}

\begin{proof}[Proof of Theorem \ref{thm:ARNPG-IMD}]
By $|\tilde{r}_k(s,a)| = |\langle \tilde{G}_k, r_{1:m}(s,a) \rangle| \leq \|\tilde{G}_k\|_{1} \|r_{1:m}(s,a)\|_\infty \leq L$, we know $\|\tilde{r}_k\|_\infty \leq L$. Recall $\alpha \geq \frac{\beta}{(1 - \gamma)^3}$. Taking $\epsilon_k = \frac{\alpha \log(|\Ac|)}{(1 - \gamma)K}$, we choose $t_k = \lceil\frac{1}{1 - \gamma} \log(\frac{5LK}{\beta \log(|\Ac|)}) + 1\rceil.$
%$t_k - 1 = \frac{1}{1 - \gamma} \log(\frac{5 L K}{2\gamma^2\beta \log(|\Ac|)}) \geq \frac{1}{1 - \gamma} \log(\frac{5 \|\tilde{r}_k\|_\infty}{(1-\gamma)^2 \epsilon_k})$.
%$\frac{1}{1 - \gamma} \log(\frac{5 LK}{(1-\gamma) \alpha \log(|\Ac|)}) + 1 $
%$t \geq \frac{1}{1 - \gamma} \log(\frac{5 L}{(1-\gamma)^2 \epsilon_k}) + 1$.
Thus by Proposition \ref{pro:fundamental-inequality}, for any policy $\pi$, we have the fundamental inequality 
\begin{align}
\label{eqn:fund_smooth}
    V^{\pi_{k+1}}_{\tilde{r}_k}(\rho) - \alpha \frac{D_{d^{\pi_{k+1}}_\rho}(\pi_{k+1} || \pi_k)}{1 - \gamma} & \geq V^{\pi}_{\tilde{r}_k}(\rho) - \alpha \frac{D_{d^{\pi}_\rho}(\pi || \pi_k) - D_{d^{\pi}_\rho}(\pi || \pi_{k+1})}{1 - \gamma} - \epsilon_k.
\end{align}
For the RHS of (\ref{eqn:fund_smooth}), by the concavity of $F$, we have
\begin{align*}
V^{\pi}_{\tilde{r}_k}(\rho) - V^{\pi_{k}}_{\tilde{r}_k}(\rho) & = \langle \tilde{G}_k, V^{\pi}_{1:m}(\rho) - V^{\pi_k}_{1:m}(\rho) \rangle \geq F(V_{1:m}^{\pi}(\rho)) - F(V_{1:m}^{\pi_{k}}(\rho)).
\end{align*}

For the LHS of (\ref{eqn:fund_smooth}), by the fact that $F$ is $\beta$-smooth, we know
\begin{align*}
V^{\pi_{k+1}}_{\tilde{r}_k}(\rho) - V^{\pi_{k}}_{\tilde{r}_k}(\rho) & = \langle \tilde{G}_k, V^{\pi_{k+1}}_{1:m}(\rho) - V^{\pi_k}_{1:m}(\rho) \rangle \\
&\leq F(V_{1:m}^{\pi_{k+1}}(\rho)) - F(V_{1:m}^{\pi_{k}}(\rho)) + \frac{\beta}{2} \left\|V_{1:m}^{\pi_k}(\rho) - V_{1:m}^{\pi_{k+1}}(\rho)\right\|^2_\infty.
\end{align*}
% Moreover, we prove in Appendix \ref{appendix:support} by the properties of KL-regularization that
From Lemma \ref{lem:kl_property} and recalling $\alpha \ge \frac{\beta}{(1 - \gamma)^3}$,
\begin{align*}
    \frac{\beta}{2} \|V^{\pi_k}_{1:m}(\rho) - V^{\pi_{k+1}}_{1:m}(\rho) \|^2_\infty %\leq \frac{\gamma^2 \beta}{(1 - \gamma)^4} D_{d^{\pi_{k+1}}_{\rho}}(\pi_{k+1} || \pi_k) 
    \leq \frac{\gamma^2 \beta}{(1 - \gamma)^4} D_{d^{\pi_{k+1}}_\rho}(\pi_{k+1} || \pi_k) \leq \alpha \frac{D_{d^{\pi_{k+1}}_\rho}(\pi_{k+1} || \pi_k)}{1 - \gamma}.
\end{align*}
Substituting these three inequalities into the fundamental inequality \eqref{eqn:fund_smooth}, telescoping from $k = 0$ to $K-1$, and selecting $\pi = \pi^*$, we can conclude that
\begin{align*}
    &\frac{1}{K} \sum_{k=1}^{K} F(V_{1:m}^{\pi_k}(\rho)) \geq F(V_{1:m}^{\pi^*}(\rho)) - \frac{\alpha D_{d^{\pi^*}_\rho}(\pi^*||\pi_0)}{(1-\gamma)K} - \frac{1}{K} \sum_{k=0}^{K-1} \epsilon_k \ge F(V_{1:m}^{\pi^*}(\rho)) - \frac{2 \alpha \log(|\Ac|)}{(1-\gamma)K}.
\end{align*}
% $\sum_{k=0}^{K-1} \epsilon_k = \frac{\alpha \log(|\Ac|)}{1 - \gamma}$, $\epsilon_k = \frac{\alpha \log(|\Ac|)}{(1 - \gamma)K}$, $t \geq \frac{1}{1 - \gamma} \log(\frac{5 L K}{2\gamma^2\beta \log(|\Ac|)}) + 1 \geq \frac{1}{1 - \gamma} \log(\frac{5 LK}{(1-\gamma) \alpha \log(|\Ac|)}) + 1 $
\end{proof}

\begin{proof}[Proof of Corollary \ref{cor:smooth-scalar}]
Note that $T = \sum_{k = 0}^{K-1} t_k = \Theta(\frac{K}{1-\gamma} \log(K))$. It implies $\frac{K}{1 - \gamma} = \Theta(T / \log(T))$. Substituting this into Theorem \ref{thm:ARNPG-IMD} concludes Corollary \ref{cor:smooth-scalar}.
\end{proof}

\subsection{ARNPG-EPD for CMDP}
We first introduce the properties of the Lagrange multiplier updates \eqref {eqn:dual-update-equation} in the following lemma. 
%\dk{Refer to the  Lagrange multiplier definition in main paper, (10)}. 
\begin{lemma}[Properties of Lagrange multiplier updates]
\label{lem:L_property}
Based on the update of the Lagrange multipliers $\lambda_{k}$, for any $i \in [2:m]$ we have:
\begin{enumerate}
    \item At any macro step k, $\lambda_{k, i} \geq 0$.
    \item At any macro step k, $\lambda_{k, i} + \eta' (b_i - V_{i}^{\pi_k}(\rho)) \geq 0$.
    \item At macro step 0, $|\lambda_{0, i}| \leq \eta' |V_{i}^{\pi_0}(\rho) - b_i|$; at any macro step $k > 0$, $|\lambda_{k, i}| \geq \eta' | V_{i}^{\pi_k}(\rho) - b_i|$.
\end{enumerate}
\end{lemma}

\textit{Remark.} The first property guarantees the feasibility of the Lagrange multipliers; the second property ensures that the Lagrangian in the inner loop can indeed maximize the constraint rewards; and the third property is a key supporting step for the analysis of the constraint violation.

\begin{proof}[Proof of Lemma \ref{lem:L_property}] Taking any $i \in [2: m]$, we prove each property respectively.
\begin{enumerate}
    \item Note that $\lambda_{0, i} = \max\{0, \eta'(V_i^{\pi_{0}}(\rho) - b_i)\} \geq 0$ by initialization. Suppose $\lambda_{k, i} \geq 0$. The update is $\lambda_{k+1, i} = \max\left\{\eta' (V^{\pi_{k+1}}_{i}(\rho) - b_{i}), \lambda_{k,i} + \eta'(b_{i} - V^{\pi_{k+1}}_{i}(\rho)) \right\}$. \\
    If $b_i - V_{i}^{\pi_{k+1}}(\rho) < 0$, then $\lambda_{k+1, i} \geq 0$ by the first component in the $\max\{\cdot, \cdot\}$. \\
    If $b_i - V_{i}^{\pi_{k+1}}(\rho) \geq 0$, then $\lambda_{k+1, i} \ge 0$ by the second component in the $\max\{\cdot, \cdot\}$. \\
    Thus, $\lambda_{k+1, i} \geq 0$, and property can be proved by induction.
    
    \item For $k= 0$, $\lambda_{0, i} + \eta' (b_i - V_{i}^{\pi_0}(\rho)) = \max\{\eta' (b_i - V_{i}^{\pi_0}(\rho)), 0\} \geq 0$. \\
    The update is $\lambda_{k+1, i} = \max\left\{\eta' (V^{\pi_{k+1}}_{i}(\rho) - b_{i}), \lambda_{k,i} + \eta'(b_{i} - V^{\pi_{k+1}}_{i}(\rho)) \right\}$. Thus for $k \geq 0$, $\lambda_{k+1, i} + \eta' (b_i - V_{i}^{\pi_{k+1}}(\rho)) = \max\left\{0, \lambda_{k,i} + 2\eta'(b_{i} - V^{\pi_{k+1}}_{i}(\rho)) \right\} \geq 0$.
    
    \item For $k=0$, the initialization is $\lambda_{0, i} = \max\{0, \eta'(V_i^{\pi_{0}}(\rho) - b_i)\}$. \\
    ~~ If $V_{i}^{\pi_0}(\rho) - b_i \leq 0$, then $\lambda_{0, i} = 0$ and $|\lambda_{0, i}| \leq \eta'| V_{i}^{\pi_0}(\rho) - b_i|$. \\
    ~~ If $V_{i}^{\pi_0}(\rho) - b_i > 0$, then $\lambda_{0, i} = \eta'( V_{i}^{\pi_0}(\rho) - b_i)$ and $|\lambda_{0, i}| = \eta' |V_{i}^{\pi_0}(\rho) - b_i|$. \\
    For $k \geq 0$, the update is $\lambda_{k+1, i} = \max\left\{\eta' (V^{\pi_{k+1}}_{i}(\rho) - b_{i}), \lambda_{k,i} + \eta'(b_{i} - V^{\pi_{k+1}}_{i}(\rho)) \right\}$. \\
    If $V_{i}^{\pi_{k+1}}(\rho) - b_i \leq 0$, then $\lambda_{k+1, i} = \lambda_{k, i} +  \eta'(b_{i} - V^{\pi_{k+1}}_{i}(\rho))$, and $|\lambda_{k+1,i}| = \lambda_{k,i} + \eta'|V_i^{\pi_{k+1}}(\rho) - b_i| \geq \eta'|V_i^{\pi_{k+1}}(\rho) - b_i|$ by the first property that $\lambda_{k,i} \geq 0$.
    
    If $V_{i}^{\pi_{k+1}}(\rho) - b_i > 0$, then $\lambda_{k+1, i} \geq \eta' (V_{i}^{\pi_{k+1}}(\rho) - b_i) > 0$. Thus $|\lambda_{k+1, i}| \ge \eta'|V_i^{\pi_{k+1}}(\rho) - b_i|$.
\end{enumerate}
\end{proof}
We now analyze the optimality gap and constraint violation separately.
\subsubsection{Optimality gap of ARNPG-EPD}
Recall the definition of the reward in the ascent direction
\begin{align}
    \tilde{r}_k(s, a) = r_1(s, a) + \sum_{i=2}^m [\lambda_{k, i} + \eta' (b_i - V_{i}^{\pi_k}(\rho)) ] r_i(s, a). 
\end{align}
Since $r_i(s,a) \leq 1$, we can verify that $|\tilde{r}_k(s,a)| \leq 1 + \frac{\eta'(m-1)}{1 - \gamma} + \sum_{i=2}^m\lambda_{k,i} =: L_k$, which implies $\|\tilde{r}_k\|_\infty \leq L_k$. Taking $\epsilon_k = \frac{\alpha \log(|\Ac|)}{(1 - \gamma)K}$, we choose $t_k = \lceil\frac{1}{1 - \gamma} \log(\frac{5 L_k K}{2 \eta' m \log(|\Ac|)}) + 1 \rceil$.
% $t_k - 1 \geq \frac{1}{1 - \gamma} \log(\frac{5 L_k K}{2\gamma^2\beta \log(|\Ac|)}) \geq \frac{1}{1 - \gamma} \log(\frac{5 \|\tilde{r}_k\|_\infty}{(1-\gamma)^2 \epsilon_k})$.

Since $\lambda_{k, i} + \eta' (b_i - V^{\pi_k}_{i}(\rho)) \geq 0$ by the second property in Lemma \ref{lem:L_property}, and $V^{\pi^*}_{i}(\rho) \geq b_i$ for any $i \in [2:m]$, taking $\pi = \pi^*$ in Proposition \ref{pro:fundamental-inequality} gives
\begin{equation}
\label{eqn:inner_analysis}
\begin{aligned}
    &V_{1}^{\pi_{k+1}}(\rho) + \sum_{i = 2}^m [\lambda_{k,i} + \eta'(b_{i} - V^{\pi_{k}}_{i}(\rho))] \cdot [V^{\pi_{k+1}}_{i}(\rho) - b_{i}] - \alpha \frac{D_{d^{\pi_{k+1}}_{\rho}}(\pi_{k+1} || \pi_k)}{1 - \gamma}\\
    \geq & V_{1}^{\pi^*}(\rho) - \alpha\frac{D_{d^{\pi^*}_{\rho}}(\pi^* || \pi_k) - D_{d^{\pi^*}_{\rho}}(\pi^* || \pi_{k+1})}{1 - \gamma} - \epsilon_k.
\end{aligned}
\end{equation}

Denote $\delta_{k,i} := b_i - V_i^{\pi_k}(\rho)$ as the constraint violation for the $i$-th constraint at macro step $k$. We thus have
% \begin{align*}
%     [\lambda_{k,i} + \eta'(b_{i} - V^{\pi_{k}}_{i}(\rho))] \cdot (V^{\pi_{k+1}}_{i}(\rho) - b_{i}) = \lambda_{k,i} (V^{\pi_{k+1}}_{i}(\rho) - b_{i}) - \eta' (V^{\pi_{k+1}}_{i}(\rho) - b_{i})^2. 
% \end{align*}
\begin{align*}
    [\lambda_{k,i} + \eta'(b_{i} - V^{\pi_{k}}_{i}(\rho))] \cdot (V^{\pi_{k+1}}_{i}(\rho) - b_{i}) = - \lambda_{k,i} \delta_{k+1, i} - \eta' \delta_{k,i}\delta_{k+1,i}.
\end{align*}
We can then bound this two terms respectively. 
% \begin{align}
%     \lambda_{k,i} \delta_{k+1, i} \geq \frac{\lambda_{k+1,i}^2 - \lambda_{k,i}^2 }{2 \eta'} - \eta' \delta_{k+1,i}^2
% \end{align}
\begin{itemize}
\item $\lambda_{k,i} \delta_{k+1, i}$: Note that $\lambda_{k+1, i} = \max\{- \eta' \delta_{k+1,i}, \lambda_{k,i} + \eta' \delta_{k+1, i} \}$. \\
    If $\lambda_{k+1, i} = - \eta' \delta_{k+1,i}$, then
    \begin{align*}
        \frac{1}{2}\lambda_{k+1, i}^2 - \frac{1}{2}\lambda_{k, i}^2 - \eta'^2 \delta_{k+1,i}^2 = - \frac{1}{2}\lambda_{k, i}^2 - \frac{\eta'^2}{2} \delta_{k+1,i}^2 \leq \eta' \lambda_{k, i} \delta_{k+1,i},
    \end{align*}
    which implies $ - \lambda_{k, i} \delta_{k+1,i} \leq \frac{\lambda_{k, i}^2 - \lambda_{k+1, i}^2}{2\eta'} + \eta' \delta_{k+1,i}^2$.

    If $\lambda_{k+1, i} = \lambda_{k, i} + \eta' \delta_{k+1,i}$, then
    \begin{align*}
        \eta' \lambda_{k, i} \delta_{k+1,i} &= \frac{1}{2} (\lambda_{k, i} + \eta' \delta_{k+1,i})^2 - \frac{1}{2} \lambda_{k, i}^2 - \frac{\eta'^2}{2} \delta_{k+1,i}^2 \geq \frac{1}{2} \lambda_{k+1, i}^2 - \frac{1}{2} \lambda_{k, i}^2 - \eta'^2 \delta_{k+1,i}^2,
    \end{align*}
    which also implies $ - \lambda_{k, i} \delta_{k+1,i} \leq \frac{\lambda_{k, i}^2 - \lambda_{k+1, i}^2}{2\eta'} + \eta' \delta_{k+1,i}^2$.
    \item $\eta' \delta_{k,i} \delta_{k+1, i}$: Note that $\eta' \delta_{k,i} \delta_{k+1, i} = \frac{\eta'}{2}\delta_{k,i}^2 + \frac{\eta'}{2}\delta_{k+1,i} - \frac{\eta'}{2}(\delta_{k,i} - \delta_{k+1,i})^2$, and $\frac{\eta'}{2}(\delta_{k,i} - \delta_{k+1,i})^2 \leq \frac{\gamma^2 \eta'}{(1 - \gamma)^4} D_{d^{\pi_{k+1}}_\rho}(\pi_{k+1} || \pi_k)$. We thus have $-\eta' \delta_{k,i} \delta_{k+1, i} \leq -\frac{\eta'}{2}(\delta_{k,i}^2 + \delta_{k+1,i}^2) + \frac{\gamma^2 \eta'}{(1 - \gamma)^4} D_{d^{\pi_{k+1}}_\rho}(\pi_{k+1} || \pi_k)$.
\end{itemize}
Substituting the above upper bounds into (\ref{eqn:inner_analysis}) leads to
\begin{align*}
    & V_1^{\pi_{k+1}}(\rho) + \frac{ \|\lambda_k\|_2^2 - \|\lambda_{k+1}\|_2^2 }{2 \eta'} + \eta'\frac{\|\delta_{k+1}\|_2^2 - \|\delta_k\|_2^2}{2} + \left(\frac{\eta' \gamma^2 m}{(1 - \gamma)^4} - \frac{\alpha}{1 - \gamma} \right) D_{d^{\pi_{k+1}}_\rho}(\pi_{k+1} || \pi_k) \\
    \geq& V_{1}^{\pi^*}(\rho) - \alpha\frac{D_{d^{\pi^*}_{\rho}}(\pi^* || \pi_k) - D_{d^{\pi^*}_{\rho}}(\pi^* || \pi_{k+1})}{1 - \gamma} - \epsilon_k.
\end{align*}
Recall $\alpha \ge \frac{2 \eta' m}{(1 - \gamma)^3}$, it then follows from telescoping that
\begin{align}
    \sum_{k = 1}^{K} V^{\pi_k}_{1}(\rho) & \geq K V^{\pi^*}_{1}(\rho) - \alpha \frac{D_{d^{\pi^*}_\rho}(\pi^* || \pi_{0}) - D_{d^{\pi^*}_\rho}(\pi^* || \pi_{K})}{1 - \gamma}  - \sum_{k=0}^{K-1} \epsilon_k \notag \\
    &\quad + \eta'\frac{\|\delta_0\|_2^2 - \|\delta_{K}\|_2^2}{2} + \frac{\|\lambda_K\|_2^2 - \|\lambda_0\|_2^2}{2\eta'} \\
    & = K V^{\pi^*}_{1}(\rho) - \alpha \frac{D_{d^{\pi^*}_\rho}(\pi^* || \pi_{0}) - D_{d^{\pi^*}_\rho}(\pi^* || \pi_{K})}{1 - \gamma}  - \sum_{k=0}^{K-1} \epsilon_k \notag \\
    &\quad + \left(\frac{\|\lambda_K\|_2^2}{2 \eta'} - \eta'\frac{\|\delta_{K}\|_2^2}{2}\right) + \eta' \frac{\|\delta_0\|_2^2 - \|\lambda_0\|_2^2}{2} - \frac{1/\eta' - \eta'}{2}\|\lambda_0\|_2^2 \label{eqn:thm_first_init} \\
    & \stackrel{(a)}{\geq} K V^{\pi^*}_{1}(\rho)  - \alpha \frac{D_{d^{\pi^*}_\rho}(\pi^* || \pi_{0}) - D_{d^{\pi^*}_\rho}(\pi^* || \pi_{K})}{1 - \gamma} - \sum_{k=0}^{K-1} \epsilon_k - \frac{1/\eta' - \eta'}{2}\|\lambda_0\|_2^2 \notag \\
    & \stackrel{(b)}{\geq} K V_{1}^{\pi^*}(\rho) - \frac{3 \alpha \log(|\Ac|)}{1 - \gamma}. 
    %- \frac{\eta'}{2(1-\gamma)^2}.% - \frac{1/\eta' - \eta'}{2}\|\lambda_0\|^2.
\label{eqn:thm_first}
\end{align}
$(a)$ holds due to the third property of Lemma \ref{lem:L_property}, and $(b)$ holds since $\pi_0$ is the uniformly distributed policy. Thus $D_{d^{\pi^*}_\rho}(\pi^* || \pi_{0}) = \sum_{s \in \Sc} d_{\rho}^{\pi^*}(s) \sum_{a \in \Ac}$ $\pi^*(a|s) \log (|\Ac| \pi^*(a|s)) \le \log(|\Ac|)$, $\sum_{k=0}^{K-1} \epsilon_k = \frac{\alpha \log(|\Ac|)}{1 - \gamma}$, and 
$\lambda_{0,i}^2 = \eta'^2 [\delta_{0, i}]_+^2$ implying
$\frac{1/\eta' - \eta'}{2}\|\lambda_0\|^2 \leq \frac{(\eta' - \eta'^3)\|\delta_{0}\|^2}{2} \leq \frac{\eta'}{2(1-\gamma)^2} \leq \frac{\alpha \log(|\Ac|)}{1 - \gamma}$. 
We now obtain the bound \eqref{eqn:PD-opt-gap}, after dividing by $K$ on both sides.

\subsubsection{Violation gap of ARNPG-EPD}
Recall that $\delta_{k,i} := b_i - V_i^{\pi_k}(\rho)$ is the constraint violation for the $i$-th constraint at macro step $k$. We aim to provide an upper bound on $\sum_{k = 1}^{K} \delta_{k,i}$ to control the constraint violation.

For any $i \in [2:m]$, since $\lambda_{k, i} = \max\{-\eta' \delta_{k,i}, \lambda_{k-1, i} + \eta' \delta_{k,i}\} \geq \lambda_{k-1, i} + \eta' \delta_{k,i}$, we have
\begin{align}
    \sum_{k = 1}^{K} \delta_{k,i} \leq \frac{\lambda_{K, i} - \lambda_{0, i}}{\eta'} \leq \frac{\lambda_{K, i}}{\eta'} \leq \frac{\|\lambda_K\|_2}{\eta'} \leq \frac{\|\lambda^*\|_2 + \|\lambda_K - \lambda^*\|_2}{\eta'}. \label{eq:cvb-st-1}
\end{align}
To upper bound the constraint violation, it therefore suffices to bound the dual variables. 

Consider the Lagrangian with optimal dual variable $\Lc(\pi, \lambda^*) = V^{\pi}_{1}(\rho) + \sum_{i=2}^m \lambda^*_i (V^{\pi}_{i}(\rho) - b_i)$, whose maximum value $V^{\pi^*}_{1}(\rho)$ is achieved by the optimal policy $\pi^*$. We know
\begin{align*}
    & K V^{\pi^*}_{1}(\rho) \stackrel{(a)}{=} K \Lc(\pi^*, \lambda^*) \geq \sum_{k = 1}^{K} \Lc(\pi_k, \lambda^*) = \sum_{k=1}^K V^{\pi_k}_{1}(\rho) + \sum_{i = 2}^m \lambda^*_i \sum_{k=1}^K (V^{\pi_k}_{i}(\rho) - b_i) \\
    =& \sum_{k=1}^K V^{\pi_k}_{1}(\rho) - \sum_{i = 2}^m \lambda^*_i \sum_{k=1}^K \delta_{k,i} \stackrel{(b)}{\geq} \sum_{k=1}^K V^{\pi_k}_{1}(\rho) - \frac{1}{\eta'} \sum_{i = 2}^m \lambda^*_i \lambda_{K, i} \\
    \stackrel{(c)}{\geq}& K V^{\pi^*}_{1}(\rho) - \alpha\frac{D_{d^{\pi^*}_\rho}(\pi^* || \pi_{0}) - D_{d^{\pi^*}_\rho}(\pi^* || \pi_{K})}{1 - \gamma} + \frac{\|\lambda_{K}\|^2}{2\eta'} - \frac{\eta' \|\delta_K\|^2}{2} - \frac{\lambda^*_i\sum_{i = 2}^m \lambda_{K, i}}{\eta'} - \Delta_K \\
    \geq& K V^{\pi^*}_{1}(\rho) - \frac{\alpha \log(|\Ac|)}{1 - \gamma}  + \frac{\alpha D_{d^{\pi^*}_\rho}(\pi^* || \pi_{K})}{1-\gamma} + \frac{\|\lambda_{K}\|_2^2}{2\eta'} - \frac{\eta' \|\delta_K\|_2^2}{2} - \frac{\lambda^*_i\sum_{i = 2}^m \lambda_{K, i}}{\eta'} - \Delta_K,
\end{align*}
where $\Delta_K := \frac{2\alpha \log(|\Ac|)}{1 - \gamma} \geq \sum_{k=0}^{K-1}\epsilon_k + \frac{1/\eta' - \eta'}{2} \|\lambda_0\|_2^2$. 
Then $(a)$ holds due to complementary slackness $\lambda^*_i(V^{\pi^*}_i(\rho) - b_i) = 0$, $(b)$ follows from \eqref{eq:cvb-st-1}, and  $(c)$ follows from (\ref{eqn:thm_first_init}) and the third property of Lemma \ref{lem:L_property}. It then follows that
\begin{align}
    \frac{\|\lambda_K\|_2^2}{2\eta'} - \frac{\lambda^*_i\sum_{i = 2}^m \lambda_{K, i}}{\eta'} \leq \frac{\alpha \log(|\Ac|)}{1 - \gamma}  - \frac{\alpha D_{d^{\pi^*}_\rho}(\pi^* || \pi_{K})}{1-\gamma} + \frac{\eta' \|\delta_K\|_2^2}{2} + \Delta_K. \label{eqn:key-violation}
\end{align}

Denoting $\delta^*_i := b_i - V_i^{\pi^*}(\rho) \leq 0$, according to Lemma \ref{lem:kl_property}, we have
\begin{align}
    \frac{\alpha D_{d^{\pi^*}_\rho}(\pi^* || \pi_{K})}{1-\gamma} \geq \frac{(1-\gamma)^3 \alpha}{2 \gamma^2} \|\delta_{K} - \delta^*\|_\infty^2 \geq \frac{(1-\gamma)^3 \alpha}{2 \gamma^2 m} \|\delta_{K} - \delta^*\|_2^2. \label{eqn:D2delta}
\end{align}
We can also obtain
\begin{align}
    - \frac{(1-\gamma)^3 \alpha}{2 \gamma^2 m} \|\delta_{K} - \delta^*\|_2^2 + \frac{\eta'}{2} \|\delta_K\|_2^2
    & = \left(\frac{\eta'}{2} - \frac{\gamma^2 m \eta'^2}{2[\gamma^2 m \eta' - (1 - \gamma)^3 \alpha]}\right) \|\delta^*\|^2 \\
    & + \frac{\gamma^2 m \eta' - (1-\gamma)^3 \alpha}{2 \gamma^2 m} \left\|\delta_K - \delta^* + \frac{\gamma^2 m \eta'}{\gamma^2 m \eta' - (1-\gamma)^3 \alpha} \delta^*\right\|^2, \notag
\end{align}
by substituting $a = \frac{(1-\gamma)^3 \alpha}{2 \gamma^2 m}, b=\frac{\eta'}{2}, x = \delta_K - \delta^*, y = \delta^*$ into the binomial equation
\begin{align*}
    -a \|x\|_2^2 + b \|x + y\|_2^2 = (b - \frac{b^2}{b-a})\|y\|_2^2 + (b-a) \|x + \frac{b}{b-a}y\|_2^2.
\end{align*}
Recalling $\alpha \geq \frac{2 \eta' m}{(1 - \gamma)^3}$, we can verify that $\frac{\gamma^2 m \eta' - (1-\gamma)^3 \alpha}{2 \gamma^2 m} \leq 0$ and $\frac{\eta'}{2} - \frac{\gamma^2 m \eta'^2}{2[\gamma^2 m \eta' - (1 - \gamma)^3 \alpha]} \leq \eta'$. It follows that
\begin{align}
    - \frac{(1-\gamma)^3 \alpha}{2 \gamma^2 m} \|\delta_{K} - \delta^*\|_2^2 + \frac{\eta'}{2} \|\delta_K\|_2^2 \leq \eta' \|\delta^*\|_2^2. \label{eqn:complicated-binomial}
\end{align}

Substituting \eqref{eqn:D2delta} and \eqref{eqn:complicated-binomial} into \eqref{eqn:key-violation} gives
\begin{align}
    \frac{1}{2\eta'}\|\lambda_K - \lambda^*\|_2^2 =& \frac{1}{2\eta'}\|\lambda^*\|^2 + \frac{1}{2\eta'} \|\lambda_{K}\|^2 - \frac{1}{\eta'} \sum_{i = 1}^m \lambda^*_i \lambda_{K, i} \notag\\
    \leq& \frac{1}{2\eta'} \|\lambda^*\|_2^2 + \frac{\alpha \log(|\Ac|)}{1 - \gamma} + \Delta_K + \eta' \|\delta^*\|^2 \notag\\
    \leq & \frac{1}{2\eta'} \|\lambda^*\|_2^2 + \frac{3 \alpha \log(|\Ac|)}{1 - \gamma} + \frac{\eta' (m-1)}{(1-\gamma)^2} \notag\\
    \leq & \frac{1}{2\eta'} \|\lambda^*\|_2^2 + \frac{4 \alpha \log(|\Ac|)}{1- \gamma}, \label{eqn:lambda_bound}
\end{align}
where the last inequality follows from $\frac{\eta' (m-1)}{(1 - \gamma)^2} \leq \frac{\alpha \log(|\Ac|)}{1 - \gamma}$. Using the above bound in \eqref{eq:cvb-st-1}, we get 
\begin{align}
    \sum_{k = 1}^{K} \delta_{k,i} & \leq \frac{\|\lambda^*\|_2}{\eta'} + \frac{\|\lambda_{K} - \lambda^*\|_2}{\eta'} \leq \frac{\|\lambda^*\|_2}{\eta'} + \sqrt{ \frac{\|\lambda^*\|_2^2}{\eta'^2} + \frac{8 \alpha \log(|\Ac|)}{(1 - \gamma)\eta'}} \notag\\
    & \leq 2 \frac{\|\lambda^*\|_2}{\eta'} + 3\sqrt{\frac{\alpha \log(|\Ac|)}{(1 - \gamma)\eta'}},
\label{eqn:thm_second}
\end{align}
from which we the  constraint violation upper bound given in (\ref{eqn:con-vio}) follows.

\begin{proof}[Proof of Theorem \ref{thm:ARNPG-EPD}]
We can conclude Theorem \ref{thm:ARNPG-EPD} from the above discussion on the optimality gap and the constraint violation.
\end{proof}

% \subsubsection{Proof of Corollary \ref{cor:ARNPG-EPD}}
\begin{proof}[Proof of Corollary \ref{cor:ARNPG-EPD}]
Note that the number of iterations in the inner loop depends on the value of dual variables, i.e., $t_k = \lceil\frac{1}{1 - \gamma} \log(\frac{5 L_k K}{2 \eta' m \log(|\Ac|)}) + 1 \rceil$ with $L_k = 1 + \frac{\eta' (m-1)}{1-\gamma} + \sum_{i=2}^m \lambda_{k, i}$. It is easy to verify that
\begin{align*}
    \frac{1}{2\eta'}\|\lambda_k - \lambda^*\|_2^2 \leq \frac{1}{2\eta'} \|\lambda^*\|_2^2 + \frac{4 \alpha \log(|\Ac|)}{1- \gamma}
\end{align*}
in the same manner as the proof of inequality \eqref{eqn:lambda_bound}. It then follows that
\begin{align*}
    \sum_{i =2}^m \lambda_{k,i} = \|\lambda_k\|_1 &\leq \sqrt{m} \|\lambda_k\|_2 \leq \sqrt{m} (\|\lambda_k - \lambda^*\| + \|\lambda^*\|) \\
    &\leq \sqrt{2m\|\lambda^*\|_2^2 + \frac{8\eta' m \alpha \log(|\Ac|)}{1- \gamma}} = O\left( \sqrt{m}\|\lambda^*\|_2 + \frac{m \log(|\Ac|)}{(1-\gamma)^2}\right).
\end{align*}
We then have $t_k = \Theta\left(\frac{1}{1 - \gamma} \log(K) \right)$, and $T = \sum_{k=0}^{K-1} t_k = \Theta\left(\frac{K}{1 - \gamma} \log(K) \right)$. We  conclude the proof by $\frac{K}{1 - \gamma} = \Theta(T / \log(T))$.
\end{proof}

\subsection{ARNPG-OMDA for max-min trade-off}

\subsubsection{Smoothness property} 
% \label{sec:max-min}
Define $\Xc := \Vc_\rho \times \Delta([m]) \subset \Rb^{2m}$. Define a norm $\Psi$ on $\Rb^{2m}$ by $\Psi(v, \lambda) = \|v\|_\infty + \|\lambda\|_1$. Its dual norm is $\Psi^*(v, \lambda) = \|v\|_1 + \|\lambda\|_\infty$. 

Define $G^{v, -\lambda}(X) := ( \nabla_v \Phi(X), - \nabla_\lambda \Phi(X) )$ for $X \in \Xc$. Assume the function $\Phi$ is $\beta$-smooth w.r.t.~the $\Psi$-norm over its domain $\Xc$, i.e., 
\begin{align}
    \Psi^*( G^{v,-\lambda}(X) - G^{v, -\lambda}(X')) \leq \beta \Psi(X - X'), \quad \forall X, X' \in \Xc. \label{eqn:minimax-smooth}
\end{align}

Define $E_k$, which will be an auxiliary term for the convergence analysis, as follows:
\begin{align}
    E_k := & \langle \tilde{G}^v_{k} - \tilde{G}^v_{k+1}, V^{\pi_{k+1}}_{1:m}(\rho) - V^{\tilde{\pi}_{k+1}}_{1:m}(\rho) \rangle + \alpha \frac{D_{d^{\tilde{\pi}_{k+1}}_\rho}(\tilde{\pi}_{k+1} || \pi_k) + D_{d^{\pi}_\rho}(\pi_{k+1} || \tilde{\pi}_{k+1}) }{1 - \gamma} \notag \\
    & + \langle \tilde{G}_{k}^{\lambda} - \tilde{G}_{k+1}^{\lambda}, \tilde{\lambda}_{k+1} - \lambda_{k+1} \rangle + \frac{D(\lambda_{k+1} || \tilde{\lambda}_{k+1}) + D(\tilde{\lambda}_{k+1} || \lambda_k)}{\eta'}. \label{eqn:Ek}
\end{align}
\begin{lemma}[Technical lemma for smoothness]\label{lem:minimax-technical}
When $\alpha \geq \frac{6 \beta}{(1 - \gamma)^4}$ and $\eta' \leq \frac{1}{6 \beta}$, $\sum_{k = 0}^{K-1} E_k \geq 0$.
\end{lemma}
\begin{proof}[Proof of Lemma \ref{lem:minimax-technical}]
% \begin{align}
%     &\langle \tilde{G}^v_{k+1} - \tilde{G}^v_{k}, V^{\pi_{k+1}}_{1:m}(\rho) - V^{\tilde{\pi}_{k+1}}_{1:m}(\rho) \rangle \\
%     &= \langle \nabla_v \Phi(V^{\tilde{\pi}_{k+1}}_{1:m}(\rho), \tilde{\lambda}_{k+1}) - \nabla_v \Phi(V^{\tilde{\pi}_{k}}_{1:m}(\rho), \tilde{\lambda}_{k}) , V^{\pi_{k+1}}_{1:m}(\rho) - V^{\tilde{\pi}_{k+1}}_{1:m}(\rho) \rangle \\
%     &= \langle \nabla_v \Phi(V^{\tilde{\pi}_{k+1}}_{1:m}(\rho), \tilde{\lambda}_{k+1}) - \nabla_v \Phi(V^{\tilde{\pi}_{k+1}}_{1:m}(\rho), \tilde{\lambda}_{k}) , V^{\pi_{k+1}}_{1:m}(\rho) - V^{\tilde{\pi}_{k+1}}_{1:m}(\rho) \rangle \notag \\
%     &\quad + \langle \nabla_v \Phi(V^{\tilde{\pi}_{k+1}}_{1:m}(\rho), \tilde{\lambda}_{k}) - \nabla_v \Phi(V^{\tilde{\pi}_{k}}_{1:m}(\rho), \tilde{\lambda}_{k}) , V^{\pi_{k+1}}_{1:m}(\rho) - V^{\tilde{\pi}_{k+1}}_{1:m}(\rho) \rangle \\
%     & \leq \beta_{v,\lambda}\|\tilde{\lambda}_{k+1} - \tilde{\lambda}_{k}\|_1 \| V^{\pi_{k+1}}_{1:m}(\rho) - V^{\tilde{\pi}_{k+1}}_{1:m}(\rho) \|_\infty \notag \\
%     & \quad + \beta_{v,v}\|V^{\pi_{k+1}}_{1:m}(\rho) -V^{\tilde{\pi}_{k}}_{1:m}(\rho)\|_\infty \| V^{\pi_{k+1}}_{1:m}(\rho) - V^{\tilde{\pi}_{k+1}}_{1:m}(\rho) \|_\infty \\
%     & \leq \frac{\beta_{v,\lambda}}{1 - \gamma} \|\tilde{\lambda}_{k+1} - \tilde{\lambda}_{k}\|_1 \|d^{\pi_{k+1}}_\rho - d^{\tilde{\pi}_{k+1}}_\rho\|_1 \notag \\
%     &\quad + \frac{\beta_{v,v}}{(1 - \gamma)^2} \|d^{\pi_{k+1}}_{\rho} - d^{\tilde{\pi}_{k}}_{\rho}\|_1 \| d^{\pi_{k+1}}_{\rho} - d^{\tilde{\pi}_{k+1}}_{\rho} \|_1.
% \end{align}

Recall the definition of $E_k$ \eqref{eqn:Ek}. 
Let $X_k := ( V^{\pi_{k}}_{1:m}(\rho), \lambda_{k} ) \in \Xc$ and $\tilde{X}_k := ( V^{\tilde{\pi}_{k}}_{1:m}(\rho), \tilde{\lambda}_{k} ) \in \Xc$; $G^{v, -\lambda}_k := G^{v, -\lambda}(X_k)$ and $\tilde{G}^{v, -\lambda}_{k} := G^{v, -\lambda}(\tilde{X}_k)$. We can then rewrite $E_k$ as
\begin{align*}
    E_k = & \langle \tilde{G}^{v,-\lambda}_k - \tilde{G}^{v,-\lambda}_{k+1}, X_{k+1} - \tilde{X}_{k+1} \rangle + \alpha \frac{D_{d^{\tilde{\pi}_{k+1}}_\rho}(\tilde{\pi}_{k+1} || \pi_k) + D_{d^{\pi}_\rho}(\pi_{k+1} || \tilde{\pi}_{k+1}) }{1 - \gamma} \notag \\
    & + \frac{D(\lambda_{k+1} || \tilde{\lambda}_{k+1}) + D(\tilde{\lambda}_{k+1} || \lambda_k)}{\eta'}.
\end{align*}
We can obtain
\begin{align*}
    & \langle \tilde{G}^{v,-\lambda}_{k+1} - \tilde{G}^{v,-\lambda}_k, X_{k+1} - \tilde{X}_{k+1} \rangle \stackrel{(a)}{\leq} \Psi^*(\tilde{G}^{v,-\lambda}_{k+1} - \tilde{G}^{v,-\lambda}_k) \Psi(X_{k+1} - \tilde{X}_{k+1}) \\
    \stackrel{(b)}{\leq} & \Psi^*(\tilde{G}^{v,-\lambda}_{k+1} - G^{v,-\lambda}_k) \Psi(X_{k+1} - \tilde{X}_{k+1}) + \Psi^*(G^{v,-\lambda}_k - \tilde{G}^{v,-\lambda}_k) \Psi(X_{k+1} - \tilde{X}_{k+1}) \\
    \stackrel{(c)}{\leq} & \beta \Psi(\tilde{X}_{k+1} - X_{k}) \Psi(X_{k+1} - \tilde{X}_{k+1}) + \beta \Psi(X_{k} - \tilde{X}_{k}) \Psi(X_{k+1} - \tilde{X}_{k+1}) \\
    \stackrel{(d)}{\leq}& \frac{\beta}{\sqrt{8} - 2} \Psi(\tilde{X}_{k+1} - X_{k})^2 + \left( \frac{\beta}{\sqrt{8} + 2} + \frac{\beta}{2} \right) \Psi(X_{k+1} - \tilde{X}_{k+1})^2 + \frac{\beta}{2} \Psi(X_k - \tilde{X}_{k})^2.
\end{align*}
Inequality $(a)$ follows from the Cauchy-Schwarz inequality for the $\Psi$-norm; $(b)$ from the triangle inequality; $(c)$ from the smoothness of function $\Phi$ defined in \eqref{eqn:minimax-smooth}; and $(d)$ from $ac + bc \leq \frac{a^2}{\sqrt{8} - 2} + \frac{c^2}{\sqrt{8} + 2} + \frac{b^2}{2} + \frac{c^2}{2}$.

Since $X_0 = \tilde{X}_0$, $\frac{1}{\sqrt{8} + 2} + \frac{1}{2} + \frac{1}{2} = \frac{1}{\sqrt{8} - 2}$, and $\Psi(v, \lambda)^2 \leq 2 \|v\|_\infty^2 + 2 \|\lambda\|_1^2$, we have
\begin{align*}
    & \sum_{k = 0}^{K-1} \langle \tilde{G}^{v,-\lambda}_{k+1} - \tilde{G}^{v,-\lambda}_k, X_{k+1} - \tilde{X}_{k+1} \rangle \\
    \leq& \frac{\beta}{\sqrt{8} - 2} \sum_{k = 0}^{K-1} \Psi(\tilde{X}_{k+1} - X_{k})^2 + \frac{\beta}{\sqrt{8} - 2} \sum_{k = 1}^{K} \Psi(X_k - \tilde{X}_{k})^2 \\
    \leq& \frac{2 \beta}{\sqrt{8} - 2} \sum_{k = 0}^{K-1} \left( \|V^{\tilde{\pi}_{k+1}}_{1:m}(\rho) - V^{\pi_k}_{1:m}(\rho)\|_\infty^2 + \|\tilde{\lambda}_{k+1}-\lambda_{k}\|_1^2 \right. \notag \\
    & \left. \hspace{3cm} + \|V^{\tilde{\pi}_{k+1}}_{1:m}(\rho) - V^{\pi_{k+1}}_{1:m}(\rho)\|_\infty^2 + \|\tilde{\lambda}_{k+1} - \lambda_{k+1}\|_1^2 \right).
\end{align*}
Noting that $\frac{2 \beta}{\sqrt{8} - 2} \leq 3 \beta$, by Lemma \ref{lem:kl_property} we have 
\begin{align*}
    \frac{2 \beta}{\sqrt{8} - 2} \|V^{\tilde{\pi}_{k+1}}_{1:m}(\rho) -  V^{\pi_k}_{1:m}(\rho)\|_\infty^2 \leq \frac{6 \gamma^2 \beta}{(1 - \gamma)^4} D_{d^{\tilde{\pi}_{k+1}}_{\rho}}(\tilde{\pi}_{k+1} || \pi_k).
\end{align*}
By Pinsker's inequality, we have 
\begin{align*}
    \frac{2 \beta}{\sqrt{8} - 2} \|\tilde{\lambda}_{k+1} - \lambda_{k+1}\|_1^2 \leq 6\beta D(\lambda_{k+1} || \tilde{\lambda}_{k+1}).
\end{align*}
Since $\alpha \geq \frac{6 \beta}{(1 - \gamma)^4}$ and $\eta' \leq \frac{1}{6 \beta}$, we conclude that $\sum_{k=0}^{K - 1}E_k \geq 0$.
\end{proof}

\subsubsection{Convergence of ARNPG-OMDA}
\begin{proof}[Proof of Theorem \ref{thm:ARNPG-OMDA}]
By $|\tilde{r}_k(s,a)| = |\langle \tilde{G}_k^{v}, r_{1:m}(s,a) \rangle| \leq \|\tilde{G}_k^{v}\|_{1} \|r_{1:m}(s,a)\|_\infty \leq L$, we know $\|\tilde{r}_k\|_\infty \leq L$. Taking $\epsilon_k = \frac{\alpha \log(|\Ac|)}{(1 - \gamma)K}$, we choose $t_k = \lceil\frac{1}{1 - \gamma} \log(\frac{5 L K}{6 \beta \log(|\Ac|)}) + 1 \rceil$.

Then by Proposition \ref{pro:fundamental-inequality}, for any policy $\pi$, we have two fundamental inequalities for the updates $\tilde{\pi}_{k+1}$ and $\pi_{k+1}$ respectively:
\begin{align*}
    V^{\tilde{\pi}_{k+1}}_{\tilde{r}_k}(\rho) - \alpha \frac{D_{d^{\tilde{\pi}_{k+1}}_\rho}(\tilde{\pi}_{k+1} || \pi_k)}{1 - \gamma} & \geq V^{\pi}_{\tilde{r}_k}(\rho) - \alpha \frac{D_{d^{\pi}_\rho}(\pi || \pi_k) - D_{d^{\pi}_\rho}(\pi || \tilde{\pi}_{k+1})}{1 - \gamma} - \epsilon_k, \\
    V^{\pi_{k+1}}_{\tilde{r}_{k+1}}(\rho) - \alpha \frac{D_{d^{\pi_{k+1}}_\rho}(\pi_{k+1} || \pi_k)}{1 - \gamma} & \geq V^{\pi}_{\tilde{r}_{k+1}}(\rho) - \alpha \frac{D_{d^{\pi}_\rho}(\pi || \pi_k) - D_{d^{\pi}_\rho}(\pi || \pi_{k+1})}{1 - \gamma} - \epsilon_k.
\end{align*}
Note that $V_{\tilde{r}_k}^{\pi}(\rho) = \langle \tilde{G}^v_k,  V^{\pi}_{1:m}(\rho) \rangle $. Taking $\pi = \pi_{k+1}$ in the first inequality, and summing two inequalities gives
\begin{align}
    & \langle \tilde{G}^{v}_{k+1},  V^{\tilde{\pi}_{k+1}}_{1:m}(\rho) - V^{\pi}_{1:m}(\rho) \rangle  \geq \alpha \frac{ D_{d^{\pi}_\rho}(\pi || \pi_{k+1}) - D_{d^{\pi}_\rho}(\pi || \pi_k)}{1 - \gamma} - 2 \epsilon_k \label{eqn:fund-v} \\
    & \quad \quad \quad + \langle \tilde{G}^v_{k} - \tilde{G}^v_{k+1}, V^{\pi_{k+1}}_{1:m}(\rho) - V^{\tilde{\pi}_{k+1}}_{1:m}(\rho) \rangle + \alpha \frac{D_{d^{\tilde{\pi}_{k+1}}_\rho}(\tilde{\pi}_{k+1} || \pi_k) + D_{d^{\pi}_\rho}(\pi_{k+1} || \tilde{\pi}_{k+1}) }{1 - \gamma}.  \notag
\end{align}

We can similarly get the inequality for $\lambda$ that
\begin{align}
    \langle\tilde{G}_{k}^{\lambda}, \tilde{\lambda}_{k+1} \rangle +  \frac{D(\tilde{\lambda}_{k+1} || \lambda_k)}{\eta'} & \leq \langle\tilde{G}_{k}^{\lambda}, \lambda \rangle +  \frac{D(\lambda || \lambda_k) - D(\lambda || \tilde{\lambda}_{k+1})}{\eta'}, \\
    \langle\tilde{G}_{k+1}^{\lambda}, \lambda_{k+1} \rangle +  \frac{D(\lambda_{k+1} || \lambda_k)}{\eta'} & \leq \langle\tilde{G}_{k+1}^{\lambda}, \lambda \rangle +  \frac{D(\lambda || \lambda_k) - D(\lambda || \lambda_{k+1})}{\eta'}.
\end{align}
Taking $\lambda = \lambda_{k+1}$ in the first inequality and summing two inequalities gives
\begin{align}
    & \langle\tilde{G}_{k+1}^{\lambda}, \lambda - \tilde{\lambda}_{k+1} \rangle \geq \frac{D(\lambda || \lambda_{k+1}) - D(\lambda || \lambda_k)}{\eta'} \notag \\
    &\quad\quad  + \langle \tilde{G}_{k}^{\lambda} - \tilde{G}_{k+1}^{\lambda}, \tilde{\lambda}_{k+1} - \lambda_{k+1} \rangle + \frac{D(\lambda_{k+1} || \tilde{\lambda}_{k+1}) + D(\tilde{\lambda}_{k+1} || \lambda_k)}{\eta'}. \label{eqn:fund-lambda}
\end{align}

Recall the definition of $E_k$ in \eqref{eqn:Ek} that
\begin{align*}
    E_k = & \langle \tilde{G}^v_{k} - \tilde{G}^v_{k+1}, V^{\pi_{k+1}}_{1:m}(\rho) - V^{\tilde{\pi}_{k+1}}_{1:m}(\rho) \rangle + \alpha \frac{D_{d^{\tilde{\pi}_{k+1}}_\rho}(\tilde{\pi}_{k+1} || \pi_k) + D_{d^{\pi}_\rho}(\pi_{k+1} || \tilde{\pi}_{k+1}) }{1 - \gamma} \notag \\
    & + \langle \tilde{G}_{k}^{\lambda} - \tilde{G}_{k+1}^{\lambda}, \tilde{\lambda}_{k+1} - \lambda_{k+1} \rangle + \frac{D(\lambda_{k+1} || \tilde{\lambda}_{k+1}) + D(\tilde{\lambda}_{k+1} || \lambda_k)}{\eta'}.
\end{align*}
We then have
\begin{align*}
    & - \Phi(V^{\pi}_{1:m}(\rho) , \tilde{\lambda}_{k+1}) + \Phi(V^{\tilde{\pi}_{k+1}}_{1:m}(\rho) , \lambda) \\
    =& \Phi(V^{\tilde{\pi}_{k+1}}_{1:m}(\rho) , \tilde{\lambda}_{k+1}) - \Phi(V^{\pi}_{1:m}(\rho) , \tilde{\lambda}_{k+1}) + \Phi(V^{\tilde{\pi}_{k+1}}_{1:m}(\rho) , \lambda) - \Phi(V^{\tilde{\pi}_{k+1}}_{1:m}(\rho) , \tilde{\lambda}_{k+1}) \\
    \stackrel{(a)}{\geq}& \langle \tilde{G}^{v}_{k+1},  V^{\tilde{\pi}_{k+1}}_{1:m}(\rho) - V^{\pi}_{1:m}(\rho) \rangle + \langle\tilde{G}_{k+1}^{\lambda}, \lambda - \tilde{\lambda}_{k+1} \rangle \\
    \stackrel{(b)}{\geq}& \alpha \frac{ D_{d^{\pi}_\rho}(\pi || \pi_{k+1}) - D_{d^{\pi}_\rho}(\pi || \pi_k)}{1 - \gamma} + \frac{D(\lambda || \lambda_{k+1}) - D(\lambda || \lambda_k)}{\eta'} - 2 \epsilon_k + E_k.
    %& \stackrel{(c)}{\geq} \alpha \frac{ D_{d^{\pi}_\rho}(\pi || \pi_{k+1}) - D_{d^{\pi}_\rho}(\pi || \pi_k)}{1 - \gamma} + \frac{D(\lambda || \lambda_{k+1}) - D(\lambda || \lambda_k)}{\eta'} - 2 \epsilon_k.
\end{align*}
Inequality $(a)$ is by the concavity of $\Phi(\cdot, \tilde{\lambda}_{k+1})$ and convexity of $\Phi(V^{\tilde{\pi}_{k+1}}_{1:m}(\rho), \cdot)$. Inequality
$(b)$ is based on combining \eqref{eqn:fund-v} and \eqref{eqn:fund-lambda}.

Taking $\pi = \pi^*$ and $\lambda = \argmin_{\lambda' \in \Lambda} \Phi\left(\frac{1}{K} \sum_{k=1}^{K} V^{\tilde{\pi}_{k}}_{1:m}(\rho), \lambda' \right)$, we have
\begin{align*}
    & F\left(\frac{1}{K} \sum_{k=1}^{K} V^{\tilde{\pi}_{k}}_{1:m}(\rho)\right) = \Phi\left(\frac{1}{K} \sum_{k=1}^{K} V^{\tilde{\pi}_{k}}_{1:m}(\rho), \lambda \right) \geq \frac{1}{K} \sum_{k=1}^{K} \Phi(V^{\tilde{\pi}_{k}}_{1:m}(\rho), \lambda) \\
    \geq& \frac{1}{K} \sum_{k=0}^{K-1} \Phi(V^{\pi^*}_{1:m}(\rho), \tilde{\lambda}_{k+1}) + \alpha \frac{D_{d^{\pi^*}_\rho}(\pi^* || \pi_{K}) - D_{d^{\pi^*}_\rho}(\pi^* || \pi_0)}{(1 - \gamma)K} + \frac{D(\lambda || \lambda_K) - D(\lambda || \lambda_0)}{\eta' K} \notag \\
    \quad& - \frac{2}{K}\sum_{k=0}^{K-1} \epsilon_k + \frac{1}{K}\sum_{k=0}^{K-1} E_k \\
    \stackrel{(a)}{\geq}& F(V^{\pi^*}_{1:m}(\rho)) - \frac{3\alpha \log(|\Ac|)}{(1 - \gamma)K} - \frac{\log(m)}{\eta' K}.
\end{align*}
Inequality $(a)$ is due to $D_{d_{\rho}^{\pi^*}} (\pi^*||\pi_0) \leq \log(|\Ac|)$ and Lemma \ref{lem:minimax-technical}.
\end{proof}

\begin{proof}[Proof of Corollary \ref{cor:ARNPG-OMDA}]
Note that $T = \sum_{k = 0}^{K-1} t_k = \Theta(\frac{K}{1-\gamma} \log(K))$. It implies $\frac{K}{1 - \gamma} = \Theta(T / \log(T))$. Substituting this into Theorem \ref{thm:ARNPG-OMDA} concludes Corollary \ref{cor:ARNPG-OMDA}.
\end{proof}

\section{More related works} \label{sec:app_related_work}

There are in general two scenarios when considering problems of multi-objective Markov decision processes (cf. survey \cite{roijers2013survey}): single-policy scenario and multi-policy scenario. This work focuses on the single-policy scenario, which we will simply refer to as multi-objective MDP, and we relegate the discussion of the multi-policy scenario to the end of this section. 

\paragraph{Multi-objective MDP}
In the single-policy scenario,
% , given the user's criteria, 
the agent optimizes the reward objectives according to the user's criteria. 
% to fulfill the \prk{What is ``demand"?} demand. 
The case of linear scalarization with known weights simply collapses to a canonical MDP.  Therefore, the focus of this line of research is on nonlinear scalarization. The lexicographical ordering of objectives was considered by Wray et al. \cite{wray2015multi} and Saisubramanian et al. \cite{saisubramanian2021multi}. Moffaert et al. \cite{van2013scalarized} studied a Chebyshev scalarization i.e., weighted $L_{\infty}$ scalarization via a Q-learning approach. Bai et al. \cite{bai2021joint} established an $O(1/\epsilon^4)$ sample complexity for a policy-gradient method under smooth concave scalarization. Besides optimization (pure exploration), the exploration-exploitation trade-off has also been studied, where with concave scalarization, an $O(\sqrt{T})$ regret was established under the Lipschitz continuous assumption on $F$ by \cite{agarwal2019reinforcement, wu2020accommodating}.

\paragraph{Constrained MDP} The constrained MDP (CMDP) is viewed as a special case of a multi-objective MDP in this paper. CMDPs have attracted much attention recently. %The global convergence analysis of the PG methods for MDPs has also been extended to CMDPs. 
Ding et al. \cite{ding2020natural} proposed an NPG-PrimalDual algorithm that uses a primal-dual approach with NPG and showed that it can achieve $O(1/\sqrt{T})$ global convergence for both the optimality gap and the constraint violation. Xu et al. \cite{xu2021crpo} proposed a primal approach called constrained-rectified policy optimization (CRPO) that updates the policy alternatively between optimizing objective and decreasing constraint violation, and enjoys the same $O(1/\sqrt{T})$ global convergence. Whether policy optimization for CMDP can achieve global convergence with a rate faster than $O(1/\sqrt{T})$ was stated as an open problem in \cite{ding2020natural}. There are two concurrent works \cite{ying2021dual,li2021faster} tackling the problem with positive answers, with some ergodic assumptions made to facilitate the proof of $\tilde{O}(1/T)$ global convergence. Ying et al. \cite{ying2021dual} propose an NPG-aided dual approach, where the dual function is smoothed by entropy regularization in the objective function. They show an $\tilde{O}(1/T)$ convergence rate to the optimal policy of the \textit{entropy-regularized} CMDP, but not to the true optimal policy, for which there is an established slow $O(1/\sqrt{T})$ convergence rate. They also make an additional strong assumption that the initial state distribution covers the entire state space. While such an assumption was initially used in the analysis of the global convergence of PG methods for MDPs \cite{agarwal2021theory, mei2020global}, it is not required when analyzing the global convergence of NPG methods \cite{agarwal2021theory, cen2021fast}. Moreover, this assumption does not necessarily hold for safe RL or CMDP, since the algorithm may need to avoid dangerous states even at initialization, with the consequence that the optimal policy depends on the initial state distribution. Li et al. \cite{li2021faster} propose a primal-dual approach with an $O(\log^2(T)/T)$ convergence rate to the true optimal policy by smoothing the Lagrangian with suitable regularization on both primal and dual variables. However, they assume that the Markov chain induced by any stationary policy is ergodic in order to ensure the smoothness of the dual function. This assumption, though weaker than the assumption made by \cite{ying2021dual}, will generally not hold in problems where one wishes to avoid unsafe states altogether.

\paragraph{Multi-policy scenario}
In the \textit{multi-policy} scenario, the agent is to determine not one policy but a set of policies, so that once a set of weights of a linear scalarization is subsequently provided, a policy from the determined set can be deployed \cite{roijers2013survey}. This approach essentially aims at determining the convex coverage set (CCS) of the Pareto frontier. Yang et al. \cite{yang2019generalized} proposed an envelope Q-learning algorithm based on the multi-objective Bellman optimality operator to utilize the convex envelope of the solution frontier. Zhou et al. \cite{zhou2020provable} studied the sample complexity under a generative model setting. Kyriakis et al. \cite{kyriakis2022pareto} utilized a policy gradient solver to search for a direction that is simultaneously an ascent direction for all objectives, utilizing a new loss function called Pareto Policy Adaptation. Wu et al. \cite{wu2020accommodating} considered preference-free exploration, where the agent collects samples in the exploration phase and computes a near-optimal policy for any preference-weighted reward function during the planning phase.

\end{document}